\documentclass{article}

\usepackage[main, final]{neurips_2026}

\usepackage[utf8]{inputenc} 
\usepackage[T1]{fontenc}    
\usepackage{hyperref}       
\usepackage{url}            
\usepackage{booktabs}       
\usepackage{amsfonts}       
\usepackage{microtype}      
\usepackage[table,xcdraw]{xcolor}
\usepackage{caption} 
\usepackage{graphicx} 
\usepackage{algorithm}
\usepackage{algpseudocode}
\usepackage{amsmath, amssymb, leftindex, nicefrac}
\usepackage{amsthm, mathrsfs}
\usepackage{tocloft,titletoc} 
\usepackage{subcaption}
\usepackage{todonotes}
\usepackage{changes}
\usepackage{array}
\usepackage{longtable}  
\usepackage{wrapfig}
\usepackage{threeparttable}

\newtheorem{definition}{Definition}[section]
\newtheorem{assumption}[definition]{Assumption}

\newtheorem{remark}[definition]{Remark}

\newcommand{\ip}[2]{\left\langle #1,#2\right\rangle}
\newtheorem{lemma}[definition]{Lemma}
\newtheorem{theorem}[definition]{Theorem}
\newtheorem{proposition}[definition]{Proposition}
\newtheorem{corollary}[definition]{Corollary}

\title{Towards Understanding the Power and Limits of\\the Muon Optimizer: A River-Valley Perspective}

\author{%
Tianqi Shen\textsuperscript{1} \quad
Runze Shi\textsuperscript{1} \quad
Jinji Yang\textsuperscript{1} \quad
Jianhao Ma\textsuperscript{2} \quad
Jiaye Teng\textsuperscript{3} \quad
Ziye Ma\textsuperscript{1}
\\
\textsuperscript{1}Department of Computer Science, City University of Hong Kong
\\
\textsuperscript{2}Department of Industrial Engineering, Tsinghua University
\\
\textsuperscript{3}School of Statistics and Data Science, Shanghai University of Finance and Economics
\\
{\setlength{\tabcolsep}{1em}
\begin{tabular}{@{}r l@{}}
\texttt{tianqshen5-c@my.cityu.edu.hk} & \texttt{runzeshi3-c@my.cityu.edu.hk} \\
\texttt{jinji.yang@cityu.edu.hk}      & \texttt{jianhao@tsinghua.edu.cn} \\
\texttt{tengjiaye@sufe.edu.cn}        & \texttt{ziyema@cityu.edu.hk} \\
\end{tabular}}
}

\begin{document}

\maketitle

\begin{abstract}
    Momentum Orthogonalized by Newton--Schulz (Muon) has emerged as a promising optimizer for large language model (LLM) training, yet when and why it outperforms established alternatives such as AdamW remain insufficiently understood. This gap matters in practice, as a sub-optimal choice of optimizer or even hyper-parameters can slow progress throughout expensive pre-training and waste substantial computational resources. Existing theories largely adopt a local-improvement perspective, emphasizing singular-value flattening or implicit curvature adaptation, but they neither determine which optimizer is preferable nor explain why Muon is typically fast only during early training. To better explain Muon's complicated behavior over the entire training trajectory, we develop a trajectory-level characterization of Muon dynamics based on a new mixed-spectrum matrix sensing (MSMS) model motivated by the river--valley view of LLM loss landscapes. Under simplifying assumptions, our continuous-time analysis reveals a qualitatively different convergence mechanism: Muon advances piecewise linearly toward the solution, whereas gradient descent and Adam-like flows approach it asymptotically along curved paths. With finite step sizes, however, the same normalized update magnitude makes Muon prone to overshooting near the solution, so its admissible step size shrinks with the residual scale. This links Muon's early advantage and late-stage limitations to a common mechanism and explains its sensitivity to learning-rate scheduling. Language model pre-training experiments support our discovery and suggest that switching from Muon to AdamW can preserve rapid initial progress while improving late-stage refinement.
    Our code is publicly available at \url{https://github.com/AmadeusSTQ/River-Valley-Muon}.
\end{abstract}

\section{Introduction}
\label{sec:introduction}

Non-Euclidean gradient-based optimizers~\citep{crawshaw2026exploration}, e.g. Momentum Orthogonalized by Newton--Schulz (Muon)~\citep{jordan2024muon}, 
are increasingly used for training large language models (LLMs), as documented in recent technical reports. 
Some notable examples include Kimi K2~\citep{team2025kimi}, GLM-5~\citep{zeng2026glm}, DeepSeek-V4~\citep{xu2026deepseek}, and Qwen3.8-Flash~\citep{qiu2026design}.
In parallel, substantial system-level engineering efforts 
have been devoted to optimizing Muon across modern software and hardware stacks~\citep{khona2026soap}.
Despite this growing adoption, it remains unclear under what conditions these methods could outperform the traditional Euclidean optimizers~\citep{bernstein2024old} that have long dominated large-scale training, 
such as AdamW~\citep{loshchilov2017decoupled}.
Given the enormous cost of LLM pretraining~\citep{cottier2024rising}, even modest differences in optimizer efficiency can lead to substantial financial and environmental consequences~\citep{schlotthauer2025pre,strubell2019energy}. Without guidance on when Muon or AdamW is more efficient, practitioners risk committing substantial resources to a suboptimal optimizer~\citep{van2026reinforcement}.

Recent studies~\citep{qi2026delving,dragutinovic2026use,shumaylov2026muon,wang2026muon1} have begun to address this question, 
but a definitive picture has yet to emerge.
A prevailing explanation attributes Muon's advantage to its ability to flatten the singular-value spectrum of the momentum or gradient matrix~\citep{boreiko2025towards,nguyen2026spectral,wang2026muon1,beneventano2026spectral}, 
while another line of work interprets Muon as a pseudo-second-order method that implicitly exploits Hessian information~\citep{shah2025practical,si2025adamuon,du2026newton,che2026second}.
Together, they suggest that Muon can mitigate the effects of heterogeneous curvature and rapidly reach a reasonably good solution.
If these explanations were sufficient, however, Muon's advantages should persist throughout optimization, leading to consistently faster convergence. Regrettably, existing empirical and theoretical results do not yet support such a conclusion. Most theoretical studies only establish a faster convergence or progress in the early stage, and typically for toy models that have a drastically different landscape to LLMs, such as Gaussian-mixture model~\citep{vasudeva2026muon}, isotropic-curvature model~\citep{su2025isotropic} and associative-memory models~\citep{wang2026muon2,wang2026muon,kim2026sharpcapacity}. 
Moreover, both empirical evidence~\citep{liu2025muon} and recent theoretical studies indicate that Muon may even suffer from poor late-stage convergence: it can exhibit suboptimal convergence rates or even fail to converge under certain optimization regimes~\citep{chang2025convergence,parshakova2026muon,do2026muon}.

\begin{figure}[ht]
	\centering
	\includegraphics[width=1.0\linewidth]{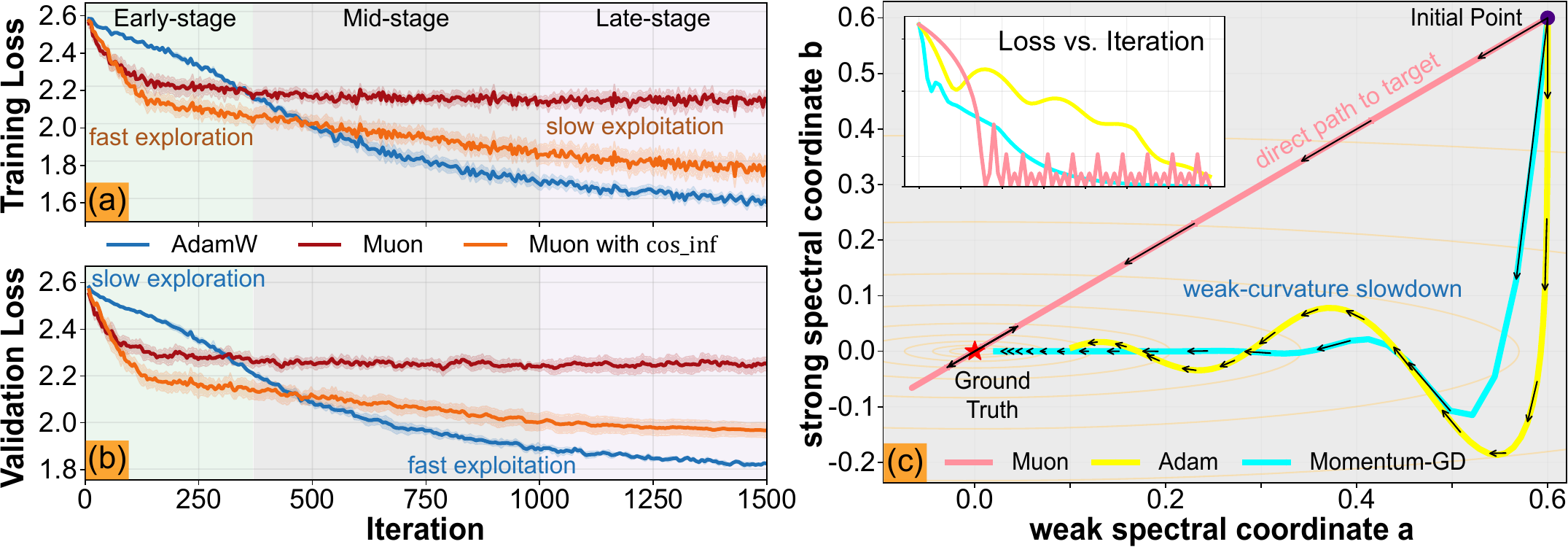}
	\caption{
        Key observations comparing Non-Euclidean \& Euclidean gradient-based optimizers.
        \textbf{(a)} and \textbf{(b)} report the losses for pre-training NanoChat on the Shakespeare dataset.
        The shaded bands denote the spread of losses obtained under different learning rates, reflecting each optimizer's sensitivity to the scheduler.
        \textbf{(c)} illustrates our proposed MSMS testbed in a simplified two-dimensional setting (see Appendix~\ref{subsubsec:2d_anisotropic_spectral_slice_case}), designed to capture salient features of LLM training dynamics.
    }
	\label{figure:motivation}
\end{figure}

These existing results indicate a non-uniform behavior of Muon convergence throughout the pre-training phase, which we summarize as a "fast-then-slow" dynamic. This means that Muon usually makes rapid early progress but slows down during the later exploitation stage. To showcase this dynamic more clearly, we have pre-trained NanoChat~\citep{nanochat} on the Tiny Shakespeare dataset~\citep{karpathy_char_rnn} with both Muon and AdamW, plotted in Figure~\ref{figure:motivation} (a) and (b). As our plot indicates, with a constant learning rate (LR), Muon rapidly reduces the loss in the first 500 iterations, but subsequently enters a slow exploitation phase and plateaus at a relatively high loss.
An LR-decay scheduler alleviates this late-stage stagnation, but compromises Muon's fast exploration in the early stage.
This fast-then-slow dynamic cannot be explained by existing theoretical frameworks discussed above which focus solely on spectral flattening or local second-order improvements.

To address this critical gap, we instead provide a trajectory-level characterization of Muon convergence that captures its time-varying optimization behavior. As illustrated in Figure~\ref{figure:motivation}(c), our analysis shows that, in the mixed-spectrum matrix sensing (MSMS) setting, Muon makes approximately linear progress toward the ground truth. In contrast, the convergence rates of Adam~\citep{kingma2014adam} and gradient descent (GD) with momentum~\citep{ma2026preconditioning} gradually decrease as their iterates approach the ground truth, with their early-stage trajectories typically following gradient-dominated directions. As another key contribution, we introduce MSMS as a new mathematical model designed to better capture the heterogeneous spectral structure observed in the optimization landscapes of real-world LLM training. To make the resulting high-dimensional dynamics analytically tractable, we further construct an invariant manifold that reduces the original training dynamics to a three-dimensional system of ordinary differential equations (ODEs).

This leads to the following insights.
When step-sizes are allowed to be infinitely small (i.e., gradient flow setting), the speed of Muon's progress need not diminish as its trajectory approaches a critical point, whereas the speed of GD/Adam-like flows typically decreases as their distance to the critical point shrinks.
This distinction explains Muon's advantage in rapid early-stage exploration.
Once finite step sizes are reintroduced through discretization, however, Figure~\ref{figure:motivation}'s (c) shows that Muon can overshoot the target solution (i.e., GT), while GD/Adam-like refinement can remain stable.
This provides an explanation for Muon's difficulty in late-stage exploitation: avoiding overshooting requires meticulously tuning the LR scheduler, which becomes particularly difficult as the trajectory approaches the target; an overly small late-stage LR, in turn, can slow its convergence.
These conclusions are not limited to the MSMS setting:
we further establish the same insights in a more general setting in Section~\ref{sec:river-valley-as-tool-in-generalized-settings} and corroborate them with experiments on LLMs in Section~\ref{sec:empirical-evidence-for-Muon-as-an-early-exploration-optimizer}.

\section{Related Work}
\label{sec:related-work}


\paragraph{Current understanding of Muon.}
Much of the existing understanding of Muon is derived from empirical observations~\citep{liu2025muon} or stylized case studies that do not fully capture the characteristics of LLM training.
\citet{su2025isotropic} studies Muon through an isotropic-curvature model, showing that spectral updates improve conditioning by homogenizing singular values of gradient matrix.
\citet{vasudeva2026muon} demonstrate that, in a Gaussian-mixture data model, Muon implicitly learns all principal components at comparable rates, thereby promoting a more balanced recovery of the underlying data components.
\citet{wang2026muon2} analyze associative-memory learning under heavy-tailed data and prove that Muon yields more balanced class-wise learning than Adam through a more isotropic update spectrum.
Following them, \citet{kim2026sharpcapacity} further show, in linear associative-memory models, that Muon achieves a substantially faster initial recovery rate than stochastic gradient descent (SGD) and admits a larger storage capacity.
Other studies~\citep{tveit2025muon,gronich2026implicit,gonon2026insights} rely on toy Transformers, smooth homogeneous neural networks, or even simple quadratic objectives.
Overall, these works primarily explain Muon's advantage in fast early-stage exploration~\citep{wen2026fantastic}; by contrast, Muon's poor late-stage exploitation remains less well understood.
For instance, \citet{chang2025convergence} show that Muon variants achieve only a suboptimal ergodic convergence rate in stochastic non-convex optimization and \citet{do2026muon} prove non-convergence in asymptotic stochastic optimization. \citet{parshakova2026muon} find that Muon can even fail to converge on convex Lipschitz functions regardless of the LR schedule.
To narrow this gap, we provide a trajectory-level analysis that tracks Muon's behavior across both the early-stage exploratory and late-stage exploitative phases to understand its strengths and weaknesses.

\paragraph{LLM landscape geometry.}
A pervasive observation in LLM training is strong anisotropy.
\citet{godey2024anisotropy} show that anisotropic directions in Transformers can characterize inductive preferences toward spiked features.
More directly, \citet{zhu2026accelerating,huang2026spectra,bernas2026revisiting} suggest that LLM loss landscapes and gradient spectra are highly anisotropic: loss reduction is often dominated by progress along a small number of spectral directions, while useful information is distributed across a long tail of weaker directions.
Similar anisotropic phenomena~\citep{weickert1998anisotropic,boscaini2016learning,yu2026anisotropic,choi2026isotropic} also arise in vision models.
This anisotropy motivates the river--valley picture, which decomposes the landscape~\citep{davis2025gradient} into slowly varying directions and sharply curved transverse directions.
\citet{wen2025understanding} introduce this geometry to explain warmup--stable--decay (WSD) schedules: optimization progresses along a flat ``river'' while oscillating across steep valleys, and LR decay suppresses these oscillations to reveal the accumulated progress.
Empirical landscape visualizations further support this interpretation, both during cooldown~\citep{dremov2025training} and with multi-power-law LR schedules~\citep{luo2025multi}.
Related analyses~\citep{lu2026towards} also explain how momentum can stabilize optimization in such anisotropic geometries.
Concurrent work by \citet{kim2026amuse} likewise adopts this geometric picture to interpret Muon's rapid progress and oscillatory behavior in LLM pretraining.
Although the river--valley assumption is only an approximate mechanistic model, it has become a useful explanatory framework for several phenomena in modern LLM training.
We thus develop an MSMS model that captures the anisotropy of LLM landscapes and whose dynamics reduce to ODEs with natural river--valley decomposition on invariant manifolds, enabling a detailed characterization.

\section{Preliminaries}
\label{sec:preliminaries}

\paragraph{Muon Optimizer.}
Unlike GD/Adam-like methods, which update parameters using the raw gradient or its coordinate-wise rescaling, Muon replaces each matrix-valued gradient or momentum with its orthogonalized counterpart. 
This operation preserves the singular directions of the gradient or momentum while discarding its singular values. 
Consequently, Muon not only flattens the singular-value spectrum of the update but also normalizes its overall magnitude.
Concretely, given a matrix parameter $X$ and an objective function $h(X)$, Muon performs the update\footnote{$M$ with superscript and  subscript denotes momentum instead of plain matrix used in the Section~\ref{sec:mixed-spectrum-matrix-sensing-model}.}
\begin{equation}
    \begin{aligned}
        M^{\rm Muon}_t 
        &= \nabla h\!\left(X^{\rm Muon}_t\right) + \xi_t M^{\rm Muon}_{t-1}, \\
    	X^{\rm Muon}_{t+1} 
        &= X^{\rm Muon}_t - \eta_t \, \mathtt{msign}\!\left(M^{\rm Muon}_t\right),
    \end{aligned}
    \label{definition:muon-update}
    \tag{Muon update}
\end{equation}
where $\xi_t$ is the momentum coefficient and $\eta_t$ is the learning rate. 
The operator $\mathtt{msign}(\cdot)$ returns the polar factor of its input matrix and is defined as
\begin{equation}
	\mathtt{msign}(Y) := \arg\min_Q
	\left\{
		\|Y-Q\|_F \,\mid\, \text{either } QQ^\top = I \text{ or } Q^\top Q = I
	\right\}.
	\label{definition:msign}
	\tag{\(\mathtt{msign}\) operator}
\end{equation}
In Section~\ref{sec:mixed-spectrum-matrix-sensing-model}, we begin with a simplified version of Muon that removes momentum and implementation details, e.g. Newton--Schulz iterations. 
This simplified setting allows us to isolate the effect of gradient orthogonalization itself. 
In Section~\ref{sec:river-valley-as-tool-in-generalized-settings}, we further study Muon with momentum in more general settings, with the additional preliminaries required for this extension collected in Appendix~\ref{subsec:spectral-river-geometry}.

\section{Mixed-spectrum Matrix Sensing Model}
\label{sec:mixed-spectrum-matrix-sensing-model}

\paragraph{Classical matrix sensing.}
Matrix sensing (MS) problem~\citep{davenport2012introduction} aims to recover a low-rank positive semidefinite (PSD) matrix $M^\star = X^\star {X^\star}^\top \in \mathbb{R}^{n \times n}$ with rank $r_\star < n$ from a set of linear measurements $\mathcal{A}$.
Burer-Monteiro (BM) factorized MS replaces the rank constraint~\citep{tillmann2013computational} with a factor $X \in \mathbb{R}^{n \times r_\mathrm{search}}$ such that $XX^\top \approx M^\star$. 
Here, $r_\mathrm{search}$ is a tunable search-rank hyperparameter.
The BM-factorized MS problem is then formulated as follows:
\begin{equation}
	\min_{X\in\mathbb{R}^{n\times r_\mathrm{search}}} 
	\frac{1}{2} \| \mathcal{A}(XX^\top) - b \|_2^2,
	\qquad
	b=\mathcal{A}(X^\star {X^\star}^\top).
    \label{equation:main_matrix_objfunc}
    \tag{BM-factorized MS}
\end{equation}
Here, $b\in\mathbb{R}^m$ is the known measurement vector, and $\mathcal{A}:\mathbb{R}^{n \times n}\to\mathbb{R}^m$ (whose adjoint operator is denoted as $\mathcal{A}^*$) is a known operator defined as $\mathcal{A}(M) = [\langle A_1, M \rangle, \langle A_2, M \rangle, \cdots \langle A_m, M \rangle]^\top$, with symmetric sensing matrices $A_i$. 
$\langle A_i, M \rangle = \operatorname{Tr}(A_i^\top M)$ is the Frobenius inner product.
This problem arises in numerous applications~\citep{li2011real,candes2015phase,rajani2014blind,jin2019towards}, and is equivalent to the training of two-layer quadratic neural networks \citep{li2018algorithmic}. 

\subsection{Problem Setting of Mixed-spectrum MS}

\paragraph{Empirical observations.} 

\begin{figure*}[ht]
	\centering
	\includegraphics[width=\textwidth]{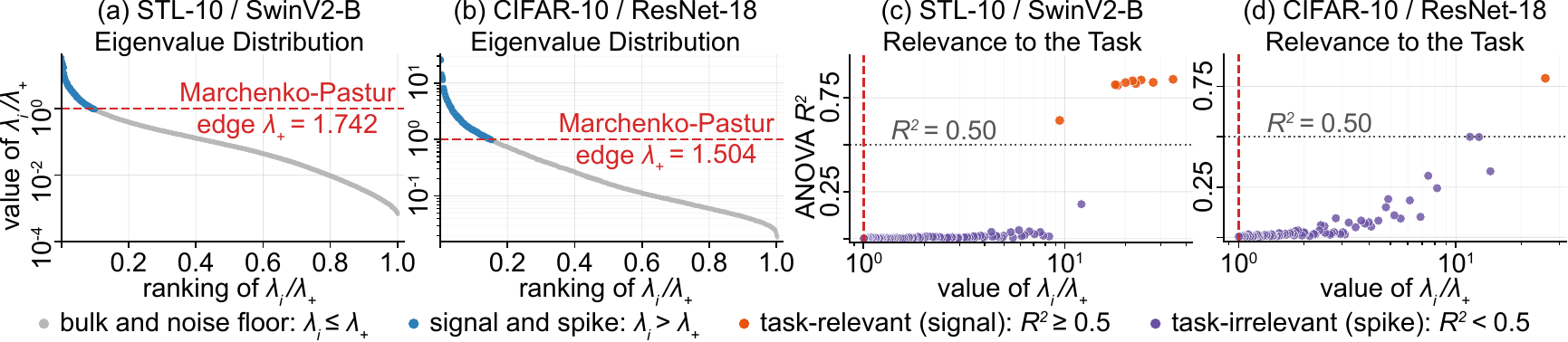}
	\caption{
		Empirical evidence for the mixed-spectrum covariance structure in image representations.
		Across datasets and architectures, the covariance spectrum of pre-classifier feature maps exhibits a pronounced bulk--signal--spike structure.
		Larger-scale settings with SwinV2-B~\citep{liu2022swin} or STL-10~\citep{coates2011analysis}, produce more task-relevant eigendirections; smaller-scale settings with ResNet-18~\citep{he2016deep} or CIFAR-10~\citep{krizhevsky2009learning} exhibit fewer.
	}
	\label{figure:natural_mixed_spectrum_evidence}
\end{figure*}

The optimization geometry of \ref{equation:main_matrix_objfunc} is determined by the sensing distribution: each sample corresponds to an $A_i$, and the squared gradient is proportional to $\mathcal{A}^*\mathcal{A}(XX^\top-M^\star)$. 
We thus design the sensing matrices in MSMS to be anisotropic, so that they can induce anisotropic gradient and Hessian geometries.
This choice is motivated by empirical evidence from both vision and language models. Figure~\ref{figure:natural_mixed_spectrum_evidence} illustrates the covariance spectra of pre-classifier feature maps from fine-tuned vision models. 
In panels (a) and (b), the Marchenko--Pastur upper edge~\citep{marvcenko1967distribution} is used as a reference threshold under the isotropic assumption:
$\lambda_+ = (1+\sqrt{d/n})^2$, where $d$ and $n$ denote the number of channels and pixels, respectively.
Most eigenvalues (gray) lie below $\lambda_+$, forming an isotropic bulk or background noise, while only a small fraction lie above it. 
Among these larger eigenvalues, panels (c) and (d) show that many (purple) have little association with label information, due to a low ANOVA coefficient $R^2$~\citep{neter1996applied}, revealing nuisance spikes; only a small remainder (orange) corresponds to task-relevant signal.
An analogous observation has been reported in LLM gradients, where energy concentrates in a compact, low-rank spike while weaker, context-specific semantic information spreads across a long spectral tail~\citep{huang2026spectra}. Together, these observations suggest that mixed-spectrum structure is common in deep learning and motivate its use as a tractable abstraction in MSMS.

\paragraph{Mixed-spectrum MS model.}
To model such anisotropic structure in a tractable way, we introduce a mixed-spectrum MS model with choice $r_\mathrm{search}=n$.
We assume that the GT matrix is an orthogonal projector $P_\star = UU^\top$ with $U^\top U = I_r$.
Let $\mathbb{R}^n$ admit the orthogonal decomposition
\begin{equation}
    P_\star + P_\diamond + P_\circ = I_n,
    \qquad
    \operatorname{dim}(P_\star)
    =
    \operatorname{dim}(P_\diamond)
    =
    \operatorname{dim}(P_\circ)
    =r,
    \label{equation:orthogonal_decomposition}
\end{equation}
where $P_\diamond = VV^\top$ with $V^\top V = I_r$, and $P_\circ = WW^\top$ with $W^\top W = I_r$.
The subspaces spanned by $U$, $V$, and $W$ are mutually orthogonal with $3r = n = r_\mathrm{search}$.
We construct sensing matrix as
\begin{equation}
    A = G
    +
    \pi_\star P_\star
    +
    \pi_\diamond P_\diamond
    +
    \pi_\circ P_\circ,
    \label{equation:mixed_spectrum_sensing_matrix}
\end{equation}
where $G\sim \mathrm{GOE}(n)$ is an isotropic Gaussian noise component, and $\pi:=[\pi_\star,\pi_\diamond,\pi_\circ]^\top\sim\mathcal{N}(0,\Sigma)$
with $\Sigma\succeq 0$ being a covariance matrix. Thus, this sensing model consists of an isotropic Gaussian background together with three structured spectral components: $P_\star$, which is directly aligned with the target signal; $P_\diamond$, which represents dominant but task-irrelevant spiked directions; and $P_\circ$, which captures the remaining bulk or residual directions. This signal--spike--bulk decomposition is inspired by~\citep{braun2026spectral} but more closely aligned with the 
LLM observations~\citep{huang2026spectra}.

Under this model, we study the corresponding population objective (see Lemma~\ref{lemma:population_loss_and_its_gradients})
\begin{equation}
    \min_{X\in\mathbb{R}^{n\times r_\mathrm{search}}}
    \frac{1}{2}
    \mathbb{E}
    \left[
        \left\langle
            A,
            XX^\top - P_\star
        \right\rangle^2
    \right].
    \label{equation:main_mixed_spectrum_matrix_objfunc}
    \tag{Mixed-spectrum MS}
\end{equation}
Equivalently, the population loss can be calculated as follows if denoting $M := XX^\top \in\mathbb{R}^{n\times n}$:
\begin{equation}
    \mathcal{F}(M)
    =
    \frac{1}{n}
    \left\|M-P_\star\right\|_F^2
    +
    \frac{1}{2}
    \begin{bmatrix}
        \langle P_\star, M-P_\star\rangle \\
        \langle P_\diamond, M\rangle \\
        \langle P_\circ, M\rangle
    \end{bmatrix}^\top
    \Sigma 
    \begin{bmatrix}
        \langle P_\star, M-P_\star\rangle \\
        \langle P_\diamond, M\rangle \\
        \langle P_\circ, M\rangle
    \end{bmatrix}.
    \label{equation:main_population_loss}
\end{equation}

\paragraph{Invariant manifold and reduced loss.}
Here, an \emph{invariant manifold} means that if the optimization is initialized on the manifold, then the subsequent iterates remain on it. Specifically, for the variable $M=XX^\top$ and the factor form $X$, we define $\mathcal{M}$ for vanilla GD and $\mathcal{X}$ for simplified Muon:
\begin{align}
    \mathcal{M}
    &:=
    \left\{
        M = \alpha P_\star + \beta P_\diamond + \gamma P_\circ
        :\; \alpha,\beta,\gamma \in \mathbb{R}
    \right\},
    \tag{Invariant manifold $\mathcal{M}$}
    \label{equation:invariant_manifold_M}
    \\
    \mathcal X
    &:=
    \left\{
        X=
        \big[a U\;\; b V\;\; c W\big]
        :\; a,b,c \in \mathbb{R}
    \right\}.
    \tag{Invariant manifold $\mathcal{X}$}
    \label{equation:invariant_manifold_X}
\end{align}
A key property of the population loss~\eqref{equation:main_population_loss} is that these manifolds are invariant under the corresponding optimization dynamics. 
If vanilla GD is initialized in $\mathcal{M}$, then the entire trajectory remains in $\mathcal{M}$ (see Lemma~\ref{lemma:gd_dynamics}). 
Similarly, if Muon is initialized in $\mathcal{X}$, then its trajectory stays in $\mathcal{X}$ (see Lemma~\ref{lemma:muon_dynamics}). 
Moreover, $\mathcal{X}$ maps exactly into $\mathcal{M}$, since for any $X\in\mathcal{X}$, $XX^\top = a^2P_\star + b^2P_\diamond + c^2P_\circ \in \mathcal{M}$.
Studying the two invariant manifolds has three advantages. First, they reduce matrix-valued dynamics to a low-dimensional scalar system, substantially simplifying the analysis. Second, the invariant manifold is precisely $\operatorname{span}\{P_\star,P_\diamond,P_\circ\}$, which constitutes the entire subspace of interest. Third, common small initializations of the form $\epsilon I$~\citep{stoger2021small} lie on the manifold because $I=P_\star+P_\diamond+P_\circ$.
The reduced scalar loss on $\mathcal{M}$ is
\begin{align}
    f(\alpha,\beta,\gamma)
    =
    \frac12
    \begin{bmatrix}
        \alpha-1 \\ \beta \\ \gamma
    \end{bmatrix}^{\top}
    \hat{K}
    \begin{bmatrix}
    	\alpha-1 \\ \beta \\ \gamma
    \end{bmatrix},
    \qquad
    \hat{K} := \frac{2r}{n}I_3+r^2\Sigma \succeq 0.
    \label{equation:reduced_M_loss}
\end{align}
Its minimizer is $(\alpha,\beta,\gamma)=(1,0,0)$, yielding a natural river--hill splitting in which $\alpha$ (signal) measures progress along the river direction, while $\beta$ (spike) and $\gamma$ (bulk) measure progress along the hill directions.
On $\mathcal{X}$, the induced coordinates satisfy $\alpha=a^2, \beta=b^2, \gamma=c^2$; hence the reduced objective $f(a,b,c)$ is obtained by substituting them into $f(\alpha,\beta,\gamma)$.
These scalar reductions later yield a convenient low-dimensional ODE system for comparing simplified Muon with vanilla GD.

\begin{figure*}[ht]
    \centering
    \includegraphics[width=0.95\textwidth]{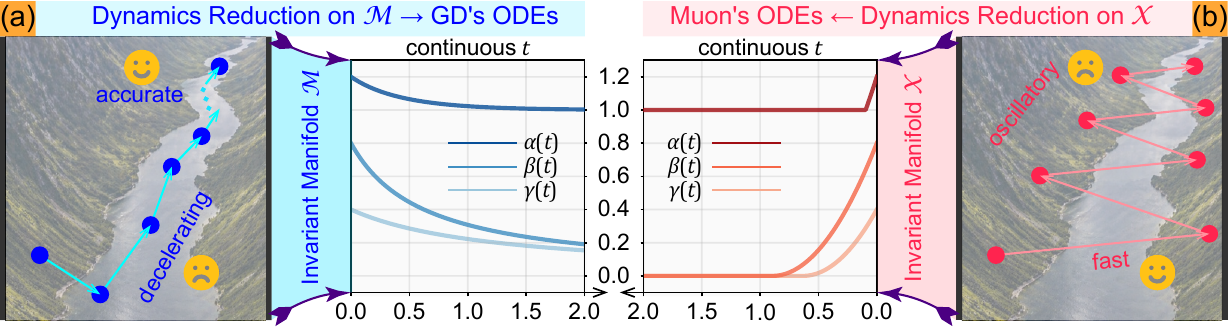}
    \caption{
    Reduction of the optimization dynamics in a river--valley to scalar ODEs on the invariant manifolds:
    (a) vanilla GD on \(\mathcal{M}\), using \eqref{equation:gd_diag_closed_form};
    (b) simplified Muon on \(\mathcal{X}\), using \eqref{equation:muon_diag_closed_form}.
    }
    \label{figure:reduction_on_invariant_manifold}
\end{figure*}

\subsection{Dynamics Analysis of Mixed-spectrum MS from a River-valley View}
We now analyze the continuous/discrete dynamics of vanilla GD and simplified Muon on the invariant manifolds introduced above under zero-momentum setting. 
We use the river--hill splitting
\begin{equation}
	\hat{K}
	=:
	\begin{bmatrix}
		\hat{a} & \hat{b}^\top \\
		\hat{c} & \hat{D}
	\end{bmatrix}
	\in \mathbb{S}^{3\times 3},
\quad
\hat{a} \in \mathbb{R},
	\quad
	\hat{b} = \hat{c} \in \mathbb{R}^2,
	\quad
	0 \prec \hat{D} \in \mathbb{S}^{2\times 2}.
	\nonumber
\end{equation}

\paragraph{A diagonal special case.}
To build intuition, we first consider a particularly simple case in which $\hat{K}$ is diagonal.
In this case, the dynamics of $\alpha$, $\beta$, and $\gamma$ are completely decoupled, and the river is given by $\beta=\gamma=0$.
For notational simplicity, we define $\kappa := \hat{K} / r$.

\begin{proposition}[Reduced ODE system of vanilla GD and simplified Muon, diagonal case]
\label{proposition:reduced_continuous_dynamics_of_gd_and_muon_diag}
    Assume that $\kappa=\mathtt{Diag}(\kappa_{11},\kappa_{22},\kappa_{33})$, the initialization satisfies $\beta_0,\gamma_0>0$ and $\alpha_0>1$.
    The vanilla GD flow restricted to $\mathcal{M}$ is governed by the following ODE system with its closed-form solution:
    \begin{equation}
    	\begin{cases}
    		\dot{\alpha}
    		=
    		-4\kappa_{11}\alpha(\alpha-1)
    		\quad&\Rightarrow\quad
    		\alpha(t) = 
    		\frac{1}{
    			1 + \left((1-\alpha_0)/\alpha_0\right)
    			e^{-4 \kappa_{11} t}
    		},
    		\\
    		\dot{\beta}
    		=
    		-4\kappa_{22}\beta^2
    		\quad&\Rightarrow\quad
    		\beta(t) = 
    		\frac{\beta_0}{
    			1 + 4 \kappa_{22} \beta_0 t
    		},
    		\\
    		\dot{\gamma}
    		=
    		-4\kappa_{33}\gamma^2
    		\quad&\Rightarrow\quad
    		\gamma(t) = 
    		\frac{\gamma_0}{
    			1 + 4 \kappa_{33} \gamma_0 t
    		}.
    	\end{cases}
    	\label{equation:gd_diag_closed_form}
    \end{equation}
    The simplified Muon flow restricted to $\mathcal{X}$ also induces the ODE system with its closed-form solution:
    \begin{equation}
    	\begin{cases}
    		\dot{\alpha}
    		\in
    		-2\sqrt{\alpha}\,
    		\mathtt{sign}
    		[
    		\kappa_{11}(\alpha-1)
    		]
    		\quad&\Rightarrow\quad
    		\alpha(t) = 
    		(
    		\max \{
    		\sqrt{\alpha_0} - t, 1
    		\}
    		)^2,
    		\\
    		\dot{\beta}
    		\in
    		-2\sqrt{\beta}\,
    		\mathtt{sign}
    		(
    		\kappa_{22}\beta
    		)
    		\quad&\Rightarrow\quad
    		\beta(t) = 
    		(
    		\max \{
    		\sqrt{\beta_0} - t, 0
    		\}
    		)^2,
    		\\
    		\dot{\gamma}
    		\in
    		-2\sqrt{\gamma}\,
    		\mathtt{sign}
    		(
    		\kappa_{33}\gamma
    		)
    		\quad&\Rightarrow\quad
    		\gamma(t) = 
    		(
    		\max \{
    		\sqrt{\gamma_0} - t, 0
    		\}
    		)^2.
    	\end{cases}
    	\label{equation:muon_diag_closed_form}
    \end{equation}
\end{proposition}
Therefore, in the diagonal case, Figure~\ref{figure:reduction_on_invariant_manifold} plots the closed-form solutions, enabling a direct comparison of the optimization behaviors.
Without finite-step-size effects, simplified Muon progresses substantially faster than vanilla GD along both the river and hill components.
However, Figure~\ref{figure:mixed_spectrum_ms_diag_coefficients} shows that with a finite step size $\eta>0$, simplified Muon may overshoot switching surfaces such as $\alpha=1$, inducing oscillations, whereas vanilla GD converges more slowly but more accurately.

\paragraph{General PSD case.}
While the diagonal special case is a useful sanity check, it precludes coupling between the $\alpha$ (signal), $\beta$ (spikes), and $\gamma$ (bulk).
In our MSMS model, $\pi\sim\mathcal N(0,\Sigma)$ is intended to capture correlated spectral anisotropy in learned representations, e.g. empirical covariances in vision/language models exhibit heterogeneous directions with nontrivial cross-correlations. We thus study the general PSD case $\hat K\succeq 0$, where the off-diagonal blocks of $\hat{K}$ encode these couplings.

\begin{table*}[ht]
	\caption{River-side quantities in the general PSD case.}
	\label{table:river_quantities_psd}
	\centering
	\resizebox{1.0\linewidth}{!}{
		\begin{tabular}{p{0.3\linewidth} p{0.9\linewidth}}
			\hline
			River item and notation & 
			Geometric interpretation \\ 
			\hline
			variable $\theta_k = \alpha_k - 1$ & 
			Coordinate along the coupled river. It measures progress of the signal component toward the target value $\alpha=1$. \\ 
			force $\phi_k = \hat a \theta_k + \hat b^\top \omega_k$ & 
			Reduced gradient component in the river direction before eliminating the hill variables. \\ 
			energy $E_k = \hat a_{\rm eff}\theta_k^2$ & 
			Effective loss remaining after the hill variables are optimally adjusted. Equivalently, it is the Schur-complement energy along the valley floor. \\ 
			\hline
			\multicolumn{2}{c}{GD river iteration:\quad $\theta_{k+1}=\theta_k-4\eta(1+\theta_k)\phi_k$} \\
			\multicolumn{2}{c}{Muon river iteration:\quad $\theta_{k+1}=\theta_k-2\eta\sqrt{1+\theta_k}\,\mathtt{sign}(\phi_k)$} \\
			\hline
		\end{tabular}
	}
	\begin{flushleft}
		\footnotesize
		* This table highlights only the river-side quantities. The full river-hill notation and discrete dynamics are given in Table~\ref{table:variable_force_energy}, \eqref{equation:gd_river_valley_iteration}, and \eqref{equation:muon_river_valley_iteration}. The hill energy is measured by the hill force $\psi_k$ rather than the raw hill coordinate $\omega_k$ (see Remark~\ref{remark:a_brief_notation_summary}).
	\end{flushleft}
\end{table*}

The reduced discrete iterations are summarized in Table~\ref{table:variable_force_energy}, where $\hat a_{\rm eff} := \hat a-\bar q$ with $\bar q := \hat c^\top\hat D^{-1}\hat c$. It gives a natural river-valley decomposition (the interpretation is different from the diagonal case):
when $\hat{K}$ is not diagonal, the raw coordinates ($\theta,\omega$) are coupled, and the river-valley structure is effectively rotated.
Consequently, the magnitude of the raw hill variable $\omega_k$ alone does not correctly measure the transverse deviation from the river. The relevant transverse quantity is instead the hill force $\psi_k$.
We then define the reduced river as the no-hill-force (as well as no-hill-energy) manifold
\begin{equation}
	\mathcal{R} :=
	\{(\theta,\omega):\psi=0\} = 
	\{(\theta,\omega):\omega=-\hat{D}^{-1} \hat{c}\,\theta\}.
	\tag{River in PSD case}
\end{equation}
Thus the river is generally not the set $\omega=0$; it is a tilted one-dimensional manifold parameterized by $\theta$.
On this manifold, the river force becomes $\phi=\hat a\theta-\hat b^\top\hat D^{-1}\hat c\,\theta=\hat a_{\rm eff}\theta$.
Therefore, if $\hat{a}_{\rm eff}>0$, the only point on the river with zero river force is $\theta=0$, which further gives $\omega=0$.
In other words, $\psi=0$ alone means that the iterate lies on the river, but exact recovery also requires the river force to vanish.
This rotated decomposition also separates the reduced quadratic energy.
Indeed, by completing the square, \eqref{equation:reduced_M_loss} becomes $\hat{a}_{\rm eff}\theta^2 + \psi^\top \hat{D}^{-1} \psi$ (see Definition~\ref{definition:river_hill_energy}), which is exactly the river and hill energies defined in Table~\ref{table:variable_force_energy}. Hence $H_k=O(\eta^2)$ means that the iterate lies inside an $\psi_k=O(\eta)$ vicinity (we call a hill-force tube) around the river. 
In the following Theorem~\ref{theorem:formal_mixed_spectrum_ms}, we formally study this regime when the hill energy has reached the $O(\eta^2)$ scale and $E_k\gg H_k$. 
The proof follows from Propositions~\ref{proposition:GD-river-progress-after-hill-reduction},~\ref{proposition:Muon-river-progress-after-hill-reduction}, and~\ref{proposition:late_stage_behavior_in_relative_tube}; see Section~\ref{subsec:river_valley_of_reduced_mixed_spectrum_ms} for details.

\begin{theorem}[Dynamics along river for the PSD case]
\label{theorem:formal_mixed_spectrum_ms}
Assume that the river curvature $\hat a_{\rm eff} > 0$. 
Suppose that, on the time interval under consideration, $0<\underline\alpha \le 1+\theta_k \le \overline\alpha$.
For an algorithm $\square\in\{\mathrm{GD},\mathrm{Muon}\}$, after a hill-reduction time $N_\square$, the trajectory remains in an $O(\eta)$ hill-force tube:
\begin{equation}
    H_k \le R_\square^2\eta^2,
    \qquad
    \forall k\ge N_\square,
    \label{equation:post_hill_relaxation}
\end{equation}
for some constant $R_\square>0$ independent of $\eta$. 
Hitting times $T_\square$ are defined as
\begin{equation}
    T_\square := \inf\{
		k \geq N_\square \,:\,
		|\theta_k| \leq C_\square \eta
	\}.
    \label{equation:hitting_times}
\end{equation}
Then the following statements hold.

\textbf{(i) GD makes geometric progress along the river.}
Assume that the GD iterate satisfies \eqref{equation:post_hill_relaxation} with constants $N_{\rm GD},R_{\rm GD}$ and $\eta$ is sufficiently small.
Define 
$
    C_{\rm GD} := \max\{
        2\sqrt{\bar q} / \hat{a}_{\rm eff}, 
        \sqrt{\Gamma_{\rm GD} / \hat{a}_{\rm eff}} 
    \} R_{\rm GD}
$ with a fixed constant $\Gamma_{\rm GD} \gg 1$.
Then, for every $k\ge N_{\rm GD}$ such that $|\theta_k|\ge C_{\rm GD}\eta$,
\begin{equation}
    |\theta_{k+1}|
    \le
    \left(1-2\underline\alpha\,\hat a_{\rm eff}\eta\right)|\theta_k|,
    \label{eq:gd_geometric_river_progress}
\end{equation}
and the iterate is in a river-dominant phase with $E_k \gg H_k$.
Hence the number of post-transition GD iterations required to reduce a constant-scale river error
to the $O(\eta)$ scale is bounded by
\begin{equation}
    T_{\rm GD}-N_{\rm GD}
    \le
    \left\lceil
    \frac{1}{2\underline\alpha\,\hat a_{\rm eff}\eta}
    \log
    \frac{|\theta_{N_{\rm GD}}|}{C_{\rm GD}\eta}
    \right\rceil
    =
    O\!\left(\frac{1}{\eta}\log\frac{1}{\eta}\right).
    \label{eq:gd_iteration_complexity}
\end{equation}

\textbf{(ii) Muon makes linear progress along the river.}
Assume that the Muon iterate satisfies \eqref{equation:post_hill_relaxation} with constants
$N_{\rm Muon},R_{\rm Muon}$. Define 
$
    C_{\rm Muon} := \max\{
        2\sqrt{\bar q} R_{\rm Muon} / \hat a_{\rm eff}, 
        2\sqrt{\overline{\alpha}}, 
        R_{\rm Muon} \sqrt{\Gamma_{\rm Muon} / \hat{a}_{\rm eff}}
    \}
$ with a fixed constant $\Gamma_{\rm Muon} \gg 1$.
Then, for every $k\ge N_{\rm Muon}$ such that $|\theta_k|\ge C_{\rm Muon}\eta$,
\begin{equation}
    |\theta_{k+1}|
    \le
    |\theta_k|-2\sqrt{\underline\alpha}\,\eta ,
    \label{eq:muon_linear_river_progress}
\end{equation}
and the iterate is in a river-dominant phase with $E_k \gg H_k$.
Therefore the number of post-transition Muon iterations required to reduce a constant-scale river
error to the $O(\eta)$ scale is bounded by
\begin{equation}
    T_{\rm Muon}-N_{\rm Muon}
    \le
    \left\lceil
    \frac{\left(|\theta_{N_{\rm Muon}}|-C_{\rm Muon}\eta\right)_+}
    {2\sqrt{\underline\alpha}\,\eta}
    \right\rceil
    =
    O\!\left(\frac{1}{\eta}\right).
    \label{eq:muon_iteration_complexity}
\end{equation}

Finally, assume that once $\theta_k$ enters the $O(\eta)$-neighborhood of the river equilibrium, both vanilla GD and simplified Muon remain in a relative river tube (see Remark~\ref{remark:relative_river_tube_natural}):
\begin{equation}
	\sqrt{\bar q\, H_k}
	\le
	\bar\rho\,\hat a_{\rm eff}|\theta_k|,
	\qquad
	\bar\rho\in(0,1).
	\label{equation:relative_river_tube}
\end{equation}
Under this condition, GD can further refine the river variable from $O(\eta)$ to any target accuracy $\varepsilon\ll \eta$, whereas fixed-step simplified Muon increases the river energy after one step, i.e., $E_{k+1}>E_k$.
\end{theorem}
At the hitting times $T_\square$, both algorithms satisfy $|\theta_{T_\square}|=O(\eta)$ and $H_{T_\square}=O(\eta^2)$, and therefore
$
    \|\omega_{T_\square}\|
    \le
    \|\hat D^{-1}\|
    (\|\psi_{T_\square}\|+\|\hat c\|\,|\theta_{T_\square}|)
    =
    O(\eta)
$.
Both methods thus reach an $O(\eta)$-neighborhood of the river equilibrium $(\theta,\omega)=(0,0)$, but their behavior there is qualitatively different.
Vanilla GD is force-proportional and can continues to refine as the residual
forces vanish (Remark~\ref{remark:GD-discuss}, Proposition~\ref{proposition:late_stage_behavior_in_relative_tube}). 
By contrast, Muon remains sign-based and non-smooth; with a fixed step size, it may overshoot the equilibrium and generate persistent oscillations (Remark~\ref{remark:Muon-discuss}, Proposition~\ref{proposition:late_stage_behavior_in_relative_tube}). 
Consequently, simplified Muon typically reaches the neighborhood faster, whereas GD attains better final accuracy, consistent with Figure~\ref{figure:reduction_on_invariant_manifold} (also see the Lyapunov analyses in Remarks~\ref{remark:GD-Lyapunov} and ~\ref{remark:Muon-Lyapunov}).

\section{River-Valley as Tool in Generalized Settings}
\label{sec:river-valley-as-tool-in-generalized-settings}

In the previous sections, we have mostly focused on the simplified Muon without momentum and our specific \ref{equation:main_mixed_spectrum_matrix_objfunc} model to provide fine-grained analysis and explicit trajectory dynamics. However, we argue that the river--valley framework can extend beyond these simplified assumptions and help validate our insights in more generalized settings. Let's generalize the problem\footnote{$X$ can be easily generalized to a non-square matrix, but we keep the shape to be squared for simplicity.} to optimize over an objective function $\mathcal{F}(X): \mathbb{R}^{n \times n} \to \mathbb{R}$ to search for a critical point of interest $X^\star \in \mathbb{R}^{n \times n}$. This point could be the global optimum or any point that generalizes well. Suppose our target $X^\star$ admits a singular value decomposition (SVD) as $X^\star=U_\star\Sigma_\star V_\star^\top$, then following suit of our river-valley analysis, we define the spectral river as follows:
\begin{equation}
	\mathcal{R}^\star
	:=
	\left\{
	X^\star + \mathcal{B}(x)
	\;\mid\; 
	x\in\mathbb{R}^n
	\right\},
	\qquad
	\mathcal{B}(x):=U_\star \mathtt{Diag}(x) V_\star^\top,
	\nonumber
\end{equation}
such that along the river, we are only away from the target $X^\star$ up to a spectral scale difference, so only the singular values are different. 

In this section, we first show, once the iterate enters the spectral river, the initial convergence speed of Muon is faster than that of GD, whereas in the final phase close to $X^\star$, Muon could easily overshoot and oscillate, hindering convergence. These phenomena reflect the same set of observations made in the mixed-spectrum MS case, which indicates a broader applicability of our theory. Our generalized proofs hold contingent on two basic requirement: 1) the stored momentum $M^{\rm Muon/GD}_0$ at $X_0$ has to point mostly in the river direction and 2) the river is approximate gradient-invariant near $X^\star$. Here, gradient-invariance near $X^\star$ holds if, for all sufficiently small $x$, $\nabla \mathcal{F}(X^\star+\mathcal{B}(x))\in \mathcal{R}^\star$. When this condition is met, it holds that $\nabla^2 f(X^\star) [\mathcal{B}(x)] = \mathcal{B} (Kx)$ for a fixed $K$ (please find details in Appendix~\ref{subsec:spectral-river-geometry}). With these assumptions in mind, we first present the informal early-stage result:

\begin{theorem}[Early-stage Descent along the River, informal]
	Suppose the iterate $X=X^\star+\mathcal{B}(x)$ lies on the spectral river, and define the river-restricted objective $f(x):=\mathcal{F}(X^\star+\mathcal{B}(x))-\mathcal{F}(X^\star)$.
	Let $L_{\mathrm{full}}$ denote a smoothness constant of $\mathcal{F}$ in the ambient matrix space, and let $L_{\mathrm{riv}}$ denote a smoothness constant of $f$ restricted to the river. Then, under the assumption that the orthogonalized momentum after update is mostly aligned with $\nabla f(x)$, we show that the one-step decreases of Muon and Momentum GD respectively satisfy
	\begin{equation}
		\text{Muon decrease}
		\;\gtrsim\;
		\frac{\rho^2\|\nabla f(x)\|_1^2}{2L_{\rm riv} n},
		\qquad
		\text{Momentum GD decrease}
		\;\lesssim\;
		\frac{\|\nabla f(x)\|_2^2}{2L_{\rm full}}.
		\nonumber
	\end{equation}
	Here, $\rho\in(0,1]$ measures the alignment between the effective momentum direction and the river gradient $\nabla f(x)$, and the ratio $\|\nabla f(x)\|_1 / \|\nabla f(x)\|_2$ quantifies how uniformly the gradient mass is distributed across coordinates (i.e., the degree of non-sparsity).
\end{theorem}
The formal statements and its associated proof appear in Theorem~\ref{theorem:early-stage-muon-acceleration-along-the-spectral-river}. 
Since $L_{\rm riv}$ is the smoothness constant of the objective restricted to the spectral river, whereas $L_{\rm full}$ is the ambient smoothness constant over all directions, one always has $L_{\rm riv}\le L_{\rm full}$. Moreover, when the local Hessian is much flatter on the river than on its transverse hill complement, we have $L_{\rm riv}\ll L_{\rm full}$, thereby making Muon's decrease much bigger. 
We refer to Remark~\ref{remark:when-river-smoothness-ll-hill-smoothness} for further discussion.

Then, we turn to the late-stage analysis of starting from a point $X_0$ along the river that is close to $X^\star$ (i.e., already in the final phase of convergence). We show that momentum GD could provide steady decrease in loss whereas Muon with momentum diverges in two steps. Before presenting the results, we discuss why this would happen. For a proposed update $X\gets X-\eta D$, Taylor expansion gives
\begin{equation}
	\mathcal{F}(X-\eta D)
	=
	\mathcal{F}(X)-
	\eta
	\left\langle \nabla \mathcal{F}(X), D \right\rangle
	+
	\frac{\eta^2}{2}
	\left\langle D, \nabla^2 \mathcal{F}(X)[D] \right\rangle
	+ 
	\mathcal O(\eta^3 \|D\|_F^3).
	\nonumber
\end{equation}
If we assume that $\eta$ is sufficiently small such that $\mathcal{O}(\eta^3) \approx 0$, then the quadratic model decreases when $ 0<\eta<\eta_D^{\max}(X)$ with $\eta_D^{\max}(X)
:=2\langle \nabla \mathcal{F}(X), D \rangle/\langle D, \nabla^2 \mathcal{F}(X)[D] \rangle$, provided $\langle \nabla \mathcal{F}(X), D \rangle>0$ and $\langle D, \nabla^2 \mathcal{F}(X)[D] \rangle>0$. At a river point $X=X^\star+a\mathcal{B}(x)$, suppose gradient invariance holds such that $\nabla \mathcal{F}(X) \approx a\mathcal{B}(Kx)$. Regardless of the momentum, for the spectrally normalized muon update $D=\mathcal{B}(\mathbf{1})$, the one-step safe range is
\begin{equation}
	0<\eta< \eta_D^{\max}(X) 
	= 2 
	\frac{
		\ip{a \mathcal{B}(Kx)
		}{
			\mathcal{B}(\mathbf{1})
	}}{\ip{\mathcal{B}(\mathbf{1})}{\mathcal{B}(K \mathbf{1})}}
    = 2a y,
	\qquad
	y := 
	\frac{\mathbf{1}^\top Kx}{\mathbf{1}^\top K\mathbf{1}}.
	\label{equation:muon-one-step-safe}
\end{equation}
If momentum keeps the same polar direction for two consecutive Muon steps, then two steps give the displacement $2\eta\mathcal{B}(\mathbf{1})$, and the two-step safe range is approximately $0<\eta<a y$, thus Muon's admissible step-size scales with $\mathcal{O}(a)$, which necessitates a slowing progress. For GD, however, we can easily show that the admissible step-size is agnostic of residual scale $a$ (see Proposition~\ref{proposition:quad-separation} as an illustration). This means that it is the contraction of admissible step-size that is problematic with Muon. We also prove in Proposition~\ref{proposition:two-step-descent-for-Adam-like-method} that Adam could behave more like GD in the sense that it does not require a shrinking step-size of $\mathcal{O}(a)$ in the final phase, which means that Adam and its variants could be a better choice for late-stage convergence. This serves as a theoretical footnote as to why manual LR decay is necessary for Muon. However, a manually prescribed decay schedule could create mismatch with the actual optimization progress, calling for more principled approaches (see Appendix~\ref{subsubsec:guidance-algorithm-design} for details). We believe an alternative, or improvement to explicit scheduling is to allow for a two-stage approach where Muon is only used in the first phase, followed by a residual-agnostic refinement optimizer. To make our arguments more precise, we now present the late-stage behavior of Muon as an informal theorem:

\begin{theorem}[Late-stage Convergence on River, informal]
	Assume $X^\star$ is a nondegenerate local minimizer, and consider an initialization of the form $X_0 = X^\star + \mathcal{B}(a x_0)$ with $0<a\ll 1$, so that $X_0$ lies on the spectral river and is already close to $X^\star$. Suppose further that the river is approximately gradient-invariant (see Lemma~\ref{lemma:river-reduction}) near $X^\star$, and that the stored momentum remains dominated by river directions. 
	Define $y_0$ as per \eqref{equation:muon-one-step-safe} and $\eta^{\rm Muon}=(1+\varepsilon)a y_0$ with $0<\varepsilon<1$.
	Then momentum Muon decreases the objective at the first step but overshoots by the second step:
	\begin{equation}
		\mathcal{F}(X_1^{\mathrm{Muon}})
		<
		\mathcal{F}(X_0)
		<
		\mathcal{F}(X_2^{\mathrm{Muon}}).
		\nonumber
	\end{equation}
	By contrast, starting from the same point and with the same stored momentum, momentum gradient descent admits a residual-independent (independent of $a$) step-size threshold $\eta_0=O(1)$ such that for all sufficiently small $0<\eta^{\mathrm{GD}}\le \eta_0$,
	\begin{equation}
		\mathcal{F}(X_2^{\mathrm{GD}})
		\le
		\mathcal{F}(X_1^{\mathrm{GD}})
		\le
		\mathcal{F}(X_0).
		\nonumber
	\end{equation}
	Since $\eta^{\rm Muon}=O(a)$, one has $\eta^{\rm Muon}<\eta_0$ for sufficiently small $a$.  
\end{theorem}
Therefore, even the same numerical step size $\eta^{\rm GD}=\eta^{\rm Muon}$ is safe for GD, while Muon with step size $\eta^{\rm Muon}$ can overshoot after two persistent polar steps.  In addition, GD can often take a much larger step size than Muon in the final river stage. See Appendix~\ref{subsubsec:general-model} for proof and more details.

\section{Empirical Evidence: for Muon as an Early-stage Explorer}
\label{sec:empirical-evidence-for-Muon-as-an-early-exploration-optimizer}

As discussed in Section~\ref{sec:river-valley-as-tool-in-generalized-settings}, the late-stage convergence of Muon generally requires a decreasing LR schedule, because the admissible step size shrinks as the iterates approach a local solution. Although Muon can perform well with a carefully tuned LR schedule, finding such a schedule may require substantial experimentation, which becomes particularly costly at the scale of modern LLMs. Motivated by our theoretical analysis, we provide preliminary empirical evidence that Muon is most effective as an early-stage exploration optimizer, followed by a refinement optimizer such as AdamW in the late stage.
All the LLM pre-training experiments\footnote{Experimental setting adapted from \url{https://github.com/epfml/schedules-and-scaling}.} are conducted by training a 250M-parameter LLaMA-style Transformer~\citep{touvron2023llama} from scratch on OpenWebText2~\citep{gao2020pile}.
All runs use the same compute budget and are executed for 15k or 20k iterations. Several Muon-only baselines (with linear and \texttt{cos\_inf} LR schedules) are compared against two-stage hybrid methods that train with Muon using a constant LR and then switch to AdamW at different iterations. 
Additional details of the LLM pre-training experiments are provided in Appendix~\ref{subsec:empirical_observations_on_llm_pretraining}. 

\begin{figure*}[ht]
    \centering
    \includegraphics[width=0.9\textwidth]{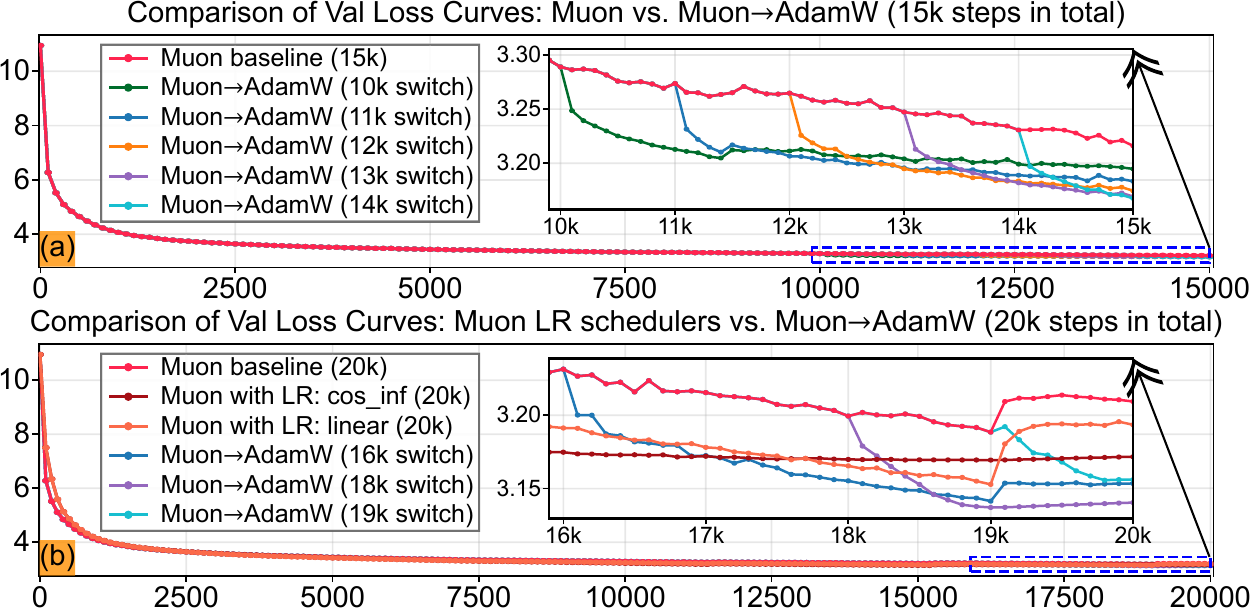}
    \caption{
        \textbf{Validation loss curves for LLM pre-training with a ``Muon$\to$AdamW'' switch.}
        \textbf{(a)} 15k-step Muon baseline and Muon$\to$AdamW runs that switch at steps 10k, 11k, 12k, 13k and 14k under a constant LR (no LR decay).
        \textbf{(b)} 20k-step Muon-only baselines with constant, \texttt{cos\_inf}, and linear LR schedulers, plus Muon$\to$AdamW runs switching at 16k, 18k, and 19k. Each hybrid matches the constant-LR Muon trajectory up to the switch step and then continues with AdamW.
    }
    \label{figure:llm_pretrain_comparison_val}
\end{figure*}

Figure~\ref{figure:llm_pretrain_comparison_val} compares Muon$\to$AdamW runs with different switching times.
Panel~(a) shows that the validation loss generally decreases as the switch is delayed from 10k to 14k steps, suggesting that a later transition can be beneficial.
However, the benefit of delaying the switch is expected to diminish once the switch point is pushed too close to the end of training.
Panel~(b) extends training to 20k iterations and shows that the hybrid runs attain lower final validation losses than the Muon-only baselines equipped with either linear or \texttt{cos\_inf} LR schedulers.
Overall, these results are consistent with our theory.
As the iterates approach a local solution, Muon admits a smaller effective step size.
Although carefully tuned decreasing LR schedules can mitigate oscillations, Muon may consequently exhibit slower late-stage convergence.
By contrast, after switching to AdamW, the hybrid method follows a more stable loss trajectory and maintains a consistent downward trend, ultimately achieving a lower final validation loss than Muon even with carefully tuned LR schedules.
These experiments further suggest that the hybrid strategy is relatively insensitive to the precise switching time: across a broad range of late-stage switch points, switching to AdamW continues to produce a stable decrease in validation loss, except when the switch is made extremely late.

\section{Conclusion}
Our work complements the growing literature on Muon by identifying not only its theoretical advantage (i.e., fast early-stage exploration) but also its limitation (i.e., poor late-stage exploitation).
Our trajectory-level dynamical analysis demonstrates the utility of the river--valley perspective for disentangling signal-aligned and nuisance directions in the training landscape, while also introducing a useful \ref{equation:main_mixed_spectrum_matrix_objfunc} benchmark testbed.
At the same time, the river--valley decomposition provides only one lens on modern training dynamics; integrating this view with related phenomena, such as edge-of-stability behavior, is an important direction for future work.
Finally, our preliminary experiments suggest that multi-stage optimizer schedules, in which one switches optimizers during pre-training, may be a promising direction for improving training efficiency and final performance.

\section*{Acknowledgments}
We thank the anonymous reviewers and area chair for their constructive comments. This work was supported by the National Natural Science Foundation of China (Grant No.~62506314), the Research Grants Council of Hong Kong (Grant No.~21208525), and City University of Hong Kong (Project No.~9382001). 
We extend our gratitude to the participants for their time and interest, which has greatly contributed to the research. 
We are also thankful to Tongle Wu, a PhD candidate at the Pennsylvania State University, University Park, for his insightful discussions and guidance during the early stage of this project.
Special thanks to Lu Huang, a PhD candidate at Peking University, whose expertise in creating the visual content for this paper was invaluable.

\bibliographystyle{unsrtnat}
\bibliography{reference}

\newpage
\appendix

\renewcommand{\contentsname}{Appendix Contents} 
\setcounter{tocdepth}{3} 
\startcontents[appendix] 
\printcontents[appendix]{}{1}{\section*{Appendix Contents}} 

\section{Additional Details for Section~\ref{sec:mixed-spectrum-matrix-sensing-model}}
\label{sec:appendix-mixed-spectrum-matrix-sensing}
In this section, we note that Muon's $\mathtt{msign}$ operation effectively
combines two distinct effects: (i) flattening the singular-value spectrum and
(ii) normalizing the gradient magnitude. To disentangle these two effects,
following the setup of~\citep{vasudeva2026muon,fan2026implicit}, we express the
ODE flows considered below in the following unified direction--speed form:
\begin{equation}
    \dot{X} = -s(X)d(X), \qquad \|d(X)\|_F = 1
    \qquad\Longleftrightarrow\qquad
    \frac{dX}{dl} = -d(X), \qquad \frac{dl}{dt} = s(X),
    \nonumber
\end{equation}
where $d(X)$ denotes the update direction, $s(X)$ denotes the speed, and $l$
denotes the arc length of the trajectory.
The resulting direction--speed decompositions are summarized in Table~\ref{table:direction-speed-decomp-flow}.
Comparing Muon with NGD, or SMuon with GD, isolates the effect of flattening the singular-value spectrum, since the two methods in each pair have the same speed but different directions. In contrast, comparing Muon with SMuon, or NGD with GD, isolates the effect of gradient magnitude normalization, since the methods in each pair have the same direction but different speeds.

\begin{table*}[ht]
    \centering
    \caption{Direction--speed decomposition of the considered optimization flows.}
    \label{table:direction-speed-decomp-flow}
    \resizebox{0.7\linewidth}{!}{
        \begin{threeparttable} 
            \begin{tabular}{lll}
                \toprule
                \textbf{Optimizer} & \textbf{Direction $d(X)$} & \textbf{Speed $s(X)$} \\
                \midrule
                GD & $\nicefrac{\nabla h(X)}{\|\nabla h(X)\|_F}$ & $\|\nabla h(X)\|_F$ \\
                Normalized GD (NGD) & $\nicefrac{\nabla h(X)}{\|\nabla h(X)\|_F}$ & $1$ \\
                Muon & $\nicefrac{\mathtt{msign}(\nabla h(X))}{\|\mathtt{msign}(\nabla h(X))\|_F}$ & $1$ \\
                Scaled Muon (SMuon) & $\nicefrac{\mathtt{msign}(\nabla h(X))}{\|\mathtt{msign}(\nabla h(X))\|_F}$ & $\|\nabla h(X)\|_F$ \\
                \bottomrule
            \end{tabular}
            \begin{tablenotes}  
				\footnotesize          
				\item[1] Here, $\|\mathtt{msign}(\nabla h(X))\|_F=1$.
			\end{tablenotes}  
		\end{threeparttable} 
    }
\end{table*}

Throughout this section, none of the four optimizers uses momentum, e.g., we set $\xi_t=0$ in~\ref{definition:muon-update}. Accordingly, we sometimes refer to GD as \emph{vanilla GD} and Muon as \emph{simplified Muon}. Nevertheless,
a direct comparison between simplified Muon and vanilla GD is also meaningful,
and we use this comparison when the combined effect of direction and speed is of
interest.

\subsection{Properties of the Mixed Spectrum Sensing Model}
\label{subsec:property_of_mixed_spectrum_ms}

\begin{definition}[Gaussian Orthogonal Ensemble]
\label{definition:gaussian_orthogonal_ensemble}
	A symmetric random matrix $G\in\mathbb{R}^{n\times n}$ is said to follow the normalized Gaussian Orthogonal Ensemble, denoted by $G\sim \mathrm{GOE}(n)$, if
	\begin{equation}
		G=\frac{1}{\sqrt n}
		\begin{bmatrix}
			g_{11} & g_{12} & \cdots & g_{1n} \\
			g_{21} & g_{22} & \cdots & g_{2n} \\
			\vdots & \vdots & \ddots & \vdots \\
			g_{n1} & g_{n2} & \cdots & g_{nn} 
		\end{bmatrix},
		\qquad g_{ij}=g_{ji},
		\nonumber
	\end{equation}
	where the upper-triangular entries $\{g_{ij}:1\leq i\leq j\leq n\}$ are independent and satisfy
	\begin{equation}
		g_{ij}\sim \mathcal N(0,1)\quad (i<j),
		\qquad
		g_{ii}\sim \mathcal N(0,2).
		\nonumber
	\end{equation}
	Equivalently,
	\begin{equation}
		\mathbb E[G_{ij}]=0,\qquad
		\mathbb E[G_{ij}^2]=\frac{1}{n}\quad (i\neq j),
		\qquad
		\mathbb E[G_{ii}^2]=\frac{2}{n}.
		\nonumber
	\end{equation}
	In particular, for every fixed symmetric matrix $M\in\mathbb{R}^{n\times n}$,
	\begin{equation}
		\mathbb E\!\left[\langle G,M\rangle^2\right]
		=
		\frac{2}{n}\|M\|_F^2 .
		\nonumber
	\end{equation}
	When $\{G_i\}_{i=1}^m$ are used as sensing matrices, they are taken to be
	independent copies of $G$.
\end{definition}

\begin{lemma}[Population loss and its gradients]
\label{lemma:population_loss_and_its_gradients}
	Consider the mixed spectrum sensing model
	\begin{equation}
		A = G + 
		\pi_\star P_\star + \pi_\diamond P_\diamond + \pi_\circ P_\circ,
		\nonumber
	\end{equation}
	where $G\sim\mathrm{GOE}(n)$, $\pi=(\pi_\star,\pi_\diamond,\pi_\circ)^\top\sim
	\mathcal N(0,\Sigma)$, and $G$ is independent of $\pi$. Based on our construction, $P_\star,P_\diamond,P_\circ$ are mutually orthogonal projection matrices. Let
	\begin{equation}
		\mathcal{L}(X)
		:= \mathcal{F}(XX^\top)
		:=
		\frac12\mathbb E\!\left[
		\left\langle A,XX^\top-P_\star\right\rangle^2
		\right],
		\qquad
		\mathcal{F}(M) = \mathcal{F}(XX^\top),
		\nonumber
	\end{equation}
	with square matrix variable $M:=XX^\top\in\mathbb{R}^{n\times n}$ (due to the BM-factorized MS form). Then
	\begin{equation}
	\mathcal F(M)
		=
		\frac{1}{n}\|M-P_\star\|_F^2
		+
		\frac12 q(M)^\top \Sigma q(M),
		\label{equation:population_loss}
	\end{equation}
	where
	\begin{equation}
		q(M)
		:=
		\begin{bmatrix}
			\langle P_\star, M-P_\star\rangle \\
			\langle P_\diamond, M\rangle \\
			\langle P_\circ, M\rangle
		\end{bmatrix}.
		\nonumber
	\end{equation}
	Moreover, the Frobenius gradient of $\mathcal F$ with respect to $M$ is
	\begin{equation}
		G(M) :=
		\nabla_M \mathcal{F}(M)
		=
		\frac{2}{n}(M-P_\star)
		+
		g_\star P_\star
		+
		g_\diamond P_\diamond
		+
		g_\circ P_\circ,
		\label{equation:population_gradient}
	\end{equation}
	where
	\begin{equation}
		g(M):=
		\begin{bmatrix}
			g_\star\\
			g_\diamond\\
			g_\circ
		\end{bmatrix}
		=
		\Sigma q(M)
		\in \mathbb{R}^3.
		\nonumber
	\end{equation}
	Consequently, under the standard Frobenius-gradient convention,
	\begin{equation}
		\nabla_X \mathcal{L}(X)
		=
		2\,G(XX^\top)X.
		\nonumber
	\end{equation}
\end{lemma}

\begin{proof}
	Let $\Delta:=M-P_\star$. Since $G$ and $\pi$ are independent and centered,
	\begin{equation}
		\mathcal F(M)
		=
		\frac{1}{2}
		\mathbb E
		\left[
		\left(
		\langle G,\Delta\rangle
		+
		\pi^\top q(M)
		\right)^2
		\right]
		=
		\frac12\mathbb E\!\left[\langle G,\Delta\rangle^2\right]
		+
		\frac12\mathbb E\!\left[(\pi^\top q(M))^2\right].
		\nonumber
	\end{equation}
	Because $\Delta$ is symmetric, the GOE covariance identity gives
	\begin{equation}
		\mathbb E\!\left[\langle G,\Delta\rangle^2\right]
		=
		\frac{2}{n}\|\Delta\|_F^2.
		\nonumber
	\end{equation}
	Since $\pi\sim\mathcal N(0,\Sigma)$,
	\begin{equation}
		\mathbb E\!\left[(\pi^\top q(M))^2\right]
		=
		q(M)^\top \Sigma q(M).
		\nonumber
	\end{equation}
	Combining the two displays yields~\eqref{equation:population_loss}.
	
	It remains to compute the gradient. For any symmetric perturbation $\Delta M$, we have
	\begin{equation}
		\mathrm{d} q(M)[\Delta M]
		=
		\begin{bmatrix}
			\langle P_\star,\Delta M \rangle\\
			\langle P_\diamond,\Delta M \rangle\\
			\langle P_\circ,\Delta M \rangle
		\end{bmatrix}.
		\nonumber
	\end{equation}
	Using the symmetry of $\Sigma$,
	\begin{equation}\begin{aligned}
			\mathrm{d} \mathcal{F}(M)[\Delta M]
			&=
			\frac{2}{n}\langle M-P_\star,\Delta M \rangle
			+
			(\Sigma q(M))^\top \mathrm{d} q(M)[\Delta M]  \\
			&=
			\left\langle
			\frac{2}{n}(M-P_\star)
			+
			g_\star P_\star
			+
			g_\diamond P_\diamond
			+
			g_\circ P_\circ,
			\Delta M
			\right\rangle,
			\nonumber
	\end{aligned}\end{equation}
	where $g(M)=\Sigma q(M)$. This proves~\eqref{equation:population_gradient}.
	
	Finally, since $M=XX^\top$, for any perturbation $\Delta X$ of $X$,
	\begin{equation}
		\mathrm{d} M[\Delta X] = X (\Delta X)^\top + (\Delta X) X^\top.
		\nonumber
	\end{equation}
	Therefore,
	\begin{equation}
		\mathrm{d} \mathcal{L}(X)[\Delta X]
		=
		\left\langle
		G(M), X (\Delta X)^\top + (\Delta X) X^\top
		\right\rangle  
		=
		2\left\langle
		G(M)X, \Delta X
		\right\rangle,
		\nonumber
	\end{equation}
	where we used that $G(M) = \nabla_M \mathcal{F}(M)$ is symmetric. Hence
	\begin{equation}
		\nabla_X\mathcal L(X)
		=
		2\,\nabla_M\mathcal F(XX^\top)X
		=
		2\,G(XX^\top)X.
		\nonumber
	\end{equation}
	This completes the proof.
\end{proof}

\subsection{Reduced Dynamics of the Mixed-spectrum Matrix Sensing Model}
\label{subsec:dynamic_of_reduced_mixed_spectrum_ms}

In the reduced dynamics, we write $\kappa:=\hat K/r=2I_3/n+r\Sigma$ whenever an explicit diagonal formula is given.  
Using $\hat K$ instead only rescales time and the stepsize by the fixed rank $r$.

\begin{lemma}[Vanilla GD dynamics on the invariant manifold $\mathcal{M}$]
	\label{lemma:gd_dynamics}
	Let
	\begin{equation}
		\mathcal{M}
		:=
		\left\{
		M=\alpha P_\star+\beta P_\diamond+\gamma P_\circ:
		\alpha,\beta,\gamma\in\mathbb R
		\right\}.
		\nonumber
	\end{equation}
	If $M \in \mathcal{M}$ with $M = \alpha P_\star + \beta P_\diamond + \gamma P_\circ$, then the population matrix gradient satisfies
	\begin{equation}
		G(M) =
		u P_\star + v P_\diamond + w P_\circ
		\in \mathcal{M},
		\nonumber
	\end{equation}
	where
	\begin{equation}
		\begin{bmatrix}
			u \\ v \\ w
		\end{bmatrix}
		=
		\frac{2}{n}
		\begin{bmatrix}
			\alpha-1 \\ \beta \\ \gamma
		\end{bmatrix}
		+
		\Sigma
		\begin{bmatrix}
			r (\alpha-1) \\
			r \beta \\
			r \gamma
		\end{bmatrix}.
		\nonumber
	\end{equation}
	Moreover, the gradient flow induced by the Burer-Monteiro factorization satisfies
	\begin{equation}
		\dot{M} = -2 \left[ G(M)M + M\,G(M) \right],
		\nonumber
	\end{equation}
	and $\mathcal{M}$ is invariant under this flow. 
	
	More explicitly, we have the following ODE system if $\alpha=a^2$, $\beta=b^2$, and $\gamma=c^2$:
	\begin{equation}
		\begin{cases}
			\dot\alpha = -4u\alpha,\\
			\dot\beta  = -4v\beta,\\
			\dot\gamma = -4w\gamma.
		\end{cases}
		\Rightarrow\quad
		\begin{cases}
			\dot a = -2ua,\\
			\dot b = -2vb,\\
			\dot c = -2wc.
		\end{cases}
		\nonumber
	\end{equation}
\end{lemma}

\begin{proof}
	By the expression of the population gradient \eqref{equation:population_gradient}, where
	\begin{equation}
		\begin{bmatrix}
			g_\star \\
			g_\diamond \\
			g_\circ
		\end{bmatrix}
		=
		\Sigma
		\begin{bmatrix}
			\langle P_\star,M-P_\star\rangle\\
			\langle P_\diamond,M\rangle\\
			\langle P_\circ,M\rangle
		\end{bmatrix}.
		\nonumber
	\end{equation}
	Since $P_\star,P_\diamond,P_\circ$ are mutually orthogonal projections with
	ranks $r,r,r$, we have
	\begin{equation}
		\langle P_\star,M-P_\star \rangle
		=
		r (\alpha-1),
		\qquad
		\langle P_\diamond,M \rangle
		=
		r \beta,
		\qquad
		\langle P_\circ,M \rangle
		=
		r \gamma.
		\nonumber
	\end{equation}
	Substituting these identities into the gradient formula gives
	\begin{equation}
		G(M) = u P_\star + v P_\diamond + w P_\circ,
		\nonumber
	\end{equation}
	with
	\begin{equation}
		\begin{bmatrix}
			u \\ v \\ w
		\end{bmatrix}
		=
		\frac{2}{n}
		\begin{bmatrix}
			\alpha-1 \\ \beta \\ \gamma
		\end{bmatrix}
		+
		\Sigma
		\begin{bmatrix}
			r (\alpha-1) \\
			r \beta \\
			r \gamma
		\end{bmatrix}.
		\nonumber
	\end{equation}
	Thus $G(M)\in\mathcal{M}$.
	
	It remains to derive the closed dynamics. The gradient flow in the factor variable $X\in\mathbb{R}^{n\times r_\mathrm{search}}$ is
	\begin{equation}
		\dot X = -2\, G(M) X,
		\qquad 
		M = XX^\top.
		\nonumber
	\end{equation}
	Therefore,
	\begin{equation}
		\dot{M}
		=
		\dot{X} X^\top + X \dot{X}^\top  \\
		=
		- 2\,G(M)XX^\top - 2\,XX^\top G(M) \\
		=
		- 2\left[ G(M)M + M\, G(M) \right],
		\nonumber
	\end{equation}
	where we used the symmetry of $G(M)$.
	
	Using the projection identities, i.e., $P_i P_j=0$ for any $i\neq j$, and $P_i^2 = P_i$ for any $i,j\in\{\star,\diamond,\circ\}$, we obtain the exchangeability:
	\begin{equation}
		G(M) M
		=
		u\alpha P_\star
		+
		v\beta P_\diamond
		+
		w\gamma P_\circ
		=
		M\, G(M).
		\nonumber
	\end{equation}
	Hence
	\begin{equation}
		\dot{M}
		=
		-4 G(M) M
		=
		-4u\alpha P_\star
		-
		4v\beta P_\diamond
		-
		4w\gamma P_\circ.
		\nonumber
	\end{equation}
	Comparing this with $\dot{M} = \dot\alpha P_\star + \dot\beta P_\diamond + \dot\gamma P_\circ$ yields
	\begin{equation}
		\dot\alpha = -4u\alpha,
		\qquad
		\dot\beta = -4v\beta,
		\qquad
		\dot\gamma = -4w\gamma.
		\nonumber
	\end{equation}
	This proves that the trajectory remains in $\mathcal{M}$.
	
	Finally, if $\alpha=a^2$, $\beta=b^2$ and $\gamma=c^2$, then $\dot\alpha=2a\dot a$, $\dot\beta=2b\dot b$ and $\dot\gamma=2c\dot c$.
	Thus the above equations are equivalently written as
	\begin{equation}
		\dot a=-2ua,\qquad
		\dot b=-2vb,\qquad
		\dot c=-2wc.
		\nonumber
	\end{equation}
	The proof is complete.
\end{proof}

\begin{lemma}[Simplified Muon dynamics on the invariant manifold $\mathcal{X}$]
	\label{lemma:muon_dynamics}
	Let
	\begin{equation}
		\mathcal X
		:=
		\left\{
		X=
		\big[a U\;\; b V\;\; c W\big]:
		a,b,c\in\mathbb R
		\right\},
		\nonumber
	\end{equation}
	where the columns of $U,V,W$ are orthonormal and mutually orthogonal. For any
	$X\in\mathcal{X}$, we have
	\begin{equation}
		M := XX^\top
		=
		a^2 P_\star + b^2 P_\diamond + c^2 P_\circ
		\in\mathcal M.
		\nonumber
	\end{equation}
	Therefore,
	\begin{equation}
		G(M) = u P_\star + v P_\diamond + w P_\circ,
		\nonumber
	\end{equation}
	where
	\begin{equation}
		\begin{bmatrix}
			u \\ v \\ w
		\end{bmatrix}
		=
		\frac{2}{n}
		\begin{bmatrix}
			a^2-1 \\ b^2 \\ c^2
		\end{bmatrix}
		+
		\Sigma
		\begin{bmatrix}
			r (a^2-1) \\
			r  b^2 \\
			r  c^2
		\end{bmatrix}.
		\nonumber
	\end{equation}
	Assume that $G(M)X$ has full column rank along the trajectory. Then the simplified Muon flow
	\begin{equation}
		\dot X
		=
		-\mathtt{msign} \left[ G(M)X \right]
		\nonumber
	\end{equation}
	leaves $\mathcal{X}$ invariant. More explicitly, the dynamics on $\mathcal{X}$ are given by
	\begin{equation}
		\begin{cases}
			\dot a = -\mathtt{sign}(ua),\\
			\dot b = -\mathtt{sign}(vb),\\
			\dot c = -\mathtt{sign}(wc).
		\end{cases}
		\nonumber
	\end{equation}
\end{lemma}

\begin{proof}
	Let $X=[aU\;\; bV\;\; cW] \in\mathcal X$. Then, by orthogonality of the three subspaces,
	\begin{equation}
		M=XX^\top
		=
		a^2P_\star+b^2P_\diamond+c^2P_\circ
		\in\mathcal M.
		\nonumber
	\end{equation}
	By the previous lemma~\ref{lemma:gd_dynamics} on the invariant manifold $\mathcal{M}$, the matrix gradient satisfies
	\begin{equation}
		G(M) = u P_\star + v P_\diamond + w P_\circ,
		\nonumber
	\end{equation}
	with
	\begin{equation}
		\begin{bmatrix}
			u \\ v \\ w
		\end{bmatrix}
		=
		\frac{2}{n}
		\begin{bmatrix}
			a^2-1 \\ b^2 \\ c^2
		\end{bmatrix}
		+
		\Sigma
		\begin{bmatrix}
			r (a^2-1) \\
			r b^2 \\
			r c^2
		\end{bmatrix}.
		\nonumber
	\end{equation}
	Using $P_\star U = U,P_\diamond V = V,P_\circ W = W$, and the mutual orthogonality of these subspaces, we obtain
	\begin{equation}
		G(M) X
		=
		(uP_\star+vP_\diamond+wP_\circ)
		\big[aU\;\; bV\;\; cW\big]  \\
		=
		\big[ua U\;\; vb V\;\; wc W\big].
		\nonumber
	\end{equation}
	Thus $G(M) X$ lies in the tangent space of $\mathcal{X}$.
	Since $G(M) X$ is assumed to have full column rank,
	\begin{equation}
		\mathtt{msign}\left(G(M) X\right)
		=
		G(M) X 
		\left[
		X^\top G(M) G(M) X
		\right]^{-1/2}.
		\nonumber
	\end{equation}
	Moreover,
	\begin{equation}
		\left[
		X^\top G(M) G(M) X
		\right]
		=
		(ua)^2 I
		\oplus
		(vb)^2 I
		\oplus
		(wc)^2 I.
		\nonumber
	\end{equation}
	where $\oplus$ denotes a block-diagonal concatenation matching their column-blocks.
	Hence
	\begin{equation}
		\left[
		X^\top G(M) G(M) X
		\right]^{-1/2}
		=
		\frac{1}{|ua|}I
		\oplus
		\frac{1}{|vb|}I
		\oplus
		\frac{1}{|wc|}I,
		\nonumber
	\end{equation}
	and therefore
	\begin{equation}\begin{aligned}
			\mathtt{msign}\left(G(M) X\right)
			&=
			\big[ua U\;\; vb V\;\; wc W\big]
			\left(
			\frac{1}{|ua|}I
			\oplus
			\frac{1}{|vb|}I
			\oplus
			\frac{1}{|wc|}I
			\right) \\
			&=
			\big[
			\mathtt{sign}(ua) U\;\;
			\mathtt{sign}(vb) V\;\;
			\mathtt{sign}(wc) W
			\big].
			\nonumber
	\end{aligned}\end{equation}
	We thus can find the trajectory stays in $\mathcal{X}$ due to
	\begin{equation}
		\dot X = -\mathtt{msign}(G(M)X).
		\nonumber
	\end{equation}
	Combining this with
	\begin{equation}
		X(t)
		=
		\big[
		a(t)U\;\; b(t)V\;\; c(t)W
		\big]
		\quad\Rightarrow\quad
		\dot X(t) = [\dot a(t) U\;\; \dot b(t) V\;\; \dot c(t) W]
		\nonumber
	\end{equation}
	and matching the coefficients along the mutually orthogonal blocks $U,V,W$, we get
	\begin{equation}
		\dot a = -\mathtt{sign}(ua),
		\qquad
		\dot b = -\mathtt{sign}(vb),
		\qquad
		\dot c = -\mathtt{sign}(wc).
		\nonumber
	\end{equation}
	Therefore, $\dot X$ remains tangent to $\mathcal{X}$.
	The proof is complete.
\end{proof}

\begin{lemma}[Normalized GD and scaled Muon dynamics on $\mathcal X$]
\label{lemma:ngd_smuon_dynamics}
Let $X=[aU\; bV\; cW]\in\mathcal X$ and let $u,v,w$ be as in
Lemma~\ref{lemma:muon_dynamics}.  Define the reduced factor-gradient and
its Euclidean norm by
\begin{equation}
    z:=[a,b,c]^\top,
    \qquad
    q(z):=[ua,vb,wc]^\top,
    \qquad
    \chi(z):=2\|q(z)\|_2,
    \nonumber
\end{equation}
and define $s(z):=[\mathtt{sign}(ua),\mathtt{sign}(vb), \mathtt{sign}(wc)]^\top$ and $m(z):=\|s(z)\|_2=\sqrt{3}$.  
On every open sign cell where $\chi(z)>0$ and $m(z)>0$, 
the four direction--speed flows in Table~\ref{table:direction-speed-decomp-flow} reduce, up to fixed rank factors
absorbed into the time or step scale, to
\begin{equation}
    \begin{aligned}
        \mathrm{GD}: \dot{z}=-2q(z), 
        &\qquad
        \mathrm{NGD}: \dot{z}=-\nicefrac{2q(z)}{\chi(z)},
        \\
        \mathrm{Muon}: \dot{z}=-s(z),
        &\qquad
        \mathrm{SMuon}:\dot{z}=-\left[\nicefrac{\chi(z)}{m(z)}\right] s(z).
    \end{aligned}
    \label{equation:ngd_smuon_reduced_factor_flows}
\end{equation}
Consequently, with $\tau$ denoting the time variable of the corresponding
unscaled flow,
\begin{equation}
\frac{d\tau_{\rm NGD}}{dt}=\frac{1}{\chi(z(t))},
\qquad
\frac{d\tau_{\rm SMuon}}{dt}=\frac{\chi(z(t))}{m(z(t))},
\label{equation:ngd_smuon_time_reparam}
\end{equation}
and the continuous-time NGD trajectory is the GD trajectory reparameterized by time, whereas
the continuous-time SMuon trajectory is the Muon trajectory reparameterized by time.  At a
switching surface, these identities hold in the Filippov sense.
\end{lemma}

\begin{proof}
The factor gradient is $\nabla_X\mathcal L=2G(M)X$ and, on $\mathcal X$,
its three blocks are $2uaU$, $2vbV$, and $2wcW$.  Thus its reduced Euclidean
coordinate vector is $2q(z)$ and its Frobenius norm is $\sqrt r\,\chi(z)$;
the fixed factor $\sqrt r$ is absorbed into the stepsize.  This gives the GD and NGD rows of
\eqref{equation:ngd_smuon_reduced_factor_flows}.  Lemma~\ref{lemma:muon_dynamics}
gives the Muon row.  Normalizing the Muon polar factor gives the fixed factor
$1/m(z)$, and multiplying it by the gradient norm gives the SMuon row.
Each new vector field is a nonnegative scalar multiple of its parent vector
field, which proves \eqref{equation:ngd_smuon_time_reparam} and the claimed
continuous-time trajectory equivalence.  At $\chi=0$ we use the standard
zero-update convention; the normalized field is otherwise understood only on
the open cells specified above.
\end{proof}

\begin{remark}[Diagonal special case: closed-form solutions of vanilla GD/simplified Muon]
\label{remark:diagonal_special_case_gd_muon}
	Recall that we have defined 
	$
		\kappa
		:=
		\frac{1}{r} \hat{K}
		=
		\frac{2}{n}I_3+r\Sigma
	$.
	Assume in this remark that $\Sigma$ is diagonal. Then $\hat{K}$ and $\kappa$ are also diagonal; write
	$
		\kappa
		=
		\mathtt{Diag}
		(
			\kappa_{11},
			\kappa_{22},
			\kappa_{33}
		)
	$.
	Since $\Sigma$ is a covariance matrix in our model, its diagonal entries are nonnegative, and hence
	$
		\kappa_{11},\kappa_{22},\kappa_{33}>0
	$.
	In this case,
	\begin{equation}
		u=\kappa_{11}(\alpha-1),
		\qquad
		v=\kappa_{22}\beta,
		\qquad
		w=\kappa_{33}\gamma.
		\nonumber
	\end{equation}
	Therefore the three coordinates are completely decoupled.
	
	For vanilla GD, the reduced ODE becomes
	\begin{equation}
		\begin{cases}
			\dot\alpha
			=
			4\kappa_{11}\alpha(1-\alpha),
			\\
			\dot\beta
			=
			-4\kappa_{22}\beta^2,
			\\
			\dot\gamma
			=
			-4\kappa_{33}\gamma^2.
		\end{cases}
        \label{equation:gd_diagonal_ode}
	\end{equation}
	For any initialization $\alpha_0,\beta_0,\gamma_0>0$, the solution is
	\begin{equation}
		\begin{cases}
			\alpha_{\rm GD}(t)
			=
			\frac{1}{
				1+
				\left((1-\alpha_0) / \alpha_0\right)
				e^{-4\kappa_{11} t}
			},
			\\
			\beta_{\rm GD}(t)
			=
			\frac{\beta_0}{
				1+4\kappa_{22} \beta_0 t
			},
			\\
			\gamma_{\rm GD}(t)
			=
			\frac{\gamma_0}{
				1+4\kappa_{33} \gamma_0 t
			}.
		\end{cases}
        \label{equation:gd_diagonal_solution}
	\end{equation}
	Therefore, the signal coordinate $\alpha$ follows a logistic-type trajectory and converges exponentially to $1$, whereas the nuisance coordinates $\beta$ and $\gamma$ decay polynomially.
	
	For simplified Muon, the reduced differential inclusion becomes
	\begin{equation}
		\begin{cases}
			\dot\alpha
			\in
			2\sqrt{\alpha}\,
			\mathtt{sign}(\kappa_{11} (1-\alpha))
            =
            -2\sqrt{\alpha}\,
			\mathtt{sign}(\alpha-1),
			\\
			\dot\beta
			\in
			-2\sqrt{\beta}\,
			\mathtt{sign}(\kappa_{22} \beta)
            =
            -2\sqrt{\beta}\,
			\mathtt{sign}(\beta),
			\\
			\dot\gamma
			\in
			-2\sqrt{\gamma}\,
			\mathtt{sign}(\kappa_{33} \gamma)
            =
            -2\sqrt{\gamma}\,
			\mathtt{sign}(\gamma),
		\end{cases}
        \label{equation:muon_diagonal_inclusion}
	\end{equation}
	if $\alpha=a^2$, $\beta=b^2$ and $\gamma=c^2$.
	Under the natural sticking selection at the switching surfaces $\alpha=1$, $\beta=0$, and $\gamma=0$, if $\alpha_0>1$ and $\beta_0,\gamma_0>0$, then
	\begin{equation}
		\begin{cases}
			\alpha_{\rm Muon}(t)
			=
			\left(
			\max
			\left\{
			\sqrt{\alpha_0}-t,\,
			1
			\right\}
			\right)^2,
			\\
			\beta_{\rm Muon}(t)
			=
			\left(
			\max
			\left\{
			\sqrt{\beta_0}-t,\,
			0
			\right\}
			\right)^2,
			\\
			\gamma_{\rm Muon}(t)
			=
			\left(
			\max
			\left\{
			\sqrt{\gamma_0}-t,\,
			0
			\right\}
			\right)^2.
            \label{equation:muon_diagonal_solution}
		\end{cases}
	\end{equation}
	If instead $0<\alpha_0<1$, the first line is replaced by
	\begin{equation}
		\alpha_{\rm Muon}(t)
		=
		\left(
		\min
		\left\{
		\sqrt{\alpha_0}+t,\,
		1
		\right\}
		\right)^2.
		\nonumber
	\end{equation}
	Therefore, Muon reaches the target river coordinate $\alpha=1$ and removes the nuisance coordinates $\beta$ and $\gamma$ in finite continuous time with linear speed.
\end{remark}

\begin{remark}[Diagonal special case: NGD and SMuon]
\label{remark:diagonal_special_case_ngd_smuon}
In the diagonal case of Remark~\ref{remark:diagonal_special_case_gd_muon}, let
\begin{equation}
    \chi(\alpha,\beta,\gamma)
    :=
    2\sqrt{
      \kappa_{11}^2\alpha(\alpha-1)^2
      +\kappa_{22}^2\beta^3
      +\kappa_{33}^2\gamma^3
    }.
    \nonumber
\end{equation}
For positive coordinates, the normalized GD flow is
\begin{equation}
    \begin{cases}
     \dot\alpha_{\rm NGD}
       =-\nicefrac{4\kappa_{11}\alpha(\alpha-1)}{\chi(\alpha,\beta,\gamma)},\\
     \dot\beta_{\rm NGD}
       =-\nicefrac{4\kappa_{22}\beta^2}{\chi(\alpha,\beta,\gamma)},\\
     \dot\gamma_{\rm NGD}
       =-\nicefrac{4\kappa_{33}\gamma^2}{\chi(\alpha,\beta,\gamma)}.
    \end{cases}
    \label{equation:ngd_diagonal_ode}
\end{equation}
Its fixed-step factor update is
\begin{equation}
    \begin{bmatrix}a\\b\\c\end{bmatrix}_{k+1}^{\rm NGD}
    =
    \begin{bmatrix}a\\b\\c\end{bmatrix}_{k}^{\rm NGD}
    -\frac{2\eta}{\chi_k}
    \begin{bmatrix}u_ka_k\\v_kb_k\\w_kc_k\end{bmatrix},
    \qquad
    \chi_k=\chi(\alpha_k,\beta_k,\gamma_k).
    \nonumber
\end{equation}
Hence, if $(\alpha_{\rm GD}(\tau),\beta_{\rm GD}(\tau),\gamma_{\rm GD}(\tau))$ denotes the closed-form GD solution in~\eqref{equation:gd_diagonal_solution}, then
\begin{equation}
    \begin{aligned}
     (\alpha_{\rm NGD}(t),\beta_{\rm NGD}(t),\gamma_{\rm NGD}(t))
     &= (\alpha_{\rm GD}(\tau(t)),\beta_{\rm GD}(\tau(t)),\gamma_{\rm GD}(\tau(t))),\\
     t&=\int_0^{\tau(t)}\chi(\alpha_{\rm GD}(s),\beta_{\rm GD}(s),\gamma_{\rm GD}(s))\,ds .
    \end{aligned}
    \nonumber
\end{equation}
Therefore, the coordinate path is identical to vanilla GD.  
Its unit-speed parameterization removes the vanishing gradient scale: the arc length to the solution is finite, and a fixed discrete step can overshoot once the residual gradient is small.  
In particular, NGD has the vanilla GD geometry but the Muon-like constant-speed late-stage sensitivity.

For SMuon, we have $m(z)=m(a,b,c)=\sqrt{3}$ on a sign cell where all three blocks are
active (the harmless rank-dependent normalization is absorbed into $\eta$).
The reduced flow is
\begin{equation}
    \begin{cases}
        \dot\alpha_{\rm SMuon}
        =&-\frac{2\chi(\alpha,\beta,\gamma)}{m(z)}\sqrt\alpha\,
          \mathtt{sign}(\kappa_{11}(\alpha-1)),\\
        \dot\beta_{\rm SMuon}
        =&-\frac{2\chi(\alpha,\beta,\gamma)}{m(z)}\sqrt\beta\,
          \mathtt{sign}(\kappa_{22} \beta),\\
        \dot\gamma_{\rm SMuon}
        =&-\frac{2\chi(\alpha,\beta,\gamma)}{m(z)}\sqrt\gamma\,
          \mathtt{sign}(\kappa_{33} \gamma).
    \end{cases}
    \label{equation:smuon_diagonal_ode}
\end{equation}
Its fixed-step factor update is
\begin{equation}
    \begin{bmatrix}a\\b\\c\end{bmatrix}_{k+1}^{\rm SMuon}
    =
    \begin{bmatrix}a\\b\\c\end{bmatrix}_{k}^{\rm SMuon}
    -\eta\frac{\chi_k}{m_k}
    \begin{bmatrix}
    \mathtt{sign}(u_ka_k)\\
    \mathtt{sign}(v_kb_k)\\
    \mathtt{sign}(w_kc_k)
    \end{bmatrix}.
    \nonumber
\end{equation}
Equivalently, with the Muon closed-form path from~\eqref{equation:muon_diagonal_solution},
\begin{equation}
    \begin{aligned}
     (\alpha_{\rm SMuon}(t),\beta_{\rm SMuon}(t),\gamma_{\rm SMuon}(t))
     &= (\alpha_{\rm Muon}(\tau(t)),\beta_{\rm Muon}(\tau(t)),\gamma_{\rm Muon}(\tau(t))),\\
     \frac{dt}{d\tau}&=\frac{\sqrt{3}}
     {\chi(\alpha_{\rm Muon}(\tau),\beta_{\rm Muon}(\tau),\gamma_{\rm Muon}(\tau))}.
    \end{aligned}
    \label{equation:smuon_diagonal_time_change}
\end{equation}
Hence, SMuon keeps simplified Muon's straight, equal-coordinate path while restoring a
residual-dependent speed.  The Muon path reaches $(1,0,0)$ in finite $\tau$,
but $\chi\to0$ there and the physical SMuon time diverges (logarithmically
under the generic signal-dominated approach), so the continuous flow converges
asymptotically rather than by a finite jump.  This is the diagonal manifestation
of the direction--speed decomposition: NGD changes only the clock of GD,
whereas SMuon changes only the clock of Muon and supplies late-stage damping.
\end{remark}

\subsection{River--Valley Structure of the Reduced Mixed Spectrum Sensing Model}
\label{subsec:river_valley_of_reduced_mixed_spectrum_ms}

\subsubsection{Additional Definitions, Assumptions, and Basic Geometry}
\label{subsubsec:additional_definition_assumption_basic_geometry}
Under the MSMS construction, the scalar dynamics induced on the invariant manifolds $\mathcal{M}$ and $\mathcal{X}$ admit an exact river--valley decomposition~\citep{wen2025understanding}.
We later define the associated river and hill quantities, including coordinates, forces, and energies based on a reduced river $\mathcal{R}$.

\begin{definition}[River-hill coordinates]
\label{definition:river_hill_coordinates}
	We reparameterize the reduced coordinates $(\alpha,\beta,\gamma)$ by
	\begin{equation}
		\theta := \alpha-1 \in \mathbb R,
		\qquad
		\omega := [\beta,\gamma]^\top \in \mathbb R^2.
		\label{definition:river_hill_variables}
	\end{equation}
	Here, $\theta$ is the river variable, measuring progress toward the target signal component, while $\omega$ collects the two hill variables, corresponding to the transverse spike and bulk components.
\end{definition}

\begin{lemma}[Reduced loss on invariant manifold]
\label{lemma:reduced_loss}
	Let
	\begin{equation}
		\mathcal F(M)
		=
		\frac{1}{n}
		\left\|M-P_\star\right\|_F^2
		+
		\frac{1}{2}
		\begin{bmatrix}
			\langle P_\star, M-P_\star\rangle \\
			\langle P_\diamond, M\rangle \\
			\langle P_\circ, M\rangle
		\end{bmatrix}^\top
		\Sigma 
		\begin{bmatrix}
			\langle P_\star, M-P_\star\rangle \\
			\langle P_\diamond, M\rangle \\
			\langle P_\circ, M\rangle
		\end{bmatrix},
		\nonumber
	\end{equation}
	and restrict $M$ to the invariant manifold, i.e., 
	$
		M = \alpha P_\star+\beta P_\diamond+\gamma P_\circ.
	$
	Then the reduced loss is
	\begin{align}
		f(\alpha,\beta,\gamma)
		&=
		\frac{r}{n}
		\left[
		(\alpha-1)^2+\beta^2+\gamma^2
		\right]
		+
		\frac{r^2}{2}
		\begin{bmatrix}
			\alpha-1 \\ \beta \\ \gamma
		\end{bmatrix}^{\top}
		\Sigma
		\begin{bmatrix}
			\alpha-1 \\ \beta \\ \gamma
		\end{bmatrix}
		\nonumber \\
		&=
		\frac12
		\begin{bmatrix}
			\alpha-1 \\ \beta \\ \gamma
		\end{bmatrix}^{\top}
		\hat K
		\begin{bmatrix}
			\alpha-1 \\ \beta \\ \gamma
		\end{bmatrix},
		\qquad \text{where}\ \
		\hat K
		:=
		\frac{2r}{n}I_3+r^2\Sigma.
		\nonumber
	\end{align}
\end{lemma}


\begin{definition}[River-hill splitting of the reduced curvature]
\label{definition:river_hill_splitting_of_the_reduced_curvature}
	The above coordinate decomposition induces a natural block splitting of the reduced curvature matrix
	\begin{equation}
		\hat K
		:=
		\frac{2r}{n}I_3+r^2\Sigma
		=
		\begin{bmatrix}
			\hat a & \hat b^\top \\
			\hat c & \hat D
		\end{bmatrix}
		\in \mathbb S^3,
		\label{definition:a_b_c_D}
	\end{equation}
	where $\hat a\in\mathbb R$, $\hat b,\hat c\in\mathbb R^2$, and $\hat D\in\mathbb S^2$. Since $\hat K$ is symmetric, $\hat b=\hat c$. Moreover, $\hat D$ is a principal submatrix of the positive semidefinite matrix $\hat K$, and hence $\hat D\succeq 0$. Throughout this section, we assume that the hill block is nondegenerate:
	\begin{equation}
		\hat D \succ 0.
		\label{assumption:hill_block_positive_definite}
	\end{equation}
	We denote its extremal eigenvalues by
	\begin{equation}
		\mu := \lambda_{\min}(\hat D)>0,
		\qquad
		\nu := \lambda_{\max}(\hat D)>0.
		\label{definition:mu_nu}
	\end{equation}
	In river-hill coordinates, the reduced loss $f(\alpha,\beta,\gamma)$ becomes
	\begin{equation}
		f(\theta,\omega)
		=
		\frac12
		\begin{bmatrix}
			\theta \\ \omega
		\end{bmatrix}^{\top}
		\begin{bmatrix}
			\hat a & \hat c^\top \\
			\hat c & \hat D
		\end{bmatrix}
		\begin{bmatrix}
			\theta \\ \omega
		\end{bmatrix}.
		\nonumber
	\end{equation}
\end{definition}

\begin{definition}[River-hill forces]
	\label{definition:river_hill_forces}
	For any reduced coordinates $(\theta,\omega)\in\mathbb R\times\mathbb R^2$, define
	\begin{equation}
		\phi := \hat a\theta+\hat b^\top\omega \in\mathbb R,
		\qquad
		\psi := \hat c\theta+\hat D\omega \in\mathbb R^2,
		\label{definition:river_hill_force}
	\end{equation}
	as the river force and hill force, respectively.
	
	Equivalently, $(\phi,\psi)$ is the block decomposition of the reduced gradient:
	\begin{equation}
		\nabla f(\theta,\omega)
		=
		\begin{bmatrix}
			\phi \\
			\psi
		\end{bmatrix}.
		\nonumber
	\end{equation}
	Since $\hat D\succ0$, the hill variable can be expressed in terms of the hill force as
	\begin{equation}
		\omega
		=
		-\theta \hat D^{-1}\hat c
		+
		\hat D^{-1} \psi.
		\label{equation:omega_to_psi}
	\end{equation}
\end{definition}

\begin{definition}[River-hill energies]
\label{definition:river_hill_energy}
	The reduced loss can be written as
	\begin{equation}
		f(\theta,\omega)
		=
		\frac12
		\left(
		\hat a\theta^2
		+
		2\theta\hat c^\top\omega
		+
		\omega^\top\hat D\omega
		\right).
		\nonumber
	\end{equation}
	Using the hill force $\psi=\hat c\theta+\hat D\omega$, we have
	\begin{equation}
		\psi^\top\hat D^{-1}\psi
		=
		\theta^2\hat c^\top\hat D^{-1}\hat c
		+
		2\theta\hat c^\top\omega
		+
		\omega^\top\hat D\omega.
		\nonumber
	\end{equation}
	Therefore,
	\begin{equation}
		f(\theta,\omega)
		=
		\frac12
		\left[
		\hat a_{\rm eff}\theta^2
		+
		\psi^\top\hat D^{-1}\psi
		\right],
		\nonumber
	\end{equation}
	where the effective river curvature is defined by
	\begin{equation}
		\hat a_{\rm eff}
		:=
		\hat a-\hat c^\top\hat D^{-1}\hat c.
		\label{definition:a_eff}
	\end{equation}
	We assume throughout that
	\begin{equation}
		\hat a_{\rm eff}>0.
		\label{assumption:a_eff_positive}
	\end{equation}
	This exact decomposition motivates the river and hill energies
	\begin{equation}
		E_k
		:=
		\hat a_{\rm eff}\theta_k^2,
		\qquad
		H_k
		:=
		\psi_k^\top\hat D^{-1}\psi_k,
		\label{definition:river_hill_energy_variables}
	\end{equation}
    hence $H_k=O(\eta^2)$ means that the iterate lies inside an $\psi_k=O(\eta)$ vicinity (so we name it as $O(\eta)$ hill-force tube) around the river. 
\end{definition}

\begin{definition}[Reduced river]
	\label{definition:reduced_river}
	The reduced river is the zero-hill-energy (zero-hill-force) manifold
	\begin{equation}
		\mathcal R
		:=
		\{(\theta,\omega):H=0\}
		=
		\{(\theta,\omega):\psi=0\}
		=
		\{(\theta,\omega):\omega=-\hat D^{-1}\hat c\,\theta\}.
		\label{definition:river_manifold}
	\end{equation}
	Furthermore, since $\hat b=\hat c$, the river force satisfies
	\begin{equation}
		\phi
		=
		\hat a\theta+\hat c^\top\omega
		=
		\hat a_{\rm eff}\theta
		+
		\hat c^\top\hat D^{-1}\psi.
		\label{equation:river_effective_forcing}
	\end{equation}
	Consequently, along the river manifold $\mathcal R$, where $\psi=0$, we have $\phi=\hat a_{\rm eff}\theta$.
	Thus, once the hill force has vanished, the reduced dynamics are governed solely by the effective river curvature $\hat a_{\rm eff}$.
\end{definition}

\begin{remark}[A brief notation summary] \label{remark:a_brief_notation_summary}
	We summarize the notation used in this section and their relationships in Table~\ref{table:variable_force_energy}.
	We measure river progress by the coordinate $\theta$, because after eliminating the transverse variables, $\theta$ parametrizes the zero-hill-force valley floor and the remaining loss is $\hat a_{\rm eff}\theta^2$ (see \eqref{definition:river_hill_energy_variables}). In contrast, the raw hill variable $\omega$ is not a genuine transverse error when $\hat c\neq 0$, since the coupled river itself has $\omega=-\hat D^{-1}\hat c\theta$; hence the intrinsic hill energy must be measured by the hill force $\psi=\hat c\theta+\hat D\omega$, or equivalently by the squared $\hat D$-metric distance to the zero-hill-force manifold.
	
	\begin{table*}[h]
		\caption{Variables, forces, energies, and reduced discrete iterations in the general PSD case.}
		\label{table:variable_force_energy}
		\centering
		\resizebox{0.60\linewidth}{!}{
			\begin{tabular}{cccc}
				\hline
				River Item & 
				River Notation & 
				Hill Item & 
				Hill Notation \\
				\hline
				River variable & 
				$\theta_k = \alpha_k - 1$ & 
				Hill variable & 
				$\omega_k = [\beta_k,\gamma_k]^\top$ \\ 
				River force & 
				$\phi_k = \hat a \theta_k + \hat b^\top \omega_k$ & 
				Hill force & 
				$\psi_k = \hat c \theta_k + \hat D \omega_k$ \\ 
				River energy & 
				$E_k = \hat a_{\rm eff}\theta_k^2$ & 
				Hill energy & 
				$H_k = \psi_k^\top \hat D^{-1}\psi_k$ \\ 
				\hline
			\end{tabular}
		}
	\end{table*}
	For the subsequent local estimates, we restrict the dynamics to the compact box
	\begin{equation}
		K_{\underline\alpha,\overline\alpha,\underline\omega,\overline\omega}
		:=
		\left\{
		(\theta,\omega)\in\mathbb R\times\mathbb R^2:
		\underline\alpha \le 1+\theta\le \overline\alpha,\;
		\underline\omega \mathbf 1_2
		\preceq
		\omega
		\preceq
		\overline\omega \mathbf 1_2
		\right\},
		\label{definition:compact_box}
	\end{equation}
	where $\mathbf 1_2=(1,1)^\top$, and $\preceq$ denotes componentwise comparison.
\end{remark}

\subsubsection{Gradient Descent Reduced Dynamics}
The reduced GD iteration follows readily from applying an Euler discretization to Lemma~\ref{lemma:gd_dynamics}:
\begin{equation}
	\begin{cases}
		\theta_{k+1}
		=
		\theta_k
		-
		4\eta(1+\theta_k)\phi_k,
		\\
		\omega_{k+1}
		=
		\omega_k
		-
		4\eta
		\mathtt{Diag}(\omega_k)\psi_k .
	\end{cases}
	\label{equation:gd_river_valley_iteration}
\end{equation}

\begin{proposition}[Positivity preservation for reduced GD]
	Suppose that at iteration $k$ we have $1+\theta_k>0$ and $\omega_k\succ 0$.
	Assume further that the step size $\eta>0$ is chosen sufficiently small so that, for all $k$,
	\begin{equation}
		4\eta|\phi_k|<1,
		\qquad
		4\eta\psi_k
		\prec
		\begin{bmatrix}1\\1\end{bmatrix}.
		\nonumber
	\end{equation}
	Then, at iteration $k+1$, it holds that $1+\theta_{k+1}>0$ and $\omega_{k+1}\succ 0$.
\end{proposition}

\begin{proof}
	From~\ref{equation:gd_river_valley_iteration}, the river-coordinate update can be written as
	\begin{equation}
		1+\theta_{k+1}
		=
		(1+\theta_k)(1-4\eta\phi_k).
		\nonumber
	\end{equation}
	Since $1+\theta_k>0$ and $4\eta|\phi_k|<1$, we have $1-4\eta\phi_k>0$, and hence the right-hand side is positive.
	For each hill coordinate $i\in\{1,2\}$,
	\begin{equation}
		(\omega_{k+1})_i
		=
		(\omega_k)_i
		\bigl(1-4\eta(\psi_k)_i\bigr),
		\qquad i=1,2.
		\nonumber
	\end{equation}
	Because $(\omega_k)_i>0$ and $4\eta(\psi_k)_i<1$, it follows that $1-4\eta(\psi_k)_i>0$, and therefore $(\omega_{k+1})_i>0$ for $i=1,2$.
\end{proof}

\begin{proposition}[Hill-energy estimate for reduced GD]
	\label{proposition:hill-energy-gd}
	Assume that the GD trajectory remains in the compact box $K_{\underline\alpha,\overline\alpha,\underline\omega,\overline\omega}$.
	If $0<\eta\le\frac{\underline\omega}{8\nu\overline\omega^2}$, then there exists a constant $C_{\rm GD}>0$, depending only on $\underline\alpha,\overline\alpha,\underline\omega,\overline\omega,\hat D,\hat c$ and $\eta$, such that
	\begin{equation}
		H_{k+1}
		\le
		(1-2\underline\omega\mu\eta)H_k
		+
		C_{\rm GD}\eta\,\phi_k^2.
		\label{equation:GD-hill-estimate}
	\end{equation}
	Moreover, one admissible choice is
	\begin{equation}
		C_{\rm GD}
		=
		\frac{
			8 \overline\alpha^2
			\|\hat D^{-1}\hat c\|_2^2
		}{
			\underline\omega
		}
		+
		32 \overline\alpha^2 \eta\,
		\hat c^\top\hat D^{-1}\hat c.
		\nonumber
	\end{equation}
\end{proposition}

\begin{proof}
	By definition, $\psi_{k+1} = \hat c\theta_{k+1}+\hat D\omega_{k+1}$.
	Using~\ref{equation:gd_river_valley_iteration}, we obtain
	\begin{align}
		\psi_{k+1}
		&=
		\hat c
		\bigl[
		\theta_k-4\eta(1+\theta_k)\phi_k
		\bigr]
		+
		\hat D
		\bigl[
		\omega_k
		-
		4\eta\mathtt{Diag}(\omega_k)\psi_k
		\bigr]
		\nonumber\\
		&=
		\psi_k
		-
		4\eta
		\left[
		\hat D\mathtt{Diag}(\omega_k)\psi_k
		+
		(1+\theta_k)\hat c\,\phi_k
		\right].
		\nonumber
	\end{align}
	Define $r_k := \hat{D}\mathtt{Diag}(\omega_k)\psi_k + (1+\theta_k)\phi_k\hat{c}$.
	Then $\psi_{k+1}=\psi_k-4\eta r_k$, and therefore
	\begin{align}
		H_{k+1}-H_k
		&=
		-8\eta\,\psi_k^\top\hat D^{-1}r_k
		+
		16\eta^2 r_k^\top\hat D^{-1}r_k.
		\nonumber
	\end{align}
	For the linear term, we have
	\begin{align}
		\psi_k^\top\hat D^{-1}r_k
		=
		\psi_k^\top \mathtt{Diag}(\omega_k)\psi_k
		+
		(1+\theta_k)\phi_k\,\psi_k^\top \hat D^{-1}\hat c
		\ge
		\underline\omega \|\psi_k\|_2^2
		-
		|1+\theta_k|\,|\phi_k|\,|\psi_k^\top \hat D^{-1}\hat c|,
		\nonumber
	\end{align}
	and hence
	\begin{align}
		-8\eta\,\psi_k^\top\hat D^{-1}r_k
		\le
		-8\underline\omega\eta\|\psi_k\|_2^2
		+
		8\eta |1+\theta_k|\,|\phi_k|\,|\psi_k^\top \hat D^{-1}\hat c| .
		\nonumber
	\end{align}
	For the quadratic term, note that
	\begin{align}
		r_k^\top \hat D^{-1} r_k
		=&
		\bigl[
		\hat D\mathtt{Diag}(\omega_k)\psi_k
		+
		(1+\theta_k)\hat c\,\phi_k
		\bigr]^\top
		\hat D^{-1}
		\bigl[
		\hat D\mathtt{Diag}(\omega_k)\psi_k
		+
		(1+\theta_k)\hat c\,\phi_k
		\bigr]
		\nonumber \\
		=&
		\bigl\|
		\hat D^{1/2}\mathtt{Diag}(\omega_k)\psi_k
		+
		(1+\theta_k)\phi_k\,\hat D^{-1/2}\hat c
		\bigr\|_2^2
        \nonumber \\
		\le&
		2\|\hat D^{1/2}\mathtt{Diag}(\omega_k)\psi_k\|_2^2
		+
		2(1+\theta_k)^2\phi_k^2\|\hat D^{-1/2}\hat c\|_2^2
		\nonumber \\
		=&
		2\psi_k^\top \mathtt{Diag}(\omega_k)\hat D\mathtt{Diag}(\omega_k)\psi_k
		+
		2(1+\theta_k)^2\phi_k^2\,\hat c^\top \hat D^{-1}\hat c
        \nonumber \\
		\le&
		2\nu \overline\omega^2 \|\psi_k\|_2^2
		+
		2\overline\alpha^2\phi_k^2\,\hat c^\top \hat D^{-1}\hat c .
		\nonumber
	\end{align}
	Consequently,
	\begin{align}
		16\eta^2 r_k^\top \hat D^{-1} r_k
		\le
		32\nu\overline\omega^2\eta^2\|\psi_k\|_2^2
		+
		32\overline\alpha^2\eta^2
		\hat c^\top \hat D^{-1}\hat c\,\phi_k^2 .
		\nonumber
	\end{align}
	Applying Young's inequality to the mixed term yields
	\begin{align}
		8\eta |1+\theta_k|\,|\phi_k|\,|\psi_k^\top \hat D^{-1}\hat c|
		&\le
		2\underline\omega\eta \|\psi_k\|_2^2
		+
		\frac{
			8\overline\alpha^2\|\hat D^{-1}\hat c\|_2^2
		}{
			\underline\omega
		}
		\eta \phi_k^2 .
		\nonumber
	\end{align}
	Combining these estimates we get
	\begin{align}
		H_{k+1}-H_k
		\le
		\Bigl(
		-8\underline\omega\eta
		+
		2\underline\omega\eta
		+
		32\nu\overline\omega^2\eta^2
		\Bigr)\|\psi_k\|_2^2
		+
		\Biggl(
		\frac{
			8\overline\alpha^2\|\hat D^{-1}\hat c\|_2^2
		}{
			\underline\omega
		}\eta
		+
		32\overline\alpha^2\eta^2
		\hat c^\top \hat D^{-1}\hat c
		\Biggr)\phi_k^2 .
		\nonumber
	\end{align}
	Since $0<\eta\le \underline\omega/(8\nu\overline\omega^2)$, we have $32\nu\overline\omega^2\eta^2 \le 4\underline\omega\eta$, and thus
	\begin{equation}
		H_{k+1}-H_k
		\le
		-2\underline\omega\eta\|\psi_k\|_2^2
		+
		C_{\rm GD}\eta\,\phi_k^2,
		\nonumber
	\end{equation}
	where one may take
	\begin{equation}
		C_{\rm GD}
		=
		\frac{
			8\overline\alpha^2\|\hat D^{-1}\hat c\|_2^2
		}{
			\underline\omega
		}
		+
		32\overline\alpha^2\eta\,\hat c^\top \hat D^{-1}\hat c .
		\nonumber
	\end{equation}
	Finally, using $\|\psi_k\|_2^2\ge \mu H_k$, we conclude that
	\begin{equation}
		H_{k+1}
		\le
		(1-2\underline\omega\mu\eta)H_k
		+
		C_{\rm GD}\eta\,\phi_k^2 .
		\nonumber
	\end{equation}
	The proof is complete.
\end{proof}

\begin{corollary}[Tail bound for the GD hill energy]
	Under the assumptions of the preceding proposition, define $\rho_{\rm GD}:=1-2\underline\omega\mu\eta\in(0,1)$.
	For any $N\ge0$, let $\Phi_{N}:=\sup_{j\ge N}|\phi_j|$.
	Then, for every $k\ge N$,
	\begin{equation}
		H_k
		\le
		\rho_{\rm GD}^{k-N} H_{N}
		+
		\frac{C_{\rm GD}}{2\underline\omega\mu}\Phi_{N}^2.
		\label{equation:GD-tail-bound}
	\end{equation}
	In particular, if $\phi_k\to0$, then $H_k\to0$ and hence $\psi_k\to0$.
\end{corollary}

\begin{proof}
	Iterating~\eqref{equation:GD-hill-estimate} from $N$ to $k-1$ yields
	\begin{align}
		H_k
		\le
		\rho_{\rm GD}^{k-N}H_{N}
		+
		C_{\rm GD}\eta
		\sum_{j=N}^{k-1}
		\rho_{\rm GD}^{k-1-j}\phi_j^2
		\le&
		\rho_{\rm GD}^{k-N}H_{N}
		+
		C_{\rm GD}\eta\Phi_{N}^2
		\sum_{\iota=0}^{k-1-N}\rho_{\rm GD}^{\iota}
        \nonumber \\
		\le&
		\rho_{\rm GD}^{k-{N}}H_{N}
		+
		\frac{C_{\rm GD}\eta}{1-\rho_{\rm GD}}\Phi_{N}^2.
		\nonumber
	\end{align}
	Since $1-\rho_{\rm GD}=2\underline\omega\mu\eta$, this implies~\eqref{equation:GD-tail-bound}. Finally, if $\phi_k\to0$, then $\Phi_{N}\to0$ as $N\to\infty$, and therefore $H_k\to0$.
\end{proof}

\begin{remark}[A Lyapunov perspective on GD]
	\label{remark:GD-Lyapunov}
	The GD dynamics admit a natural interpretation from the viewpoint of Lyapunov dissipation. In the region $\alpha>0$, $\beta\ge 0$, and $\gamma\ge 0$, consider
	\begin{equation}
		V_{\mathrm{GD}}(\alpha,\beta,\gamma)
		:=
		r(\alpha-1-\ln \alpha) + r\beta + r\gamma.
		\label{equation:GD-Lyapunov}
	\end{equation}
	Defining $x:=[\alpha,\beta,\gamma]^\top$, $e_1:=[1,0,0]^\top$, $\hat{k}:=\hat{K}e_1$, a direct differentiation along the GD flow gives
	\begin{align}
		\dot V_{\mathrm{GD}}
		=
		r\left(1-\frac{1}{\alpha}\right)\dot\alpha
		+ r\dot\beta + r\dot\gamma
		&=
		4(x-e_1)^\top \mathtt{Diag}(r,r,r)\,(\hat{k}-\hat{K}x)
		\nonumber\\
		&=
		-4(x-e_1)^\top \mathtt{Diag}(r,r,r)\hat{K}(x-e_1)
		\le 0.
		\nonumber
	\end{align}
	Thus $V_{\mathrm{GD}}$ is a smooth Lyapunov function for the GD continuous-time dynamics, and its decay is governed by a quadratic dissipation mechanism around the equilibrium $x=e_1$, under the corresponding positivity assumptions on the above quadratic form.
	
	This provides a useful conceptual contrast with the Muon dynamics. For GD, the evolution is compatible with a smooth energy landscape and a correspondingly smooth dissipation law. By contrast, the sign-based Muon update is intrinsically nonsmooth, and the quantity
	\eqref{equation:Muon-Lyapunov} does not play the role of a comparably smooth Lyapunov function. This helps explain why GD typically admits cleaner convergence behavior, whereas Muon more naturally exhibits switching effects and $O(\eta)$-scale residual oscillations under a constant step size.
\end{remark}

\begin{proposition}[GD river progress after hill reduction]
	\label{proposition:GD-river-progress-after-hill-reduction}
	Assume that there exist $N_{\rm GD}\ge0$ and $R_{\rm GD}>0$ such that $H_k\le R_{\rm GD}^2\eta^2$ for all $k\ge N_{\rm GD}$. 
	In particular, the hill forcing $\psi_k$ is of order $O(\eta)$ for all $k\ge N_{\rm GD}$.
	Suppose further that $\eta\le \frac{1}{6\overline\alpha\hat a_{\rm eff}}$, and that $\eta$ is sufficiently small so that for all $k\ge N_{\rm GD}$,
	\begin{equation}
		|\theta_k|
		\ge
		\max
		\left\{
		\frac{
			2 (\hat c^\top\hat D^{-1}\hat c)^{1/2} 
		}{
			\hat a_{\rm eff}
		},
		\sqrt{\frac{\Gamma_{\rm GD}}{\hat a_{\rm eff}}}
		\right\}
		R_{\rm GD} \eta,
		\qquad
		\Gamma_{\rm GD} \gg 1.
		\label{equation:GD-river-sign-condition}
	\end{equation}
	Then the river variable $\theta_k$ contracts geometrically:
	\begin{equation}
		|\theta_{k+1}|
		\le
		(1-2\underline\alpha\hat a_{\rm eff}\eta)|\theta_k|.
		\label{equation:GD-river-contraction}
	\end{equation}
\end{proposition}

\begin{proof}
	By~\eqref{equation:river_effective_forcing}, we have $\phi_k-\hat a_{\rm eff}\theta_k = \hat c^\top\hat D^{-1}\psi_k$.
	Applying the Cauchy-Schwarz inequality in the $\hat D^{-1}$-inner product yields
	\begin{equation}
		|\hat c^\top\hat D^{-1}\psi_k|
		\le
		\bigl(\hat c^\top\hat D^{-1}\hat c\bigr)^{1/2}
		\bigl(\psi_k^\top\hat D^{-1}\psi_k\bigr)^{1/2}
		=
		\bigl(\hat c^\top\hat D^{-1}\hat c\bigr)^{1/2}
		\sqrt{H_k}.
		\nonumber
	\end{equation}
	Using $H_k\le R_{\rm GD}^2\eta^2$, we obtain $|\phi_k-\hat a_{\rm eff}\theta_k| \le (\hat c^\top\hat D^{-1}\hat c)^{1/2}R_{\rm GD}\eta$.
	If~\eqref{equation:GD-river-sign-condition} holds, then
	\begin{equation}
		|\phi_k-\hat a_{\rm eff}\theta_k|
		\le
		\frac{\hat a_{\rm eff}}2|\theta_k|
		\qquad\text{and}\qquad
		E_k = \hat a_{\rm eff}\theta_k^2 
		\ge \Gamma_{\rm GD} R_{\rm GD}^2\eta^2
		\gg H_k.
		\nonumber
	\end{equation}
	Consequently, the iterate is in a river-dominant phase, $\phi_k$ has the same sign as $\theta_k$, and
	\begin{equation}
		\frac{\hat a_{\rm eff}}2|\theta_k|
		\le
		|\phi_k|
		\le
		\frac{3\hat a_{\rm eff}}2|\theta_k|.
		\label{equation:GD-phi-two-sided}
	\end{equation}
	Recall that the GD river update is $\theta_{k+1} = \theta_k-4\eta(1+\theta_k)\phi_k$.
	Using $1+\theta_k\ge\underline\alpha$ and the lower bound in~\eqref{equation:GD-phi-two-sided}, the update decreases $|\theta_k|$ by at least $2\underline\alpha\hat a_{\rm eff}\eta|\theta_k|$, provided that no overshoot occurs.
	To preclude overshoot, we use the upper bound in~\eqref{equation:GD-phi-two-sided} and 
	$\eta\le \frac{1}{6\overline\alpha\hat a_{\rm eff}}$ to obtain
	\begin{equation}
		4\eta(1+\theta_k)|\phi_k|
		\le
		6\eta \overline\alpha \hat{a}_{\rm eff}|\theta_k|
		\le
		|\theta_k|.
		\nonumber
	\end{equation}
	Hence overshoot does not occur, and~\eqref{equation:GD-river-contraction} follows.
\end{proof}

\begin{remark}
	\label{remark:GD-discuss}
	Proposition~\ref{proposition:GD-river-progress-after-hill-reduction} shows that once the hill forcing $\psi_k$ is reduced to order $O(\eta)$ (e.g., via~\eqref{equation:GD-hill-estimate}) and river energy dominates hill energy, the GD river variable $\theta_k$ contracts geometrically according to~\eqref{equation:GD-river-contraction} until it reaches an $O(\eta)$-neighborhood of the river equilibrium $\theta=0$.
	Consequently, the number of GD iterations required to reduce $|\theta_k|$ from a constant scale to an $O(\eta)$ scale is
	\begin{equation}
		O\!\left(\frac1\eta\log\frac1\eta\right).
		\nonumber
	\end{equation}
	Moreover, in practice GD may drive the hill forcing $\psi_k$ below the $O(\eta)$ level (i.e., achieve higher accuracy), whereas the corresponding simplified Muon analysis typically guarantees only an $O(\eta)$-level transient bound. We therefore focus on the $O(\eta)$ regime to enable a like-for-like comparison with simplified Muon.
\end{remark}

\subsubsection{Simplified Muon Reduced Dynamics}
The reduced Muon iteration follows readily from applying an Euler discretization to Lemma~\ref{lemma:muon_dynamics}:
\begin{equation}
	\begin{cases}
		\theta_{k+1}
		=
		\theta_k
		-
		2\eta\sqrt{1+\theta_k}\,
		s_k,
		\qquad
		s_k \in\mathtt{sign}(\phi_k),
		\\
		\omega_{k+1}
		=
		\omega_k
		-
		2\eta
		\mathtt{Diag}(\omega_k)^{1/2}
		t_k,
		\qquad
		t_k \in\mathtt{sign}(\psi_k).
	\end{cases}
	\label{equation:muon_river_valley_iteration}
\end{equation}

\begin{proposition}[Positivity preservation for reduced Muon]
	Assume that $\eta>0$ is sufficiently small such that, for all $k$,
	\begin{equation}
		1+\theta_k>4\eta^2,
		\qquad
		\omega_k
		\succ
		4\eta^2
		\begin{bmatrix}1\\1\end{bmatrix}.
		\nonumber
	\end{equation}
	Then, for any choices $s_k\in\mathtt{Sign}(\phi_k)$ and $t_k\in\mathtt{Sign}(\psi_k)$, it holds that
	\begin{equation}
		1+\theta_{k+1}>0,
		\qquad
		\omega_{k+1}\succ0.
		\nonumber
	\end{equation}
\end{proposition}

\begin{proof}
	Since $s_k\in[-1,1]$, we have
	\begin{align}
		1+\theta_{k+1}
		=
		1+\theta_k
		-
		2\eta\sqrt{1+\theta_k}\,s_k
		=&
		\sqrt{1+\theta_k}
		(
		\sqrt{1+\theta_k}-2\eta s_k
		)
        \nonumber \\
		\ge&
		\sqrt{1+\theta_k}
		(
		\sqrt{1+\theta_k}-2\eta
		)
		>
		0.
		\nonumber
	\end{align}
	Similarly, for each $i\in\{1,2\}$,
	\begin{align}
		(\omega_{k+1})_i
		=
		(\omega_k)_i
		-
		2\eta\sqrt{(\omega_k)_i}\,(t_k)_i
		=&
		\sqrt{(\omega_k)_i}
		\left[
		\sqrt{(\omega_k)_i}
		-
		2\eta(t_k)_i
		\right]
        \nonumber \\
		\ge&
		\sqrt{(\omega_k)_i}
		\left[
		\sqrt{(\omega_k)_i}-2\eta
		\right]
		>
		0.
		\nonumber
	\end{align}
	This completes the proof.
\end{proof}

\begin{proposition}[Hill-energy estimate for reduced Muon]
	\label{proposition:hill-energy-muon}
	Assume that the Muon orbit remains in the compact box $K_{\underline\alpha,\overline\alpha,\underline\omega,\overline\omega}$.
	Define the constants
	$
		B_0
		:=
		\sqrt{\underline\omega\mu}
		-
		\sqrt{\overline\alpha}\,
		(\hat c^\top\hat D^{-1}\hat c)^{1/2}
	$
	and
	$
		B_1
		:=
		8\overline\alpha\hat c^\top\hat D^{-1}\hat c
		+
		16\nu\overline\omega
	$.
	Then
	\begin{equation}
		H_{k+1}
		\le
		H_k
		-
		4B_0\eta\sqrt{H_k}
		+
		B_1\eta^2.
		\label{equation:Muon-hill-estimate}
	\end{equation}
\end{proposition}

\begin{proof}
	By definition, $\psi_{k+1} = \hat{c} \theta_{k+1} + \hat{D} \omega_{k+1}$. Using~\ref{equation:muon_river_valley_iteration}, we obtain
	\begin{align}
		\psi_{k+1}
		&=
		\hat{c} \left[
		\theta_k
		-
		2\eta\sqrt{1+\theta_k}\,
		s_k
		\right]
		+ 
		\hat{D} \left[
		\omega_k
		-
		2\eta
		\mathtt{Diag}(\omega_k)^{1/2}
		t_k
		\right]
		\nonumber \\
		&=
		\psi_k
		-
		2\eta
		\left[
		\hat c\sqrt{1+\theta_k}\,s_k
		+
		\hat D\mathtt{Diag}(\omega_k)^{1/2}t_k
		\right].
		\nonumber
	\end{align}
	Let $r_k := \hat c\sqrt{1+\theta_k}\,s_k + \hat D\mathtt{Diag}(\omega_k)^{1/2}t_k$.
	Then $\psi_{k+1}=\psi_k-2\eta r_k$, and hence
	\begin{equation}
		H_{k+1}-H_k
		=
		-4\eta\,\psi_k^\top\hat D^{-1}r_k
		+
		4\eta^2 r_k^\top\hat D^{-1}r_k.
		\nonumber
	\end{equation}
	For the linear term, we can write
	\begin{align}
		&\psi_k^\top\hat D^{-1}r_k
		=
		\sqrt{1+\theta_k}\,s_k
		\psi_k^\top\hat D^{-1}\hat c
		+
		\psi_k^\top
		\mathtt{Diag}(\omega_k)^{1/2}t_k
		\nonumber\\
		=&
		\sqrt{1+\theta_k}\,s_k
		\psi_k^\top\hat D^{-1}\hat c
		+
		\sum_{i=1}^2
		\sqrt{(\omega_k)_i}\,|(\psi_k)_i|
		\ge
		-
		\sqrt{\overline\alpha}
		\bigl(\hat c^\top\hat D^{-1}\hat c\bigr)^{1/2}
		\sqrt{H_k}
		+
		\sqrt{\underline\omega}\|\psi_k\|_2.
		\nonumber
	\end{align}
	Since $\|\psi_k\|_2^2\ge\mu H_k$, it follows that $\psi_k^\top\hat D^{-1}r_k \ge B_0\sqrt{H_k}$.
	
	For the quadratic term, using $(u+v)^\top \hat D^{-1}(u+v) \le 2u^\top \hat D^{-1}u+2v^\top \hat D^{-1}v$,
	we obtain
	\begin{align}
		&r_k^\top\hat D^{-1}r_k
		=
		\bigl[
		\hat c\sqrt{1+\theta_k}\,s_k
		+
		\hat D\mathtt{Diag}(\omega_k)^{1/2}t_k
		\bigr]^\top
		\hat D^{-1}
		\bigl[
		\hat c\sqrt{1+\theta_k}\,s_k
		+
		\hat D\mathtt{Diag}(\omega_k)^{1/2}t_k
		\bigr]
		\nonumber\\
		\le&
		2
		\bigl(
		\hat c\sqrt{1+\theta_k}\,s_k
		\bigr)^\top
		\hat D^{-1}
		\bigl(
		\hat c\sqrt{1+\theta_k}\,s_k
		\bigr)
		+
		2
		\bigl(
		\hat D\mathtt{Diag}(\omega_k)^{1/2}t_k
		\bigr)^\top
		\hat D^{-1}
		\bigl(
		\hat D\mathtt{Diag}(\omega_k)^{1/2}t_k
		\bigr)
		\nonumber\\
		=&
		2(1+\theta_k)s_k^2\,\hat c^\top\hat D^{-1}\hat c
		+
		2\,t_k^\top
		\mathtt{Diag}(\omega_k)^{1/2}
		\hat D
		\mathtt{Diag}(\omega_k)^{1/2}
		t_k
        \nonumber \\
		\le&
		2\overline\alpha \hat c^\top\hat D^{-1}\hat c
		+
		2\nu
		\|
		\mathtt{Diag}(\omega_k)^{1/2}t_k
		\|_2^2
		\nonumber\\
		=&
		2\overline\alpha \hat c^\top\hat D^{-1}\hat c
		+
		2\nu
		\sum_{i=1}^2
		(\omega_k)_i (t_k)_i^2
		\le
		2\overline\alpha \hat c^\top\hat D^{-1}\hat c
		+
		4\nu\overline\omega.
		\nonumber
	\end{align}
	Multiplying by $4\eta^2$ yields
	\begin{equation}
		4\eta^2 r_k^\top\hat D^{-1}r_k
		\le
		8\overline\alpha \hat c^\top\hat D^{-1}\hat c\,\eta^2
		+
		16\nu\overline\omega\,\eta^2
		=
		B_1\eta^2.
		\nonumber
	\end{equation}
	Combining the bounds on the linear and quadratic terms gives~\eqref{equation:Muon-hill-estimate}.
\end{proof}

\begin{corollary}[Eventual $O(\eta)$ hill tube for reduced Muon]
	\label{corollary:hill-tube-muon}
	Under the assumptions of the preceding proposition, define
	\begin{equation}
		R_{\rm Muon}
		:=
		\max\left\{
		\frac{B_1}{2B_0},
		\sqrt{B_1}
		\right\}.
		\nonumber
	\end{equation}
	If $B_0 > 0$ and $\sqrt{H_k}\ge R_{\rm Muon} \eta$, then
	\begin{equation}
		H_{k+1}
		\le
		H_k-B_1\eta^2.
		\nonumber
	\end{equation}
	Moreover, if $H_k\le R_{\rm Muon}^2\eta^2$, then $H_{k+1}\le R_{\rm Muon}^2\eta^2$.
	Consequently, the reduced Muon orbit enters and thereafter remains in the $O(\eta)$ hill tube, i.e., there exists $N_{\rm Muon}$ such that, for all $k \ge N_{\rm Muon}$,
	\begin{equation}
		\|\psi_k\|_2
		\le
		\sqrt{\nu}R_{\rm Muon}\eta.
		\nonumber
	\end{equation}
\end{corollary}

\begin{proof}
	From~\eqref{equation:Muon-hill-estimate}, we have $H_{k+1} \le H_k-4B_0\eta\sqrt{H_k}+B_1\eta^2$.
	
	If $B_0>0$ and $\sqrt{H_k}\ge R_{\rm Muon}\eta$, then $R_{\rm Muon}\ge B_1/(2B_0)$ implies
	\begin{equation}
		4B_0\eta\sqrt{H_k}\ge 2B_1\eta^2,
		\nonumber
	\end{equation}
	and therefore $H_{k+1}\le H_k-B_1\eta^2$.
	
	To prove invariance of the tube, suppose $H_k\le R_{\rm Muon}^2\eta^2$ and consider
	\begin{equation}
		f(y)=y-4B_0\eta\sqrt y+B_1\eta^2.
		\nonumber
	\end{equation}
	This function is concave on $[0,\infty)$, hence it attains its maximum over $[0,R_{\rm Muon}^2\eta^2]$ at an endpoint.
	Using $R_{\rm Muon}\ge\sqrt{B_1}$ and $R_{\rm Muon}\ge B_1/(2B_0)$, one can check that $f(R_{\rm Muon}^2\eta^2)\le R_{\rm Muon}^2\eta^2$ and $f(0)\le R_{\rm Muon}^2\eta^2$,
	which yields $H_{k+1}\le R_{\rm Muon}^2\eta^2$. 
	Thus, the tube is forward invariant.
	
	Finally, outside the tube the energy decreases by at least $B_1\eta^2$ per step, which implies a finite hitting-time bound with
	\begin{equation}
		N_{\rm Muon}
		=
		\left\lceil
		\frac{
			\max\{
			H_0-R_{\rm Muon}^2\eta^2, 0
			\}
		}
		{B_1\eta^2}
		\right\rceil.
		\nonumber
	\end{equation}
	The proof is complete.
\end{proof}

\begin{remark}[A Lyapunov perspective on Simplified Muon]
	\label{remark:Muon-Lyapunov}
	The conclusion of Corollary~\ref{corollary:hill-tube-muon} depends essentially on the sign of $B_0$. If $B_0<0$, then Proposition~\ref{proposition:hill-energy-muon} yields
	\begin{equation}
		H_{k+1}
		\le
		H_k
		+
		4|B_0|\eta \sqrt{H_k}
		+
		B_1\eta^2,
		\nonumber
	\end{equation}
	which no longer implies a dissipative drift for the reduced hill energy $H_k$. 
	This is consistent with the optimization intuition behind Muon: sign-based dynamics can exhibit nonsmooth switching and transient oscillations in the river-valley coordinates, even when the full dynamics remain energy-dissipative.
	Writing $x=[\alpha,\beta,\gamma]^\top$, and let $\hat k$ denote the first column of $\hat K$. The Muon flow admits the Filippov signum form
	\begin{equation}
		\dot{x}
		\in
		2\mathtt{Diag}(\sqrt{x})\,
		\mathtt{Sign}(\hat{k}-\hat{K}x)
		=
		-2\mathtt{Diag}(\sqrt{x})\,
		\mathtt{Sign}\!\bigl(\hat K(x-e_1)\bigr),
		\nonumber
	\end{equation}
	A natural Lyapunov candidate is the $\ell_1$-potential of the residual,
	\begin{equation}
		V_{\rm Muon}(x)
		:=
		\|\hat K(x-e_1)\|_1.
		\label{equation:Muon-Lyapunov}
	\end{equation}
	Since $V_{\rm Muon}$ is convex and locally Lipschitz, its Clarke generalized gradient coincides with its convex subdifferential, and
	\begin{equation}
		\partial V_{\rm Muon}(x)
		=
		\hat K^\top\,
		\mathtt{Sign}\!\bigl(\hat K(x-e_1)\bigr).
		\nonumber
	\end{equation}
	Hence, for almost every $t$, one may choose $\sigma(t)\in \mathtt{Sign}(\hat K(x(t)-e_1))$ such that
	\begin{equation}
		\dot x(t)
		=
		-2\mathtt{Diag}(\sqrt{x(t)})\,\sigma(t),
		\qquad
		g(t):=\hat K^\top \sigma(t)\in \partial V_{\rm Muon}(x(t)).
		\nonumber
	\end{equation}
	By the Clarke chain rule for locally Lipschitz Lyapunov functions along differential inclusions,
	\begin{equation}
		\dot V_{\rm Muon}\bigl(x(t);\dot x(t)\bigr)
		\le
		g(t)^\top \dot x(t)
		=
		-2\,\sigma(t)^\top
		\hat K\,\mathtt{Diag}(\sqrt{x(t)})\,\sigma(t)
		\le
		0
		\qquad\text{for a.e. }t,
		\nonumber
	\end{equation}
	Therefore, although the reduced hill energy $H_k$ need not be monotone when $B_0<0$, the full Muon dynamics still admit a natural nonincreasing energy functional. 
	In this sense, the condition $B_0>0$ should be viewed as a technical assumption that excludes hill-direction excitation at the level of the reduced recursion, thereby yielding a clean one-step contraction estimate.
\end{remark}

\begin{proposition}[Muon river progress after hill reduction]
	\label{proposition:Muon-river-progress-after-hill-reduction}
	Assume there exist $N_{\rm Muon}\ge0$ and $R_{\rm Muon}>0$ such that $H_k\le R_{\rm Muon}^2\eta^2$ for all $k\ge N_{\rm Muon}$. 
	In particular, the hill forcing $\psi_k$ is $O(\eta)$ for all $k\ge N_{\rm Muon}$.
	Suppose further that $\eta$ is sufficiently small so that, for all $k\ge N_{\rm Muon}$,
	\begin{equation}
		|\theta_k|
		\ge
		\max\left\{
		\frac{
			2
			\bigl(\hat c^\top\hat D^{-1}\hat c\bigr)^{1/2}
			R_{\rm Muon}
		}{
			\hat a_{\rm eff}
		}\eta, \,
		2\sqrt{\overline\alpha}\,\eta, \,
		\sqrt{\frac{\Gamma_{\rm Muon}}{\hat a_{\rm eff}}}
		R_{\rm Muon} \eta
		\right\},
		\qquad
		\Gamma_{\rm Muon} \gg 1.
		\label{equation:Muon-river-sign-condition}
	\end{equation}
	Then the river variable $\theta_k$ contracts linearly in the sense that
	\begin{equation}
		|\theta_{k+1}|
		\le
		|\theta_k|-2\sqrt{\underline\alpha}\,\eta.
		\label{equation:Muon-river-contraction}
	\end{equation}
\end{proposition}

\begin{proof}
	By~\eqref{equation:river_effective_forcing}, we have $\phi_k-\hat a_{\rm eff}\theta_k = \hat c^\top\hat D^{-1}\psi_k$.
	Applying the Cauchy-Schwarz inequality in the $\hat D^{-1}$-inner product yields
	\begin{equation}
		|\hat c^\top\hat D^{-1}\psi_k|
		\le
		\bigl(\hat c^\top\hat D^{-1}\hat c\bigr)^{1/2}
		\bigl(\psi_k^\top\hat D^{-1}\psi_k\bigr)^{1/2}
		=
		\bigl(\hat c^\top\hat D^{-1}\hat c\bigr)^{1/2}
		\sqrt{H_k}.
		\nonumber
	\end{equation}
	For $k\ge N_{\rm Muon}$, Corollary~\ref{corollary:hill-tube-muon} gives $H_k\le R_{\rm Muon}^2\eta^2$, and hence
	\begin{equation}
		|\phi_k-\hat a_{\rm eff}\theta_k|
		\le
		\bigl(\hat c^\top\hat D^{-1}\hat c\bigr)^{1/2}
		R_{\rm Muon}\eta.
		\nonumber
	\end{equation}
	If~\eqref{equation:Muon-river-sign-condition} holds, then
	\begin{equation}
		|\phi_k-\hat a_{\rm eff}\theta_k|
		\le
		\frac{\hat a_{\rm eff}}{2}|\theta_k|
		\qquad\text{and}\qquad
		E_k = \hat a_{\rm eff}\theta_k^2 
		\ge \Gamma_{\rm Muon} R_{\rm Muon}^2\eta^2
		\gg H_k.
		\nonumber
	\end{equation}
	Consequently, the iterate is in a river-dominant phase, $\phi_k$ has the same sign as $\theta_k$, and we have $|\phi_k|\ge \frac{\hat a_{\rm eff}}{2}|\theta_k|$.
	
	Moreover, the second term in~\eqref{equation:Muon-river-sign-condition} implies
	\eqref{equation:Muon-river-sign-condition} gives $|\theta_k|\ge 2\sqrt{\overline\alpha}\,\eta$.
	Since $1+\theta_k\le \overline\alpha$, we obtain
	$2\eta\sqrt{1+\theta_k} \le 2\eta\sqrt{\overline\alpha}\le |\theta_k|$.
	Thus, the river update does not cross the origin. Using the river recursion together with $\mathtt{sign}(\phi_k)=\mathtt{sign}(\theta_k)$, we obtain
	\begin{equation}
		|\theta_{k+1}|
		=
		\left|
		\theta_k
		-
		2\eta\sqrt{1+\theta_k}\,\mathtt{sign}(\phi_k)
		\right|
		=
		\left|
		\theta_k
		-
		2\eta\sqrt{1+\theta_k}\,\mathtt{sign}(\theta_k)
		\right|
		=
		|\theta_k|
		-
		2\eta\sqrt{1+\theta_k} .
		\nonumber
	\end{equation}
	Finally, since $1+\theta_k\ge \underline\alpha$, \eqref{equation:Muon-river-contraction} follows.
\end{proof}

\begin{remark}
	\label{remark:Muon-discuss}
	Proposition~\ref{proposition:Muon-river-progress-after-hill-reduction} shows that once the hill forcing $\psi_k$ is reduced to order $O(\eta)$ (e.g., via~\eqref{equation:Muon-hill-estimate}) and river energy dominates hill energy, the Muon river variable $\theta_k$ contracts linearly according to~\eqref{equation:Muon-river-contraction} until it reaches an $O(\eta)$-neighborhood of the river equilibrium $\theta=0$.
	Consequently, the number of Muon iterations required to reduce $|\theta_k|$ from a constant scale to an $O(\eta)$ scale is
	\begin{equation}
		O\!\left(\frac1\eta\right).
		\nonumber
	\end{equation}
	We emphasize that this 
	analysis under a constant step size typically yields at best an $O(\eta)$-accuracy guarantee. This behavior is consistent with the sign-based Muon update: the induced nonsmooth switching can produce persistent $O(\eta)$-scale oscillations in both the hill and river directions. Thus, 
	one should not generally expect the discrete iterates to converge exactly to the equilibrium.
\end{remark}

\begin{proposition}[Late-stage behavior in a relative river tube]
\label{proposition:late_stage_behavior_in_relative_tube}
	Assume that the river curvature $\hat a_{\rm eff}>0$.
	Fix a constant $\bar\rho\in(0,1)$.
	Suppose that, on the time interval under consideration, the iterates satisfy $0<\underline\alpha\le1+\theta_k\le\overline\alpha$
	and remain in the following relative river tube:
	\begin{equation}
		\sqrt{\bar q H_k}
		\le
		\bar\rho\hat a_{\rm eff}|\theta_k|.
	\end{equation}
	Then the following statements hold.
	
	\textbf{(i) Vanilla GD naturally refines.}
	If the step size satisfies
	$
	0<\eta
	\le \min
	\{
		\frac{1}{
			4\overline\alpha(1+\bar\rho)\hat a_{\rm eff}
		},
		\frac{1}{
			4\underline\alpha(1-\bar\rho)\hat a_{\rm eff}
		}
	\}
	$,
	then vanilla GD can refine the river variable from $O(\eta)$ to any target accuracy $\varepsilon\ll \eta$.
	
	\textbf{(ii) Fixed-step simplified Muon overshoots near the river equilibrium.}
	Under the relative tube condition, if $0<|\theta_k|<\eta\sqrt{1+\theta_k}$,
	then the river energy increases after one Muon step: $E_{k+1}>E_k$.
\end{proposition}
\begin{proof}
	Recall that \eqref{equation:river_effective_forcing} is
	$
		\phi_k 
		=
		\hat a_{\rm eff}\theta_k
		+
		\hat c^\top\hat D^{-1}\psi_k
	$.
	By Cauchy-Schwarz in the $\hat D^{-1}$-inner product,
	\begin{equation}
		|
		\hat c^\top\hat D^{-1}\psi_k
		|^2
		\le
		(\hat c^\top\hat D^{-1}\hat c)
		(\psi_k^\top\hat D^{-1}\psi_k)
		=
		\bar q H_k .
		\nonumber
	\end{equation}
	Therefore, $\left|\phi_k-\hat a_{\rm eff}\theta_k
	\right|\le\sqrt{\bar q H_k}$.
	Using the relative tube condition \eqref{equation:relative_river_tube}, we obtain
	\begin{equation}
		\left|
		\phi_k-\hat a_{\rm eff}\theta_k
		\right|
		\le
		\bar\rho\hat a_{\rm eff}|\theta_k|.
		\nonumber
	\end{equation}
	Since $\bar\rho\in(0,1)$ and $\hat a_{\rm eff}>0$, this implies that $\phi_k$ has the same sign as $\theta_k$ whenever $\theta_k\neq0$. 
	It also implies the two-sided bound
	\begin{equation}
		(1-\bar\rho)\hat a_{\rm eff}|\theta_k|
		\le
		|\phi_k|
		\le
		(1+\bar\rho)\hat a_{\rm eff}|\theta_k|.
		\label{equation:relative_tube_force_bounds}
	\end{equation}
	We now prove the vanilla GD refinement statement. 
	The vanilla GD river iteration is
	\begin{equation}
		\theta_{k+1}
		=
		\theta_k
		-
		4\eta(1+\theta_k)\phi_k .
		\nonumber
	\end{equation}
	Since $\phi_k$ and $\theta_k$ have the same sign, the update moves toward zero in the river coordinate. 
	Using the upper bound in \eqref{equation:relative_tube_force_bounds} and $1+\theta_k\le\overline{\alpha}$,
	\begin{equation}
		4\eta(1+\theta_k)|\phi_k|
		\le
		4\eta\overline\alpha
		(1+\bar\rho)\hat a_{\rm eff}|\theta_k|.
		\nonumber
	\end{equation}
	By the step-size condition 
	$
	0<\eta
	\le
	\frac{1}{
		4\overline\alpha(1+\bar\rho)\hat a_{\rm eff}
	}
	$, 
	we have $4\eta(1+\theta_k)|\phi_k|\le|\theta_k|$.
	Hence GD does not overshoot zero, and
	\begin{equation}
		|\theta_{k+1}|
		\le
		|\theta_k|
		-
		4\eta\underline\alpha
		(1-\bar\rho)\hat a_{\rm eff}|\theta_k|
		=
		\left[
		1-
		4\eta\underline\alpha(1-\bar\rho)\hat a_{\rm eff}
		\right]
		|\theta_k|.
		\nonumber
	\end{equation}
	Iterating the contraction gives
	\begin{equation}
		|\theta_{k_0+t}|
		\le
		\left[
		1
		-
		4\eta\underline\alpha(1-\bar\rho)\hat a_{\rm eff}
		\right]^t
		|\theta_{k_0}|
		\le
		\varepsilon.
		\nonumber
	\end{equation}
	Thus it is enough to take
	\begin{equation}
		t
		\ge
		\frac{
			1
		}{
			4\eta\underline\alpha(1-\bar\rho)\hat a_{\rm eff}
		}
		\log
		\frac{|\theta_{k_0}|}{\varepsilon}.
		\nonumber
	\end{equation}
	
	We next prove the simplified Muon overshoot statement. 
	Under the relative tube condition, the simplified Muon river iteration becomes
	\begin{align}
		\theta_{k+1}
		&=\theta_k-2\eta\sqrt{1+\theta_k}\,\mathtt{sign}(\phi_k)\\
		&=\theta_k-2\eta\sqrt{1+\theta_k}\,\mathtt{sign}(\theta_k)\\
		&=\mathtt{sign}(\theta_k)
		\left(|\theta_k|-2\eta\sqrt{1+\theta_k}\right).
		\nonumber
	\end{align}
	$\theta_{k+1}$ has the opposite sign from $\theta_k$ if $|\theta_k| < 2\eta\sqrt{1+\theta_k}$. 
	This proves $\theta_k\theta_{k+1}<0$.
	
	It remains to prove the one-step river-energy increase. 
	Since $E_k=\hat a_{\rm eff}\theta_k^2$,
	we only need to compare $|\theta_{k+1}|$ and $|\theta_k|$. 
	We already have
	\begin{equation}
		|\theta_{k+1}|
		=
		||\theta_k|-2\eta\sqrt{1+\theta_k}|.
		\nonumber
	\end{equation}
	If $0<|\theta_k|<\frac{2\eta\sqrt{1+\theta_k}}{2}$, then
	\begin{equation}
		|\theta_{k+1}|
		=
		||\theta_k|-2\eta\sqrt{1+\theta_k}|
		=
		2\eta\sqrt{1+\theta_k} - |\theta_k|
		>
		|\theta_k|.
		\nonumber
	\end{equation}
	Therefore
	\begin{equation}
		E_{k+1}
		=
		\hat a_{\rm eff}\theta_{k+1}^2
		>
		\hat a_{\rm eff}\theta_k^2
		=
		E_k.
		\nonumber
	\end{equation}
	Consequently, whenever a late-stage iterate enters the interval
	$0<|\theta_k|<\eta\sqrt{1+\theta_k}$, the fixed-step update increases
	the river energy.  This is the local instability mechanism near the
	equilibrium; exact entry into the interval depends on the preceding
	discrete orbit.
	This completes the proof. 
\end{proof}

\begin{remark}[Why the relative river-tube condition is natural]
	\label{remark:relative_river_tube_natural}
	The relative river-tube condition
	\begin{equation}
		\sqrt{\bar q H_k}
		\le
		\bar\rho\hat a_{\rm eff}|\theta_k|,
		\qquad
		0<\bar\rho<1,
		\label{equation:relative_river_tube_remark}
	\end{equation}
	should be interpreted as the late-stage version of river dominance. 
	Indeed, the river force admits the decomposition
	$
	\phi_k
	=
	\hat a_{\rm eff}\theta_k
	+
	\hat c^\top\hat D^{-1}\psi_k,
	$
	and the perturbation caused by the residual hill force satisfies
	\begin{equation}
		\bigl|
		\hat c^\top\hat D^{-1}\psi_k
		\bigr|
		\le
		\sqrt{\bar q H_k},
		\qquad
		\bar q:=\hat c^\top\hat D^{-1}\hat c .
		\nonumber
	\end{equation}
	Therefore the relative tube condition is exactly the requirement that the hill-induced perturbation of the river force is at most a $\rho$-fraction of the intrinsic river force $\hat a_{\rm eff}\theta_k$. 
	Equivalently, it guarantees that $\phi_k$ has the same sign as $\theta_k$, so the dynamics can still be interpreted as refinement along the river.
	
	This condition is automatic on the exact reduced river $\psi_k=0$, and it also holds in any sufficiently thin $O(|\theta_k|)$-neighborhood (now $\theta_k$ enters an $O(\eta)$ neighborhood of equilibrium) of that river:
	\begin{equation}
		H_k
		\le
		\frac{\bar\rho^2\hat a_{\rm eff}^2}{\bar q}\,\theta_k^2
		\quad
		\Rightarrow
		\quad
		\sqrt{\bar q H_k}
		\le
		\bar\rho\hat a_{\rm eff}|\theta_k|.
		\nonumber
	\end{equation}
	Thus the condition is not an additional geometric object; it is a quantitative way of saying that, even in the coupled PSD case, the iterate is close enough to the tilted river for the remaining force to be river-dominated.
\end{remark}

In the next Appendix~\ref{subsubsec:ngd_smuon_psd_dynamics}, we can see the normalized variants have the same continuous river--hill curves but a different clock.  More precisely, under the hypotheses of the preceding
theorem and the effective-step conditions, NGD inherits the GD hill reduction and
has an $O(\eta)$ additive river decrement in the late tube.  Thus it reaches
an $O(\eta)$ neighborhood in $O(1/\eta)$ iterations, but its unit-speed
update need not vanish as the residual goes to zero and fixed-step iterates
may overshoot.  SMuon inherits the Muon path and has an
$O(\eta|\theta_k|)$ river decrement when $\chi_k\asymp|\theta_k|$; its
gradient-dependent speed vanishes at the solution, so it supports
GD-like refinement below the $O(\eta)$ scale.  The precise recursions,
hill-energy estimates, and constants are given in Proposition~
\ref{proposition:ngd_smuon_river_consequences}.

\subsubsection{Reduced Dynamics of Normalized GD and Scaled Muon in the PSD case}
\label{subsubsec:ngd_smuon_psd_dynamics}
Let
\begin{equation}
    \chi_k:=2\sqrt{u_k^2\alpha_k+v_k^2\beta_k+w_k^2\gamma_k},
    \qquad
    m_k:=\left\|\begin{bmatrix}
      \mathtt{sign}(u_ka_k)\\\mathtt{sign}(v_kb_k)\\\mathtt{sign}(w_kc_k)
    \end{bmatrix}\right\|_2,
    \nonumber
\end{equation}
where fixed factors depending on $r$ are absorbed into LR.
At $\chi_k=0$ we use the zero-update convention; otherwise the normalized
updates below are defined on the active sign cells.
The reduced NGD and SMuon Euler iterations are the GD and Muon iterations with
state-dependent effective steps
\begin{equation}
    \eta_k^{\rm NGD}:=\frac{\eta}{\chi_k},
    \qquad
    \eta_k^{\rm SMuon}:=\frac{\eta\chi_k}{m_k}.
\end{equation}
Explicitly,
\begin{equation}
    \begin{cases}
    \theta_{k+1}=\theta_k-4\eta_k^{\rm NGD}(1+\theta_k)\phi_k,\\
    \omega_{k+1}=\omega_k-4\eta_k^{\rm NGD}\mathtt{Diag}(\omega_k)\psi_k,
    \end{cases}
    \label{equation:ngd_psd_iterations}
\end{equation}
and
\begin{equation}
    \begin{cases}
    \theta_{k+1}=\theta_k-2\eta_k^{\rm SMuon}\sqrt{1+\theta_k}\,s_k,\\
    \omega_{k+1}=\omega_k-2\eta_k^{\rm SMuon}\mathtt{Diag}(\omega_k)^{1/2}t_k,
    \end{cases}
    \label{equation:smuon_psd_iterations}
\end{equation}
with $s_k\in\mathtt{sign}(\phi_k)$ and $t_k\in\mathtt{sign}(\psi_k)$.  Thus
the GD hill-energy and river-progress estimates apply to NGD after replacing
$\eta$ by $\eta_k^{\rm NGD}$, while the Muon estimates apply to SMuon after
replacing $\eta$ by $\eta_k^{\rm SMuon}$.  In continuous time, NGD is exactly
the GD path with clock $d\tau/dt=1/\chi$, and SMuon is exactly the Muon path
with clock $d\tau/dt=\chi/m$.
Here $C_{\rm GD}(h)$ denotes the GD estimate constant from
Proposition~\ref{proposition:hill-energy-gd} with step size set to $h$.
In particular, the inherited one-step hill estimates are 
\begin{align}
     H_{k+1}^{\rm NGD}
     &\le (1-2\underline\omega\mu\eta_k^{\rm NGD})H_k^{\rm NGD}
          +C_{\rm GD}(\eta_k^{\rm NGD})\eta_k^{\rm NGD}\phi_k^2,
     \label{equation:ngd_hill_energy_estimate}\\
     H_{k+1}^{\rm SMuon}
     &\le H_k^{\rm SMuon}
          -4B_0\eta_k^{\rm SMuon}\sqrt{H_k^{\rm SMuon}}
          +B_1(\eta_k^{\rm SMuon})^2,
     \label{equation:smuon_hill_energy_estimate}
\end{align}
whenever the corresponding parent estimate and effective-step smallness
conditions hold.  Therefore NGD has geometric hill contraction in its
gradient-time variable, while SMuon has the same square-root drift toward an
$O(\eta_k^{\rm SMuon})$ hill tube as Muon.

\begin{proposition}[River-valley consequences for NGD and SMuon]
\label{proposition:ngd_smuon_river_consequences}
Assume $\hat a_{\rm eff}>0$, the iterates remain in the compact box, and the
relative river-tube condition~\ref{equation:relative_river_tube} holds.
In continuous time, NGD has the same river/hill path as GD and SMuon has the
same path as Muon.  In a region where $\chi_k$ is bounded away from zero, their progress
rates are the corresponding GD/Muon rates up to a constant time rescaling.
In late river-dominant region, if
$c_-|\theta_k|\le\chi_k\le c_+|\theta_k|$ for fixed $c_-,c_+>0$, then
\begin{equation}
    |\theta_{k+1}^{\rm NGD}|\le|\theta_k^{\rm NGD}|-c_{\rm NGD}\eta,
    \qquad
    |\theta_{k+1}^{\rm SMuon}|\le(1-c_{\rm SMuon}\eta)|\theta_k^{\rm SMuon}|,
    \label{equation:ngd_smuon_late_rates}
\end{equation}
outside an $O(\eta)$ tube, for constants $c_{\rm NGD},c_{\rm SMuon}>0$.
If $C_{\rm NGD}$ and $C_{\rm SMuon}$ denote the corresponding tube constants,
the displayed inequalities imply the hitting-time bounds
\begin{equation}
    \begin{aligned}
        T_{\rm NGD}-N_{\rm NGD}
        &\le \left\lceil
        \frac{(|\theta_{N_{\rm NGD}}|-C_{\rm NGD}\eta)_+}{c_{\rm NGD}\eta}
        \right\rceil
        =O(1/\eta),
        \\
        T_{\rm SMuon}-N_{\rm SMuon}
        &\le \left\lceil
        \frac{\log(|\theta_{N_{\rm SMuon}}|/(C_{\rm SMuon}\eta))}
        {-\log(1-c_{\rm SMuon}\eta)}
        \right\rceil
        =O\!\left(\frac{1}{\eta}\log\frac{1}{\eta}\right).
    \end{aligned}
    \nonumber
\end{equation}
whenever the logarithm is positive and $0<c_{\rm SMuon}\eta<1$. Hence, NGD retains GD's geometry and $O(1/\eta)$ late-stage hitting time but can overshoot because its unit-speed update does not vanish at the solution; SMuon retains Muon's straight, uniform exploration path while its gradient-dependent speed vanishes with residual and permits asymptotic GD-like refinement below the $O(\eta)$ scale.
\end{proposition}

\begin{proof}
Equation~\ref{equation:ngd_psd_iterations} and \ref{equation:smuon_psd_iterations} is obtained by substituting
the direction--speed rows of Table~\ref{table:direction-speed-decomp-flow}
into the reduced GD and Muon recursions.  The path identities follow from
Lemma~\ref{lemma:ngd_smuon_dynamics}.  In the relative tube,
$\phi_k$ has the sign of $\theta_k$ and is comparable to
$\hat a_{\rm eff}\theta_k$ by~\eqref{equation:relative_tube_force_bounds}.
Since $\chi_k\asymp|\theta_k|$, division by $\chi_k$ turns the GD decrement
into an $O(\eta)$ decrement (NGD), whereas multiplication by $\chi_k$ turns
the Muon decrement into an $O(\eta|\theta_k|)$ decrement (SMuon), proving~\eqref{equation:ngd_smuon_late_rates}.  The hill estimates inherit the parent
ones whenever the effective-step smallness conditions hold.
\end{proof}

\section{Additional Details for Section~\ref{sec:river-valley-as-tool-in-generalized-settings}}
\label{sec:appendix-spectral-river-theory}
\subsection{Spectral-River Geometry} 
\label{subsec:spectral-river-geometry}

In this section, we study the optimization of a general objective function $\mathcal{F}(X): \mathbb{R}^{n\times n} \to \mathbb{R}$ with the goal of finding a critical point $X^\star \in \mathbb{R}^{n\times n}$, i.e., $\nabla \mathcal{F}(X^\star) = 0$. We assume that the target matrix $X^\star$ (also called the ground truth) admits the singular value decomposition (SVD)
\begin{equation}
	X^\star = U_\star \Sigma_\star V_\star^\top,
	\qquad
	U_\star^\top U_\star = I_n,
	\qquad
	V_\star^\top V_\star = I_n,
	\qquad
	\Sigma_\star = \mathtt{Diag}(\sigma_1, \ldots, \sigma_n),
	\nonumber
\end{equation}
with $\sigma_i > 0$.
The rank-deficient case follows by replacing $n$ with the signal rank $r$ and restricting attention to the corresponding signal block. Therefore, we focus on the square case for simplicity, though the optimization variable $X \in \mathbb{R}^{n\times n}$ can be generalized to a rectangular matrix.

\begin{definition}[Spectral embedding and spectral river] 
	\label{definition:spectral-embedding-spectral-river}
	We define the spectral embedding $\mathcal{B}: \mathbb{R}^n \to \mathbb{R}^{n\times n}$ by $\mathcal{B}(x) := U_\star \mathtt{Diag}(x) V_\star^\top$ for any $x\in\mathbb{R}^n$.
	The spectral river through $X^\star$ is the affine set
	\begin{equation}
		\mathcal{R}^\star :=
		\{
		X^\star + \mathcal{B}(x) \,\mid\,
		x\in\mathbb{R}^n
		\}
		\subset\mathbb{R}^{n\times n}.
		\label{equation:spectral-river}
	\end{equation}
	Since $X^\star=\mathcal{B}(\sigma)$ with $\sigma := [\sigma_1,\ldots,\sigma_n]^\top$, any $Y\in \mathcal{R}^\star$ can be written as $Y=\mathcal{B}(y)$ for some $y\in\mathbb{R}^n$, equivalently $Y = U_\star \mathtt{Diag}(y) V_\star^\top$.
	Therefore, any point $Y$ on the river shares the left/right singular vector subspaces $(U_\star, V_\star)$ with $X^\star$; only the singular-value magnitudes vary along the river. Moreover, $\mathcal{B}(\mathbf{1}) = U_\star V_\star^\top$ and $\|\mathcal{B}(x)\|_F = \|x\|_2$.
\end{definition}

\begin{definition}[River-restricted objective] 
	\label{definition:river-restricted-objective}
	We define the objective restricted to the spectral river (i.e., the pullback of $\mathcal{F}$ under the map $x \mapsto X^\star+\mathcal{B}(x)$) as
	\begin{equation}
		f(x) := \mathcal{F}\bigl(X^\star+\mathcal{B}(x)\bigr) - \mathcal{F}(X^\star),
		\qquad x\in\mathbb{R}^n.
		\label{equation:river-restricted-objective}
	\end{equation}
\end{definition}

\begin{definition}[Positive chamber] 
	\label{definition:positive-chamber}
	We define the positive chamber of the spectral river as
	\begin{equation}
		\mathcal{R}^\star_{+} :=
		\{
		X^\star + \mathcal{B}(x) \,\mid\,
		\sigma + x \succ \mathbf{0}
		\}
		=
		\{
		U_\star \mathtt{Diag}(y) V_\star^\top \,\mid\,
		y\succ \mathbf{0}
		\}.
		\label{equation:positive-chamber}
	\end{equation}
	Within this chamber, the polar factor is constant and equals $\mathtt{msign}(Y)=\mathcal{B}(\mathbf{1})$ for all $Y\in\mathcal{R}^\star_{+}$.
	Moreover, if some entries of $y$ are negative (so that $Y=U_\star \mathtt{Diag}(y) V_\star^\top$ has mixed signs), then the polar factor becomes $U_\star \mathtt{Diag}(\mathtt{sign}(y)) V_\star^\top$.
\end{definition}

\begin{definition}[Orthogonal basis and projection operator] 
	\label{definition:orthogonal-basis-projection-operator}
	For each $i=1,\cdots,n$, define the basis matrices for the river direction space $\mathtt{span}\{B_1,\ldots,B_n\}$ using the $i$-th standard basis vector $e_i \in \mathbb{R}^n$:
	\begin{equation}
		B_i:=U_\star e_i e_i^\top V_\star^\top.
		\label{equation:orthogonal-basis}
	\end{equation}
	These matrices form an orthonormal set under the Frobenius inner product, i.e., $\langle B_i, B_j\rangle = \delta_{ij}$.
	Consequently, $\mathcal{B}(x)=\sum_{i=1}^n x_i B_i$ and $\langle\mathcal{B}(x),\mathcal{B}(y)\rangle=x^\top y$.
	Moreover, for any $Y\in\mathbb{R}^{n\times n}$, the (Frobenius) orthogonal projection onto $\mathtt{span}\{B_1,\cdots,B_n\}$ is
	\begin{equation}
		\mathtt{Proj}_{\mathcal{R}}(Y)
		=
		\sum\nolimits_{i=1}^n \ip{B_i}{Y} B_i
		=
		U_\star \mathtt{Diag}
		\bigl(
		\mathtt{diag}(U_\star^\top Y V_\star)
		\bigr)
		V_\star^\top.
		\label{equation:river-projector}
	\end{equation}
\end{definition}

\begin{definition}[Spectral-river tube]
	\label{definition:spectral-river-tube}
	Fix a scale $\delta>0$.
	Define the $\delta$-tube around the spectral river $\mathcal{R}^\star:=X^\star+\mathtt{span}\{B_1,\cdots,B_n\}$ by
	\begin{equation}
		\mathcal{T}_\delta(\mathcal{R}^\star)
		:=
		\left\{
		X\in\mathbb{R}^{n\times n}
		\,\middle|\,
		\mathtt{dist}_F\!\left(X,\mathcal{R}^\star\right)\le \delta
		\right\},
		\label{equation:spectral-river-tube}
	\end{equation}
	where $\mathtt{dist}_F(X,\mathcal{R}^\star):=\inf_{Y\in\mathcal{R}^\star}\|X-Y\|_F$.
	Equivalently, each $X$ admits a unique orthogonal decomposition
	\begin{equation}
		X
		=
		X^\star+\mathcal{B}(x)+E^\perp,
		\qquad
		E^\perp\in \mathtt{span}\{B_1,\cdots,B_n\}^\perp,
		\nonumber
	\end{equation}
	and $X\in \mathcal{T}_\delta(\mathcal{R}^\star)$ if and only if $\|E^\perp\|_F\le \delta$.
\end{definition}

\begin{definition}[Hill component and river-dominance]
	\label{definition:hill-component-and-river-dominance}
	For any $Y\in\mathbb{R}^{n\times n}$, we define its hill component as $\mathtt{Proj}_{\mathcal R^\perp}(Y)$,
	where $\mathtt{Proj}_{\mathcal R^\perp}(\cdot)$ denotes the Frobenius-orthogonal projections onto
	$\mathtt{span}\{B_1,\cdots,B_n\}$'s orthogonal complement, i.e., $\mathtt{span}\{B_1,\cdots,B_n\}^\perp$.
	We say that $Y$ is river-dominated if its hill component is small relative to its river component.
\end{definition}

\begin{lemma}[Hessian projection] 
	\label{lemma:hessian-projection}
	Let $\mathcal{H}_\star:=\nabla^2 \mathcal{F}(X^\star)$, viewed as a linear operator acting on matrices. 
	If the river residual is $\mathcal{B}(x)$, then its Hessian image projected back onto the river direction space satisfies
	\begin{align}
		&\mathtt{Proj}_{\mathcal{R}} 
		\bigl(\mathcal{H}_\star[\mathcal{B}(x)]\bigr)
		=
		\sum\nolimits_{i=1}^n 
		\ip{B_i}{\mathcal{H}_\star[\mathcal{B}(x)]} B_i 
		=
		\sum\nolimits_{i=1}^n 
		\ip{B_i}{\sum\nolimits_{j=1}^n x_j \mathcal{H}_\star[B_j]} B_i 
		\nonumber \\
		=&
		\sum\nolimits_{i=1}^n 
		\left(
		\sum\nolimits_{j=1}^n x_j
		\ip{B_i}{\mathcal{H}_\star[B_j]} 
		\right) B_i 
		=
		\sum\nolimits_{i=1}^n
		\left(
		\sum\nolimits_{j=1}^n K_{ij} x_j
		\right) B_i
		=
		\mathcal{R}(Kx),
		\label{equation:hessian-projection}
	\end{align}
	where $K\in\mathbb{R}^{n\times n}$ is the coordinate matrix of the projected Hessian action in the basis $\{B_i\}_{i=1}^n$, defined by
	\begin{equation}
		K_{ij}:=\ip{B_i}{\mathcal{H}_\star[B_j]}.
	\end{equation}
	If $\mathcal{F}$ is twice continuously differentiable, then $\mathcal{H}_\star$ is self-adjoint w.r.t. the Frobenius inner product, hence $K$ is symmetric.
	Moreover, near the critical point $X^\star$ (as $a\to 0$), a second-order expansion yields
	\begin{equation}
		\nabla \mathcal{F}(X^\star+a\mathcal{B}(x)) = a\,\mathcal{H}_\star[\mathcal{B}(x)] + O(a^2).
		\nonumber
	\end{equation}
	If the off-river component of $\mathcal{H}_\star[\mathcal{B}(x)]$ is negligible so that $\mathcal{H}_\star[\mathcal{B}(x)]\approx \mathtt{Proj}_{\mathcal{R}}\bigl(\mathcal{H}_\star[\mathcal{B}(x)]\bigr)=\mathcal{B}(Kx)$ and $Kx\succ\mathbf{0}$, then the resulting gradient direction lies in the positive chamber, and hence $\mathtt{msign}(\mathcal{B}(Kx))=\mathcal{B}(\mathbf{1})$.
	Note that Muon observes $\mathcal{H}_\star[\mathcal{B}(x)]$ rather than $\mathcal{B}(x)$ itself; however, in later analysis we will assume that the river component dominates.
\end{lemma}

\subsection{Early-Stage: When Muon Accelerates River Progress} 
\label{subsec:early-muon-general}

This subsection establishes an early-stage regime in which Muon accelerates progress along the spectral river. Our goal is to isolate conditions under which the effective curvature restricted to the river is substantially smaller than the ambient curvature in the full matrix space. In such a regime, a Muon step that is aligned with the river geometry can admit a provably larger decrease than a conservative gradient step.
We adopt the following local smoothness assumption.

\begin{assumption}[Local smoothness separation for $\mathcal{F}$ and $f$] 
	\label{assumption:local-smoothness}
	Fix a neighborhood of interest around the current iterate(s). There exist constants $L_{\rm full},L_{\rm riv}>0$ with $L_{\rm riv} \ll L_{\rm full}$ such that, for all $X$ in this neighborhood and all sufficiently small perturbations $\Delta X \in \mathbb{R}^{n \times n}$,
	\begin{equation}
		\|\nabla \mathcal{F}(X + \Delta X) - \nabla \mathcal{F}(X)\|_F \leq L_{\rm full} \,\|\Delta X\|_F.
		\label{equation:F-local-smooth}
	\end{equation}
	Moreover, for all $x$ in the corresponding river neighborhood and all sufficiently small $\Delta x \in \mathbb{R}^n$, the river-restricted objective $f$ satisfies
	\begin{equation}
		f(x + \Delta x) \leq f(x) + \langle \nabla f(x), \Delta x \rangle + \frac{L_{\rm riv}}{2}\|\Delta x\|_2^2.
		\label{equation:f-descent-lemma}
	\end{equation}
	Here, \eqref{equation:F-local-smooth} is the standard Lipschitz continuity of $\nabla\mathcal{F}$ (in Frobenius norm), while \eqref{equation:f-descent-lemma} is the usual quadratic upper bound (descent lemma~\cite{nesterov2018lectures}) for $f$, which is implied by $\nabla f$ being $L_{\rm riv}$-Lipschitz on the neighborhood.
\end{assumption}
Because $L_{\rm full}$ controls smoothness in all directions, it can be dominated by high-curvature components (e.g., ``spike'' directions). In contrast, $L_{\rm riv}$ quantifies curvature only along the signal-aligned, singular-value directions parametrized by the river. This separation generalizes the mixed-spiked MS dynamics developed in this paper: river coefficients may remain well-conditioned even when transverse ``hill'' directions are stiff and therefore force conservative gradient steps. Consequently, once the iterates lie on or near the river, the effective constant $L_{\rm riv}$ governing Muon's motion can be substantially smaller than the global constant $L_{\rm full}$ that limits vanilla gradient methods.

\paragraph{Early-stage: a one-step comparison.}
We formalize a simple one-step manifestation of the early-stage speedup. 
Fix a point on the spectral river
\begin{equation}
	X_0 \;=\; X^\star + \mathcal{B}(x_0).
	\nonumber
\end{equation}
(Here ``$0$'' indicates the point at which we start the analysis, not necessarily the initialization of the algorithm.)
We compare the certified one-step decrease obtained by a momentum-Muon step aligned with the river geometry versus a conservative gradient step whose step size is limited by the ambient smoothness.
\begin{itemize}
	\item \textbf{One-step momentum-Muon (ambient form).} Given a stored momentum matrix $M_0^{\rm Muon}\in\mathbb{R}^{n\times n}$ and momentum parameter $\xi \geq 0$, define
	\begin{equation}
		M_1^{\rm Muon} = \xi M_0^{\rm Muon} + \nabla \mathcal{F}(X_0),
		\qquad
		X_1^{\rm Muon} = X_0 - \eta^{\rm Muon}\, \mathtt{msign}(M_1^{\rm Muon}).
		\nonumber
	\end{equation}
	
	\item \textbf{One-step vanilla gradient descent (ambient form).} With step size $\eta^{\rm GD}>0$,
	\begin{equation}
		X_1^{\rm GD} \;=\; X_0 - \eta^{\rm GD}\nabla \mathcal{F}(X_0).
		\nonumber
	\end{equation}
\end{itemize}
The key technical simplification is that, under a (local) \emph{gradient-invariance} condition, the ambient gradients and momenta can be represented in the river coordinates, allowing us to reduce the comparison from $\mathcal{F}$ on matrices to $f$ on $\mathbb{R}^n$.

\begin{lemma}[River reduction under gradient invariance] 
	\label{lemma:river-reduction}
	Suppose the Muon momentum is river-aligned in the sense that
	\begin{equation}
		M_0^{\rm Muon}=\mathcal{B}(m_0)
		\quad\text{and}\quad
		M_1^{\rm Muon}=\mathcal{B}(m_1)
		\qquad
		\text{for some } m_0,m_1\in\mathbb{R}^n.
		\nonumber
	\end{equation}
	Assume moreover that the river is gradient-invariant at $X_0=X^\star+\mathcal{B}(x_0)$, namely
	\begin{equation}
		\nabla \mathcal{F}(X_0)
		=
		\mathcal{B}\bigl(\nabla f(x_0)\bigr).
		\label{equation:grad-invariance}
	\end{equation}
	Then the momentum update in the ambient space is equivalent to the river-coordinate update
	\begin{equation}
		M_1^{\rm Muon}=\xi M_0^{\rm Muon}+\nabla \mathcal{F}(X_0)
		\quad\Longleftrightarrow\quad
		m_1=\xi m_0+\nabla f(x_0).
		\nonumber
	\end{equation}
	In particular, the corresponding sign directions satisfy
	$
	\mathtt{msign}(M_1^{\rm Muon})
	=
	\mathcal{B}(\mathtt{sign}(m_1))
	$.
	
	Furthermore, if $m_0$ is coordinate-wise sign-aligned with $\nabla f(x_0)$, then
	\begin{equation}
		\langle \nabla f(x_0), \mathtt{sign}(m_1)\rangle
		=\|\nabla f(x_0)\|_1.
		\nonumber
	\end{equation}
	More generally, without requiring sign alignment, define
	\begin{equation}
		\rho \;:=\; 
		\frac{\langle \nabla f(x_0), \mathtt{sign}(m_1)\rangle}
		{\|\nabla f(x_0)\|_1},
		\nonumber
	\end{equation}
	whenever $\nabla f(x_0)\neq 0$. Then $\rho\in(0,1]$ precisely captures the degree of sign agreement and yields
	\begin{equation}
		\langle \nabla f(x_0), \mathtt{sign}(m_1)\rangle 
		\;\ge\; 
		\rho\,\|\nabla f(x_0)\|_1.
		\label{equation:rho-inequality}
	\end{equation}
\end{lemma}

\begin{proof}
	The equivalence follows by substituting
	$M_0^{\rm Muon}=\mathcal{B}(m_0)$ and \eqref{equation:grad-invariance} into Muon's momentum update and using linearity of $\mathcal{B}$.
	The statement about sign directions follows from the definition of $\mathtt{msign}(\cdot)$ and the coordinate-wise nature of $\mathtt{sign}(\cdot)$.
	If $m_0$ and $\nabla f(x_0)$ are coordinate-wise sign-aligned, then $m_1=\xi m_0+\nabla f(x_0)$ has the same sign pattern as $\nabla f(x_0)$, hence
	$\langle \nabla f(x_0), \mathtt{sign}(m_1)\rangle=\|\nabla f(x_0)\|_1$.
    Equation~\ref{equation:rho-inequality} is immediate from the definition of $\rho$.
\end{proof}

By Lemma~\ref{lemma:river-reduction}, it suffices to compare the following \emph{river-coordinate} one-step updates:
\begin{itemize}
	\item \textbf{One-step momentum-Muon (river form).}
	\begin{equation}
		m_1=\xi m_0+\nabla f(x_0),
		\qquad
		x_1^{\rm Muon}
		=x_0-\eta^{\rm Muon}\,\mathtt{sign}(m_1).
		\nonumber
	\end{equation}
	
	\item \textbf{One-step vanilla gradient descent (river form).}
	\begin{equation}
		x_1^{\rm GD}
		=x_0-\eta^{\rm GD}\nabla f(x_0).
		\nonumber
	\end{equation}
\end{itemize}

\begin{theorem}[Early-stage Muon acceleration along the spectral river (one-step analysis)]
	\label{theorem:early-stage-muon-acceleration-along-the-spectral-river}
	Assume the local smoothness conditions in Assumption~\ref{assumption:local-smoothness} hold and that the gradient-invariance condition~\ref{equation:grad-invariance} holds at $x_0$.
	Fix $\rho\in(0,1]$ as in \eqref{equation:rho-inequality}, and assume $\nabla f(x_0)\neq 0$.
	
	\begin{itemize}
		\item (\textbf{Muon decrease along the sign direction.})
		Let 
		$
		\eta^{\rm Muon}
		=
		\langle \nabla f(x_0), \mathtt{sign}(m_1)\rangle / (L_{\rm riv} n)
		$,
		then the one-step update $x_1^{\rm Muon}=x_0-\eta^{\rm Muon}\mathtt{sign}(m_1)$ satisfies
		\begin{equation}
			f(x_1^{\rm Muon})
			\leq
			f(x_0)
			-
			\frac{\langle\nabla f(x_0), \mathtt{sign}(m_1)\rangle^2}
			{2L_{\rm riv}\,\|\mathtt{sign}(m_1)\|_2^2}
			\leq
			f(x_0)
			-
			\frac{\rho^2\|\nabla f(x_0)\|_1^2}
			{2L_{\rm riv}\,n},
			\nonumber
		\end{equation}
		where we used $\|\mathtt{sign}(m_1)\|_2^2=n$ (assume no ties so that $\mathtt{sign}(m_1)\in\{\pm 1\}^n$) and \eqref{equation:rho-inequality}.
		
		\item (\textbf{Conservative GD decrease under ambient step size.})
		With the conservative choice $\eta^{\rm GD}=1/L_{\rm full}$, the GD update $x_1^{\rm GD}=x_0-\eta^{\rm GD}\nabla f(x_0)$ satisfies
		\begin{equation}
			f(x_1^{\rm GD})
			\leq
			f(x_0)
			-\frac{\|\nabla f(x_0)\|_2^2}{2L_{\rm full}},
			\nonumber
		\end{equation}
		whenever $L_{\rm full}\ge L_{\rm riv}$.
	\end{itemize}
	Consequently, the certified Muon decrease is larger than the certified conservative GD decrease whenever
	\begin{equation}
		\rho^2\frac{L_{\rm full}}{L_{\rm riv}}
		\cdot
		\frac{\|\nabla f(x_0)\|_1^2}{n\|\nabla f(x_0)\|_2^2}
		\;>\; 1.
		\nonumber
	\end{equation}
	Equivalently, the certified speedup factor satisfies
	\begin{equation}
		\frac{\textnormal{Muon decrease}}{\textnormal{GD decrease}}
		\geq
		\rho^2\frac{L_{\rm full}}{L_{\rm riv}}
		\cdot
		\frac{\|\nabla f(x_0)\|_1^2}{n\|\nabla f(x_0)\|_2^2}.
		\nonumber
	\end{equation}
\end{theorem}

\begin{proof}
	For the momentum-Muon step, apply the descent lemma~\cite{nesterov2018lectures} for the $L_{\rm riv}$-smooth function $f$ with direction $\mathtt{sign}(m_1)$:
	\begin{equation}
		f(x_0-\eta \mathtt{sign}(m_1))
		\leq
		f(x_0)
		-\eta
		\langle \nabla f(x_0), \mathtt{sign}(m_1)\rangle
		+
		\frac{L_{\rm riv}}{2} \eta^2 \|\mathtt{sign}(m_1)\|_2^2.
		\nonumber
	\end{equation}
	The right-hand side is minimized at 
	$
	\eta = 
	\langle \nabla f(x_0), \mathtt{sign}(m_1)\rangle / 
	(L_{\rm riv}\|\mathtt{sign}(m_1)\|_2^2)
	$,
	which yields
	\begin{equation}
		f(x_0-\eta^{\rm Muon} \mathtt{sign}(m_1))
		\leq
		f(x_0)
		-
		\frac{\langle \nabla f(x_0), \mathtt{sign}(m_1)\rangle^2}
		{2L_{\rm riv}\|\mathtt{sign}(m_1)\|_2^2}.
		\nonumber
	\end{equation}
	Using $\|\mathtt{sign}(m_1)\|_2^2=n$ together with~\eqref{equation:rho-inequality} gives the claimed bound.
	
	For gradient descent, apply the same descent lemma with direction $\nabla f(x_0)$ and step size $\eta^{\rm GD}=1/L_{\rm full}$:
	\begin{equation}
		f\!\left(x_0-\eta^{\rm GD} \nabla f(x_0)\right)
		\leq
		f(x_0)
		-
		\frac{1}{L_{\rm full}}\|\nabla f(x_0)\|_2^2
		+
		\frac{L_{\rm riv}}{2L_{\rm full}^2}\|\nabla f(x_0)\|_2^2.
		\nonumber
	\end{equation}
	If $L_{\rm full}\ge L_{\rm riv}$, then the last two terms combine to
	$
	-\|\nabla f(x_0)\|_2^2/(2L_{\rm full})
	$,
	which completes the proof.
\end{proof}

\begin{remark}[When can $L_{\rm riv}\ll L_{\rm full}$ occur?] \label{remark:when-river-smoothness-ll-hill-smoothness}
	The inequality $L_{\rm riv}\le L_{\rm full}$ is natural: $L_{\rm riv}$ quantifies smoothness only along the river directions, while $L_{\rm full}$ controls smoothness over all ambient directions.
	A pronounced gap arises when the local Hessian is much flatter on the river than on its transverse complement.
	Concretely, with respect to the orthogonal decomposition
	$
	\mathcal{V}=\mathcal{R}_\star\oplus\mathcal{R}_\star^\perp
	$,
	if the Hessian admits an (approximately) block-structured form
	\begin{equation}
		\nabla^2 \mathcal{F}(X) \approx
		\begin{pmatrix}
			K_\star & K_\circ^\top \\
			K_\circ & K_\diamond
		\end{pmatrix},
		\nonumber
	\end{equation}
	where the cross-coupling is small (e.g., $\|K_\circ\|_{\mathrm{op}}$ is small, corresponds to ``bulk'') and the transverse hill block is much stiffer than the river block (i.e., $\|K_\diamond\|_{\mathrm{op}}\gg \|K_\star\|_{\mathrm{op}}$, ``spike'' $\gg$ ``signal''), then one expects
	\begin{equation}
		L_{\rm riv}
		\approx 
		\|K_\star\|_{\mathrm{op}}
		\qquad\text{and}\qquad
		L_{\rm full}
		\approx 
		\|\nabla^2\mathcal{F}(X)\|_{\mathrm{op}}
		\approx 
		\|K_\diamond\|_{\mathrm{op}}.
		\nonumber
	\end{equation}
	In this regime, conservative gradient steps are limited by high-curvature off-river (``hill'') directions, whereas Muon can exploit the milder geometry along the information-bearing spectral river.
\end{remark}

\subsection{Late-Stage: When Muon Fails in River Refinement} 
\label{subsec:late-muon-general}

We now turn to the late-stage regime, in which the iterate is already close to the optimum and the residual magnitude is small.
In this regime, the spectral polar operator $\mathtt{msign}(\cdot)$---which discards singular-value magnitudes by keeping only their signs---can become harmful.
Intuitively, the same scale-invariance that helps Muon rapidly align with the river direction early on may prevent the algorithm from making progressively smaller, well-calibrated corrections near the solution.
As a result, Muon can exhibit a coarse ``quantized'' motion that overshoots the minimizer, whereas a suitably small gradient(-momentum) step continues to refine.

The next subsection isolates this phenomenon in an exactly solvable quadratic model, where all statements can be proved without approximation.

\subsubsection{A toy quadratic model: explicit overshoot under $\mathtt{msign}$} 
\label{subsubsec:a-toy-model}
We strip away nonessential nonlinear effects and retain only a strictly quadratic geometry:
\begin{equation}
	\mathcal{F}_{\rm quad}(X)
	:=
	\frac12
	\left\lVert X-X^\star \right\rVert_F^2,
	\qquad
	X \in \mathbb{R}^{n\times n}.
	\label{equation:toy-quad-F}
\end{equation}
For~\eqref{equation:toy-quad-F}, one has
\begin{equation}
	\nabla \mathcal{F}_{\rm quad}(X)=X-X^\star,
	\qquad
	\nabla^2\mathcal{F}_{\rm quad}(X)=I,
	\nonumber
\end{equation}
so the gradient is exactly the residual.
In particular, at any river point $X=X^\star+\mathcal{B}(x)$,
\begin{equation}
	\nabla \mathcal{F}_{\rm quad}(X)=\mathcal{B}(x),
	\nonumber
\end{equation}
hence the gradient-invariance property holds automatically (no additional assumption is needed).

The river-restricted objective induced by \eqref{equation:toy-quad-F} is
\begin{equation}
	f_{\rm quad}(x)
	:=
	\mathcal{F}_{\rm quad}\!\left(X^\star+\mathcal{B}(x)\right)
	- \mathcal{F}_{\rm quad}(X^\star)
	=
	\frac12\|\mathcal{B}(x)\|_F^2
	=
	\frac12\|x\|_2^2,
	\label{equation:toy-quad-f}
\end{equation}
and therefore $\nabla f_{\rm quad}(x)=x$ and $\nabla^2 f_{\rm quad}(x)=I_n$.

\paragraph{Update rules (river form).}
In this subsection we work directly in river coordinates, using \eqref{equation:toy-quad-f}.
We consider a momentum-Muon update that moves along the sign direction ($\mathtt{msign}$ becomes $\mathtt{sign}$):
\begin{equation}
	m_{t+1}^{\rm Muon} = \xi m_t^{\rm Muon} + x_t^{\rm Muon},
	\qquad
	x_{t+1}^{\rm Muon} = x_t^{\rm Muon}
	-\eta^{\rm Muon}\,\mathtt{sign}(m_{t+1}^{\rm Muon}),
	\nonumber
\end{equation}
and we compare it against a (heavy-ball style) momentum gradient method that uses the full momentum magnitude (i.e., a kind of momentum-GD):
\begin{equation}
	m_{t+1}^{\rm GD}=\xi m_t^{\rm GD}+x_t^{\rm GD},
	\qquad
	x_{t+1}^{\rm GD}=x_t^{\rm GD}-\eta^{\rm GD} m_{t+1}^{\rm GD}.
	\nonumber
\end{equation}
Both are initialized from the same $(x_0,m_0)$.

\begin{proposition}[Exact quadratic separation: Muon overshoots while momentum GD refines]
	\label{proposition:quad-separation}
	Fix $\xi\ge 0$ and consider \eqref{equation:toy-quad-F}.
	Let $X_0=X^\star + a\,\mathcal{B}(x_0)$ with $a\in(0,1)$ and $x_0\succ \mathbf{0}$.
	Assume the stored momentum is river-aligned and small:
	\begin{equation}
		M_0=\mathcal{B}(m_0),
		\qquad
		m_0\succ \mathbf{0},
		\qquad
		\|m_0\|_2\le C_m\,a,
		\nonumber
	\end{equation}
	for some constant $C_m$ independent of $a$.
	Define the coordinate average $\bar{x}_0:=\frac1n\mathbf{1}^\top x_0$.
	
	Choose a Muon step size of the form
	\begin{equation}
		\eta^{\rm Muon}=(1+\varepsilon)\,a\,\bar{x}_0,
		\qquad
		\text{for some }\varepsilon\in(0,1).
		\nonumber
	\end{equation}
	Assume moreover that the second Muon momentum remains in the positive chamber:
	\begin{equation}
		\xi^2 m_0 + a
		\left[
		(1+\xi)x_0-(1+\varepsilon)\bar{x}_0\,\mathbf{1}
		\right]
		\succ \mathbf{0}.
		\label{equation:toy-in-positive-chamber}
	\end{equation}
	Then the Muon iterates satisfy a strict two-step overshoot:
	\begin{equation}
		\mathcal{F}_{\rm quad}\!\left(X_1^{\rm Muon}\right)
		<
		\mathcal{F}_{\rm quad}(X_0)
		<
		\mathcal{F}_{\rm quad}\!\left(X_2^{\rm Muon}\right).
		\label{equation:toy-muon-overshoot}
	\end{equation}
	Moreover, there exists a constant $\eta_0>0$, depending only on $(\xi,x_0,C_m)$ but not on $a$, such that the momentum-GD iterates initialized from the same $(X_0,M_0)$ satisfy
	\begin{equation}
		\mathcal{F}_{\rm quad}\!\left(X_2^{\rm GD}\right)
		<
		\mathcal{F}_{\rm quad}(X_0)
		\label{equation:toy-gd-decrease}
	\end{equation}
    for all $0<\eta^{\rm GD}\le \eta_0$.
	In particular, since $\eta^{\rm Muon}=O(a)$ while $\eta_0=O(1)$, we have $\eta^{\rm Muon}<\eta_0$ for all sufficiently small $a$.
\end{proposition}

\begin{proof}
	Throughout the proof, we use the river objective \ref{equation:toy-quad-f}, i.e., $f_{\rm quad}(x)=\frac12\|x\|_2^2$, and the identity
	$
	\mathcal{F}_{\rm quad}(X^\star+\mathcal{B}(x))
	-
	\mathcal{F}_{\rm quad}(X^\star)
	= f_{\rm quad}(x)
	$.
	Thus it suffices to compare the squared norms of the corresponding river coordinates.
	
	\paragraph{Step 1: Muon decreases at the first step.}
	At $X_0=X^\star+a\mathcal{B}(x_0)$ we have $\nabla \mathcal{F}_{\rm quad}(X_0) = a\mathcal{B}(x_0)$, hence the first Muon momentum is
	\begin{equation}
		M_1^{\rm Muon}
		=
		\xi M_0+\nabla\mathcal{F}_{\rm quad}(X_0)
		=
		\mathcal{B}(\xi m_0+a x_0).
		\nonumber
	\end{equation}
	Since $m_0\succ \mathbf{0}$ and $x_0\succ \mathbf{0}$, we obtain $\xi m_0+a x_0\succ \mathbf{0}$, and therefore $\mathtt{sign}(\xi m_0+a x_0)=\mathbf{1}$.
	Consequently, the first Muon iterate in river coordinates is
	\begin{equation}
		x_1^{\rm Muon}=a x_0-\eta^{\rm Muon}\mathbf{1}.
		\nonumber
	\end{equation}
	A direct expansion yields
	\begin{align}
		&f_{\rm quad}(x_1^{\rm Muon})-f_{\rm quad}(a x_0)
        \nonumber \\
		=&
		\frac12\|a x_0-\eta^{\rm Muon}\mathbf{1}\|_2^2-\frac12\|a x_0\|_2^2 
		=
		-\eta^{\rm Muon}\,a\,\mathbf{1}^\top x_0 + \frac{(\eta^{\rm Muon})^2}{2}\,\|\mathbf{1}\|_2^2 
		\nonumber \\
		=&
		-\eta^{\rm Muon}\,a\,n\bar{x}_0 + \frac{(\eta^{\rm Muon})^2}{2}\,n 
		=
		\frac{n}{2}\,\eta^{\rm Muon}\left(\eta^{\rm Muon}-2a\bar{x}_0\right).
		\nonumber
	\end{align}
	With $\eta^{\rm Muon}=(1+\varepsilon)a\bar{x}_0$ and $\varepsilon\in(0,1)$, we have $\eta^{\rm Muon}-2a\bar{x}_0=(\varepsilon-1)a\bar{x}_0<0$, hence $f_{\rm quad}(x_1^{\rm Muon})<f_{\rm quad}(a x_0)$, which implies
	$
	\mathcal{F}_{\rm quad}(X_1^{\rm Muon})<\mathcal{F}_{\rm quad}(X_0)
	$.
	
	\paragraph{Step 2: Muon overshoots after two steps.}
	The second Muon momentum in river coordinates equals
	\begin{align}
		m_2^{\rm Muon}
		&=
		\xi(\xi m_0+a x_0)+x_1^{\rm Muon}
		=
		\xi^2 m_0 + a(1+\xi)x_0-\eta^{\rm Muon}\mathbf{1} 
		\nonumber \\
		&=
		\xi^2 m_0 + a\left[
		(1+\xi)x_0-(1+\varepsilon)\bar{x}_0\,\mathbf{1}
		\right].
		\nonumber
	\end{align}
	By assumption~\ref{equation:toy-in-positive-chamber}, $m_2^{\rm Muon}\succ \mathbf{0}$, hence again $\mathtt{sign}(m_2^{\rm Muon})=\mathbf{1}$ and
	\begin{equation}
		x_2^{\rm Muon}=x_1^{\rm Muon}-\eta^{\rm Muon}\mathbf{1}
		=
		a x_0-2\eta^{\rm Muon}\mathbf{1}.
		\nonumber
	\end{equation}
	Expanding as before,
	\begin{align}
		&f_{\rm quad}(x_2^{\rm Muon})-f_{\rm quad}(a x_0)
        \nonumber \\
		=&
		\frac12\|a x_0-2\eta^{\rm Muon}\mathbf{1}\|_2^2
		-\frac12\|a x_0\|_2^2 
		=
		-2\eta^{\rm Muon}\,a\,\mathbf{1}^\top x_0 
		+ \frac{(2\eta^{\rm Muon})^2}{2}\,\|\mathbf{1}\|_2^2 
		\nonumber \\
		=&
		-2\eta^{\rm Muon}\,a\,n\bar{x}_0 + 2(\eta^{\rm Muon})^2 n 
		=
		2n\eta^{\rm Muon}a\bar{x}_0
		\left(\frac{\eta^{\rm Muon}}{a\bar{x}_0}-1\right).
		\nonumber
	\end{align}
	Because $\eta^{\rm Muon}=(1+\varepsilon)a\bar{x}_0>a\bar{x}_0$, the last factor is positive; thus
	$
	f_{\rm quad}(x_2^{\rm Muon})>f_{\rm quad}(a x_0)
	$,
	which yields the overshoot claim~\ref{equation:toy-muon-overshoot}.
	
	\paragraph{Step 3: two-step decrease for momentum GD with a small constant step size.}
	Now consider the momentum-GD recursion
	\begin{align}
		m_1^{\rm GD}=\xi m_0+a x_0,
		&\qquad
		x_1^{\rm GD}=a x_0-\eta^{\rm GD}m_1^{\rm GD},
		\nonumber \\
		m_2^{\rm GD}=\xi m_1^{\rm GD}+x_1^{\rm GD},
		&\qquad
		x_2^{\rm GD}=x_1^{\rm GD}-\eta^{\rm GD}m_2^{\rm GD}.
		\nonumber
	\end{align}
	A straightforward elimination gives
	\begin{equation}
		x_2^{\rm GD}
		=
		a x_0
		-\eta^{\rm GD}\Bigl[\xi(1+\xi)m_0+(2+\xi)a x_0\Bigr]
		+(\eta^{\rm GD})^2(\xi m_0+a x_0).
		\label{equation:toy-gd-x2}
	\end{equation}
	Define the $a$-independent vectors $u:=x_0$ and $v:=m_0/a$,
	so that $\|v\|_2\le C_m$ and \eqref{equation:toy-gd-x2} can be rewritten as
	\begin{equation}
		x_2^{\rm GD}
		=
		a
		\left[
		u-
		\eta^{\rm GD}\bigl(\xi(1+\xi)v+(2+\xi)u\bigr)
		+(\eta^{\rm GD})^2(\xi v+u)
		\right].
		\nonumber
	\end{equation}
	Let
	$
	q(\eta):=u-\eta\bigl(\xi(1+\xi)v+(2+\xi)u\bigr)+\eta^2(\xi v+u)
	$,
	so that $x_2^{\rm GD}=a\,q(\eta^{\rm GD})$ and
	\begin{equation}
		f_{\rm quad}(x_2^{\rm GD})-f_{\rm quad}(a x_0)
		=
		\frac{a^2}{2}
		\Bigl(\|q(\eta^{\rm GD})\|_2^2-\|u\|_2^2\Bigr).
		\nonumber
	\end{equation}
	Expanding $\|q(\eta)\|_2^2$ around $\eta=0$ gives
	\begin{equation}
		\|q(\eta)\|_2^2
		=
		\|u\|_2^2
		-2\eta\,
		\left\langle
		u, \xi(1+\xi)v+(2+\xi)u
		\right\rangle
		+\eta^2\,R(\eta),
		\nonumber
	\end{equation}
	where $R(\eta)$ is a polynomial in $\eta$ whose coefficients are bounded in terms of $(\xi,\|u\|_2,\|v\|_2)$.
	Using $u=x_0\succ\mathbf{0}$ and $v=m_0/a\succ\mathbf{0}$, we have
	\begin{equation}
		\left\langle
		u, \xi(1+\xi)v+(2+\xi)u
		\right\rangle
		\geq
		(2+\xi)\|u\|_2^2.
		\nonumber
	\end{equation}
	Moreover, there exists a constant $\bar{C}_m=\bar{C}_m(\xi,\|u\|_2,C_m)$ such that $|R(\eta)|\leq \bar{C}_m$ for all $\eta\in[0,1]$ (e.g., by bounding $\|q(\eta)\|_2$ and its coefficients via triangle inequality).
	Therefore, for every $\eta\in(0,1]$,
	\begin{equation}
		\|q(\eta)\|_2^2-\|u\|_2^2
		\le
		-2(2+\xi)\eta\|u\|_2^2+\bar{C}_m\eta^2.
		\nonumber
	\end{equation}
	Choose
	\begin{equation}
		\eta_0
		:=
		\min\left\{1,\; \frac{(2+\xi)\|u\|_2^2}{\bar{C}_m}\right\}
		=
		\min\left\{1,\; \frac{(2+\xi)\|x_0\|_2^2}{\bar{C}_m}\right\},
		\nonumber
	\end{equation}
	then for any $0<\eta^{\rm GD}\le \eta_0$ we have
	$
	\|q(\eta^{\rm GD})\|_2^2-\|u\|_2^2<0
	$,
	which implies
	$
	f_{\rm quad}(x_2^{\rm GD})<f_{\rm quad}(a x_0)
	$
	and hence \eqref{equation:toy-gd-decrease}.
\end{proof}

\subsubsection{Late-stage overshoot for general smooth objectives} 
\label{subsubsec:general-model}

We now move from the exact quadratic toy model to a local theorem for general smooth objectives.
The proof is carried out inside a late-stage spectral-river tube, where the iterate, gradient, and momentum admit a river-hill decomposition and the hill component remains small.
In this regime, Muon's $\mathtt{msign}$ update behaves as an almost-constant river sign step, so the two-step overshoot mechanism from the quadratic model persists up to perturbative errors.

\paragraph{River-dominated momentum.}
We first explain why a river-dominated momentum state is natural in the late stage.
Intuitively, once the iterates enter a spectral-river tube, the gradients become predominantly supported on the river subspace, with only small hill components.
Since the stored momentum is formed as an exponential moving average (EMA) of recent gradients, it should therefore inherit the same river--hill structure.
The following lemma makes this inheritance precise.

\begin{lemma}[EMA momentum remains river-dominated in a tube]
	\label{lemma:ema-river-dominated}
	Suppose that at a reference time $0$ the stored momentum is an exponential moving average (EMA) over a window of length $L$:
	\begin{equation}
		M_0
		=
		\sum_{j=1}^{L}\xi^{j-1}\nabla \mathcal{F}(X_{-j})
		+
		\xi^L M_{-L},
		\qquad 0\le \xi <1.
		\nonumber
	\end{equation}
	Assume that along this window the gradients admit decompositions
	\begin{equation}
		\nabla \mathcal{F}(X_{-j})
		=
		\mathcal{B}(g_j)+E_j^g,
		\quad
		g_j \succ \mathbf{0},
		\quad
		\|g_j\|_2\le \overline C_g a,
		\quad
		\min_i [g_j]_i \ge \underline C_g a,
		\quad
		\|E_j^g\|_F\le \tau_g a,
		\nonumber
	\end{equation}
	for constants $\overline C_g,\underline C_g>0$ and $\tau_g\ge 0$ independent of $a$.
	Assume also that the old momentum satisfies
	\begin{equation}
		\|M_{-L}\|_F\le C_M a,
		\qquad
		\|\mathtt{Proj}_{\mathcal R^\perp}(M_{-L})\|_F\le \tau_m a,
		\nonumber
	\end{equation}
	for constants $C_M<\infty$ and $\tau_m\ge 0$ independent of $a$.
	Then $M_0$ admits a decomposition $M_0=\mathcal{B}(m_0)+E_0^m$ with $E_0^m\in \mathtt{span}\{B_1,\cdots,B_n\}^{\perp}$, where
	\begin{align}
		m_0
		&=
		\sum_{j=1}^{L}\xi^{j-1}g_j
		+
		\xi^L\,\mathcal{B}^{-1}\!
		\bigl(\mathtt{Proj}_{\mathcal R}(M_{-L})\bigr),
		\label{equation:ema-m0} 
		\\
		E_0^m
		&=
		\sum_{j=1}^{L}\xi^{j-1}E_j^g
		+
		\xi^L\,\mathtt{Proj}_{\mathcal R^\perp}(M_{-L}).
		\label{equation:ema-Em0} 
	\end{align}
	Here, $\mathcal{B}^{-1}$ denotes the inverse of $\mathcal{B}$ restricted to the river subspace.
	Consequently,
	\begin{equation}
		\|m_0\|_2
		\le
		\overline C_g a\frac{1-\xi^L}{1-\xi}
		+
		C_M \xi^L a,
		\qquad
		\min_i [m_0]_i
		\ge
		\underline C_g a\frac{1-\xi^L}{1-\xi}
		-
		C_M\xi^L a,
		\nonumber
	\end{equation}
	and
	\begin{equation}
		\|E_0^m\|_F
		\le
		\tau_g a\frac{1-\xi^L}{1-\xi}
		+
		\tau_m\xi^L a.
		\nonumber
	\end{equation}
\end{lemma}

\begin{proof}
	We decompose every term in the EMA formula into its river component and its hill component.
	For each $j=1,\cdots,L$, the assumed gradient decomposition is
	\begin{equation}
		\nabla \mathcal{F}(X_{-j})
		=
		\mathcal{B}(g_j) + E_j^g,
		\qquad
		E_j^g \in \mathtt{span}\{B_1,\cdots,B_n\}^{\perp}.
		\nonumber
	\end{equation}
	Similarly, decompose the old momentum as
	\begin{equation}
		M_{-L}
		=
		\mathtt{Proj}_{\mathcal R}(M_{-L})
		+
		\mathtt{Proj}_{\mathcal R^\perp}(M_{-L}),
		\nonumber
	\end{equation}
	where $\mathtt{Proj}_{\mathcal R}(M_{-L})\in \mathtt{span}\{B_1,\cdots,B_n\}$ and
	$\mathtt{Proj}_{\mathcal R^\perp}(M_{-L})\in \mathtt{span}\{B_1,\cdots,B_n\}^{\perp}$.
	
	Substituting these decompositions into the EMA representation of $M_0$ and using the linearity of $\mathcal{B}$ and of the orthogonal projections, we obtain
	\begin{align}
		M_0
		&=
		\sum_{j=1}^{L}\xi^{j-1}\bigl(\mathcal{B}(g_j)+E_j^g\bigr)
		+
		\xi^L\Bigl(
		\mathtt{Proj}_{\mathcal R}(M_{-L})
		+
		\mathtt{Proj}_{\mathcal R^\perp}(M_{-L})
		\Bigr)
		\notag\\
		&=
		\mathcal{B}\!\left(
		\sum_{j=1}^{L}\xi^{j-1}g_j
		+
		\xi^L\,\mathcal{B}^{-1}\!\bigl(\mathtt{Proj}_{\mathcal R}(M_{-L})\bigr)
		\right)
		+
		\left(
		\sum_{j=1}^{L}\xi^{j-1}E_j^g
		+
		\xi^L\,\mathtt{Proj}_{\mathcal R^\perp}(M_{-L})
		\right).
		\nonumber
	\end{align}
	Therefore $M_0$ admits the decomposition $M_0=\mathcal{B}(m_0)+E_0^m$, where $m_0$ is \eqref{equation:ema-m0} and $E_0^m$ is \eqref{equation:ema-Em0}.
	Because each $E_j^g$ and $\mathtt{Proj}_{\mathcal R^\perp}(M_{-L})$ lies in $\mathtt{span}\{B_1,\cdots,B_n\}^{\perp}$, the same is true for their weighted sum, i.e., $E_0^m \in \mathtt{span}\{B_1,\cdots,B_n\}^{\perp}$ as well.
	
	It remains to verify the stated bounds.
	
	First, since $\mathcal{B}$ is an isometric identification between $\mathbb{R}^n$ and the river subspace, we have
	\begin{equation}
		\bigl\|
		\mathcal{B}^{-1}\!\bigl(\mathtt{Proj}_{\mathcal R}(Y)\bigr)
		\bigr\|_2
		=
		\|\mathtt{Proj}_{\mathcal R}(Y)\|_F
		\le
		\|Y\|_F
		\qquad
		\text{for every matrix }Y \in\mathbb{R}^{n\times n}.
		\nonumber
	\end{equation}
	Hence, using the triangle inequality together with the assumptions on $g_j$ and $M_{-L}$,
	\begin{align}
		\|m_0\|_2
		&\le
		\sum_{j=1}^{L}\xi^{j-1}\|g_j\|_2
		+
		\xi^L
		\bigl\|
		\mathcal{B}^{-1}\!\bigl(\mathtt{Proj}_{\mathcal R}(M_{-L})\bigr)
		\bigr\|_2
		\le
		\sum_{j=1}^{L}\xi^{j-1}\overline C_g a
		+
		\xi^L\|M_{-L}\|_F
		\nonumber \\
		&\le
		\overline C_g a \sum_{j=1}^{L}\xi^{j-1}
		+
		C_M\xi^L a
		=
		\overline C_g a\frac{1-\xi^L}{1-\xi}
		+
		C_M\xi^L a.
		\nonumber
	\end{align}
	Next, we prove the coordinatewise lower bound.  Fix any coordinate $i$.  By the definition of $m_0$,
	\begin{equation}
		[m_0]_i
		=
		\sum_{j=1}^{L}\xi^{j-1}[g_j]_i
		+
		\xi^L
		\left[
		\mathcal{B}^{-1}\!\bigl(\mathtt{Proj}_{\mathcal R}(M_{-L})\bigr)
		\right]_i.
		\nonumber
	\end{equation}
	Using the lower bound on $[g_j]_i$ and the estimate
	\begin{equation}
		\left|
		\left[
		\mathcal{B}^{-1}\!\bigl(\mathtt{Proj}_{\mathcal R}(M_{-L})\bigr)
		\right]_i
		\right|
		\le
		\bigl\|
		\mathcal{B}^{-1}\!\bigl(\mathtt{Proj}_{\mathcal R}(M_{-L})\bigr)
		\bigr\|_2
		\le
		\|M_{-L}\|_F
		\le
		C_M a,
		\nonumber
	\end{equation}
	we obtain
	\begin{equation}
		[m_0]_i
		\ge
		\underline C_g a \sum_{j=1}^{L}\xi^{j-1}
		-
		C_M\xi^L a
		=
		\underline C_g a\frac{1-\xi^L}{1-\xi}
		-
		C_M\xi^L a.
		\nonumber
	\end{equation}
	Since this bound is uniform in $i$, it follows that
	\begin{equation}
		\min_i [m_0]_i
		\ge
		\underline C_g a\frac{1-\xi^L}{1-\xi}
		-
		C_M\xi^L a.
		\nonumber
	\end{equation}
	Finally, for the off-river hill remainder, the triangle inequality gives
	\begin{align}
		\|E_0^m\|_F
		\le&
		\sum_{j=1}^{L}\xi^{j-1}\|E_j^g\|_F
		+
		\xi^L\|\mathtt{Proj}_{\mathcal R^\perp}(M_{-L})\|_F
        \nonumber \\
		\le&
		\tau_g a \sum_{j=1}^{L}\xi^{j-1}
		+
		\tau_m\xi^L a
		=
		\tau_g a\frac{1-\xi^L}{1-\xi}
		+
		\tau_m\xi^L a.
		\nonumber
	\end{align}
	This proves the stated bound on $E_0^m$ and completes the proof.
\end{proof}

\paragraph{Structural hypotheses.}
We now state the structural hypotheses used in Theorem~\ref{theorem:late-stage-muon-overshoot-along-the-spectral-river}.  

The first assumption is standard: it requires a nondegenerate local minimizer together with quantitative Taylor control near $X^\star$.

\begin{assumption}[Local quadratic model with third-order control]
	\label{assumption:local-quadratic-model-with-third-order-control}
	The objective $\mathcal{F}:\mathbb{R}^{n\times n}\to\mathbb{R}$ is $C^3$ in a neighborhood of $X^\star$, and $\nabla \mathcal{F}(X^\star)=0$.
	Write $H_\star := \nabla^2 \mathcal{F}(X^\star)$.
	There exist constants $0<\underline{L}\leq \overline{L}<\infty$ such that
	\begin{equation}
		\underline{L}\|Y\|_F^2
		\le
		\langle Y, H_\star[Y]\rangle
		\le
		\overline{L}\|Y\|_F^2
		\qquad
		\text{for all }Y\in\mathbb{R}^{n\times n}.
		\label{equation:hessian-elliptic}
	\end{equation}
	Moreover, for some constant $C_3<\infty$ and all sufficiently small $Y$,
	\begin{equation}
		\|\nabla \mathcal{F}(X^\star+Y)-H_\star[Y]\|_F
		\le
		C_3\|Y\|_F^2,
		\nonumber
	\end{equation}
	and
	\begin{equation}
		\left|
		\mathcal{F}(X^\star+Y)-\mathcal{F}(X^\star)
		-\langle Y,H_\star[Y]\rangle / 2
		\right|
		\le
		C_3\|Y\|_F^3.
		\nonumber
	\end{equation}
\end{assumption}

\begin{assumption}[Hessian river-to-hill leakage bound]
	\label{assumption:hessian-river-to-hill-leakage-bound}
	Let $K$ be the matrix introduced earlier (see \ref{lemma:hessian-projection}) and assume $K\succ 0$.
	There exist constants $C_y<\infty$ and $\tau_h\ge 0$ such that for every $y\in\mathbb{R}^n$ with $\|y\|_2\le C_y$,
	\begin{equation}
		H_\star[\mathcal{B}(y)] =
		\mathcal{B}(Ky) + E^h(y),
		\qquad
		\mathtt{Proj}_{\mathcal R}(E^h(y))=0,
		\qquad
		\|E^h(y)\|_F\le \tau_h \|y\|_2.
		\label{equation:hessian-leakage}
	\end{equation}
\end{assumption}

\begin{assumption}[Late-stage initialization inside a tube]
	\label{assumption:late-stage-initialization-inside-a-tube}
	Let $a>0$.
	Assume the initial iterate decomposes as
	\begin{equation}
		X_0 =
		X^\star+\mathcal{B}(a x_0)+E_0^x,
		\qquad
		\mathtt{Proj}_{\mathcal R}(E_0^x)=0,
		\qquad
		\|E_0^x\|_F\le \tau_x a,
		\label{equation:init-iterate}
	\end{equation}
	for some $x_0\in\mathbb{R}^n$ satisfying $Kx_0\succ \mathbf{0}$.
	Assume the stored momentum decomposes as
	\begin{equation}
		M_0 =
		\mathcal{B}(m_0)+E_0^m,
		\qquad
		\|m_0\|_2\le C_m a,
		\qquad
		\|E_0^m\|_F\le \tau_m a,
		\label{equation:init-momentum}
	\end{equation}
	for constants $C_m<\infty$ and $\tau_m\ge 0$ independent of $a$.
\end{assumption}

\begin{lemma}[Polar stability near a positive river matrix]
	\label{lemma:polar-stability}
	Let $y\succ \mathbf{0}$ and define $Y=\mathcal{B}(y) + E^y$.
	If $\|E^y\|_2\le \frac12 \min_i [y]_i$, then $Y$ has full column rank, and
	\begin{equation}
		\|
		\mathtt{msign}(Y)-\mathtt{msign}(\mathcal{B}(y))
		\|_F
		\le
		\frac{4}{3\,\min_i [y]_i}\|E^y\|_F.
		\label{equation:polar-bound}
	\end{equation}
\end{lemma}

\begin{proof}
	The unperturbed matrix $\mathcal{B}(y)=U_\star \mathtt{Diag}(y)V_\star^\top$ has singular values exactly $[y]_1,\cdots,[y]_n$, and hence
	$
	\sigma_{\min}(\mathcal{B}(y))=\min_i [y]_i
	$.
	By Weyl's inequality~\cite{stewart1991perturbation},
	\begin{equation}
		\sigma_{\min}(Y)
		\ge
		\sigma_{\min}(\mathcal{B}(y))-\|E^y\|_2
		\ge
		\frac12 \min_i [y]_i
		>0,
		\nonumber
	\end{equation}
	so $Y$ has full column rank.
	
	We now apply the standard perturbation bound for the unit polar factor of full-column-rank matrices:
	\begin{equation}
		\|
		\mathtt{msign}(Y)-\mathtt{msign}(\mathcal{B}(y))
		\|_F
		\le
		\frac{2}{\sigma_{\min}(Y)+\sigma_{\min}(\mathcal{B}(y))}
		\|E^y\|_F.
		\nonumber
	\end{equation}
	Using the lower bounds above gives
	\begin{equation}
		\frac{2}{\sigma_{\min}(Y)+\sigma_{\min}(\mathcal{B}(y))}
		\le
		\frac{2}{\frac12\min_i [y]_i+\min_i [y]_i}
		=
		\frac{4}{3\min_i [y]_i},
		\nonumber
	\end{equation}
	and therefore
	\begin{equation}
		\|
		\mathtt{msign}(Y)-\mathtt{msign}(\mathcal{B}(y))
		\|_F
		\le
		\frac{4}{3\min_i [y]_i}\|E^y\|_F.
		\nonumber
	\end{equation}
	This proves the claim.
\end{proof}

\paragraph{Local step size and positivity margins.}
Define
\begin{equation}
	y_0
	:=
	\frac{\mathbf{1}^\top Kx_0}{\mathbf{1}^\top K\mathbf{1}}
	\qquad\text{and}\qquad
	\eta^{\rm Muon}
	:=
	(1+\varepsilon)a y_0,
	\label{equation:muon-stepsize}
\end{equation}
where $0<\varepsilon<1$ is fixed.
Define the scaled positivity margins
\begin{equation}
	\kappa_1
	:=
	\min_i
	\frac{\xi [m_0]_i + a [Kx_0]_i}{a},
	\label{equation:kappa1}
\end{equation}
and
\begin{equation}
	\kappa_2
	:=
	\min_i
	\left[
	\frac{
		\xi\bigl(\xi m_0 + aKx_0\bigr)
		+
		aK\bigl(x_0-(1+\varepsilon)y_0\mathbf{1}\bigr)
	}{a}
	\right]_i.
	\label{equation:kappa2}
\end{equation}

\begin{theorem}[Late-stage Muon overshoot along the spectral river (two-steps analysis)]
	\label{theorem:late-stage-muon-overshoot-along-the-spectral-river}
	Assume Assumptions~\ref{assumption:local-quadratic-model-with-third-order-control}, \ref{assumption:hessian-river-to-hill-leakage-bound}, and \ref{assumption:late-stage-initialization-inside-a-tube}.
	Fix $\xi\in(0,1)$ and $\varepsilon\in(0,1)$, and define $\eta^{\rm Muon}$ by \eqref{equation:muon-stepsize}.
	Assume that $\kappa_1$ and $\kappa_2$ defined in \eqref{equation:kappa1} and \ref{equation:kappa2} are bounded below by positive constants independent of $a$.
	Then there exist constants $a_0>0$, $\tau^h_0>0$, $\tau^m_0>0$, $\tau^x_0>0$, and $\eta_0>0$ (depending only on the local constants and on $(K,\xi,x_0,C_m,\varepsilon)$, but not on $a$) such that whenever
	\begin{equation}
		0<a\le a_0,
		\qquad
		0\le \tau_h\le \tau^h_0,
		\qquad
		0\le \tau_m\le \tau^m_0,
		\qquad
		0\le \tau_x\le \tau^x_0,
		\nonumber
	\end{equation}
	the Muon iterates and momentum-GD iterates initialized from the same pair $(X_0,M_0)$ satisfy
	\begin{equation}
		\mathcal{F}(X_1^{\rm Muon})
		<
		\mathcal{F}(X_0)
		<
		\mathcal{F}(X_2^{\rm Muon}),
		\qquad
		\mathcal{F}(X_2^{\rm GD})
		<
		\mathcal{F}(X_0),
		\label{equation:local-separation}
	\end{equation}
	for every GD step size $0<\eta^{\rm GD}\le \eta_0$.
	In particular, since $\eta^{\rm Muon}=O(a)$ while $\eta_0=O(1)$, we have $\eta^{\rm Muon}<\eta_0$ for all sufficiently small $a$.
\end{theorem}

\begin{proof}
	We prove \eqref{equation:local-separation} in three steps.
	
	\paragraph{Step 1: first Muon step decreases.}
	From \eqref{equation:init-iterate}, write $Y_0:=X_0-X^\star=\mathcal{B}(a x_0)+E_0^x$ with $\|E_0^x\|_F\le \tau_x a$.
	Assumption~\ref{assumption:local-quadratic-model-with-third-order-control} gives
	\begin{equation}
		\nabla \mathcal{F}(X_0)
		=
		H_\star[Y_0]+E^g_0,
		\qquad
		\|E^g_0\|_F\le C_3\|Y_0\|_F^2.
		\nonumber
	\end{equation}
	Using \eqref{equation:hessian-leakage} with $y=ax_0$, linearity of $H_\star$, and \eqref{equation:hessian-elliptic}, we obtain
	\begin{equation}
		\nabla \mathcal{F}(X_0)
		=
		a\,\mathcal{B}(Kx_0)+E_0^\nabla,
		\qquad
		\|E_0^\nabla\|_F\le \bar{C} (\tau_h+\tau_x)a + \bar{C} a^2,
		\nonumber
	\end{equation}
	where $\bar{C}$ depends only on the local constants and on $\|x_0\|_2$.
	Throughout this proof, $\bar C$ denotes a generic positive constant, possibly changing from line to line, whose value depends only on fixed problem constants and is independent of $a$, $\tau_h$, $\tau_m$, and $\tau_x$.
	Combining the result with \eqref{equation:init-momentum} yields
	\begin{equation}
		M_1^{\rm Muon}
		=
		\xi M_0+\nabla \mathcal{F}(X_0)
		=
		\mathcal{B}\bigl(\xi m_0+aKx_0\bigr)+E_1^m,
		\quad
		\|E_1^m\|_F\le \bar{C} (\tau_m+\tau_h+\tau_x)a+\bar{C} a^2.
		\nonumber
	\end{equation}
	By \eqref{equation:kappa1},
	\begin{equation}
		\min_i[\xi m_0+aKx_0]_i \ge a\kappa_1.
		\nonumber
	\end{equation}
	Thus, shrinking $a_0,\tau^h_0,\tau^m_0,\tau^x_0$ if needed so that $\|E_1^m\|_2 \le \frac12 a\kappa_1$, Lemma~\ref{lemma:polar-stability} implies
	\begin{equation}
		\mathtt{msign}(M_1^{\rm Muon})
		=
		\mathcal{B}(\mathbf{1})+\Delta_1,
		\quad
		\|\Delta_1\|_F\le \bar{C}(\tau_m+\tau_h+\tau_x+a).
		\label{equation:muon-sign-1}
	\end{equation}
	Hence with $\|E_1^x\|_F\le \bar{C} a(\tau_m+\tau_h+\tau_x+a)$, we have
	\begin{equation}
		X_1^{\rm Muon}
		=
		X_0-\eta^{\rm Muon}\mathtt{msign}(M_1^{\rm Muon})
		=
		X^\star + 
		\mathcal{B}(a x_0-\eta^{\rm Muon}\mathbf{1})+E_1^x.
		\nonumber 
	\end{equation}
	Using the loss expansion in Assumption~\ref{assumption:local-quadratic-model-with-third-order-control} at $X^\star$ and the identity
	\begin{equation}
		\langle \mathcal{B}(y), H_\star[\mathcal{B}(y)] \rangle
		=
		y^\top K y
		+
		\langle \mathcal{B}(y), E^h(y) \rangle
		=
		y^\top K y,
		\nonumber
	\end{equation}
	we obtain
	\begin{equation}
		\mathcal{F}(X_1^{\rm Muon})-\mathcal{F}(X_0)
		=
		-\eta^{\rm Muon}a\,\mathbf{1}^\top Kx_0
		+
		\frac12(\eta^{\rm Muon})^2\mathbf{1}^\top K\mathbf{1}
		+
		O\!\left(a^2(\tau_m+\tau_h+\tau_x+a)\right).
		\nonumber
	\end{equation}
	Substituting $\eta^{\rm Muon}=(1+\varepsilon)a y_0$ and $\mathbf{1}^\top Kx_0=y_0\,\mathbf{1}^\top K\mathbf{1}$ yields a strictly negative leading term:
	\begin{equation}
		\mathcal{F}(X_1^{\rm Muon})-\mathcal{F}(X_0)
		=
		-\frac12(1-\varepsilon^2)a^2 y_0^2\,\mathbf{1}^\top K\mathbf{1}
		+
		O\!\left(a^2(\tau_m+\tau_h+\tau_x+a)\right).
		\nonumber
	\end{equation}
	Choosing $a_0,\tau^h_0,\tau^m_0,\tau^x_0$ so that the error term is dominated by the leading negative term gives
	\begin{equation}
		\mathcal{F}(X_1^{\rm Muon})<\mathcal{F}(X_0).
		\nonumber
	\end{equation}
	
	\paragraph{Step 2: second Muon step overshoots.}
	Repeat the previous gradient expansion at $X_1^{\rm Muon}$.
	From the representation of $X_1^{\rm Muon}-X^\star$ and Assumptions~\ref{assumption:local-quadratic-model-with-third-order-control}--\ref{assumption:hessian-river-to-hill-leakage-bound}, we obtain
	\begin{equation}
		\nabla \mathcal{F}(X_1^{\rm Muon})
		=
		a\,\mathcal{B}\bigl(K(x_0-(1+\varepsilon)y_0\mathbf{1})\bigr)
		+
		E_1^\nabla,
		\quad
		\|E_1^\nabla\|_F\le \bar{C} (\tau_h+\tau_m+\tau_x)a+\bar{C} a^2.
		\nonumber
	\end{equation}
	Therefore
	\begin{equation}
		M_2^{\rm Muon}
		=
		\xi M_1^{\rm Muon}+\nabla \mathcal{F}(X_1^{\rm Muon})
		=
		\mathcal{B}(\bar{m}_1)+E_2^m,
		\quad
		\|E_2^m\|_F\le \bar{C} (\tau_h+\tau_m+\tau_x)a+\bar{C} a^2,
		\nonumber
	\end{equation}
	where
	\begin{equation}
		\bar{m}_1
		:=
		\xi(\xi m_0+aKx_0)
		+
		aK\bigl(x_0-(1+\varepsilon)y_0\mathbf{1}\bigr).
		\nonumber
	\end{equation}
	By \eqref{equation:kappa2}, we know $\min_i[\bar{m}_1]_i\ge a\kappa_2$.
	Shrinking $a_0,\tau^h_0,\tau^m_0,\tau^x_0$ if necessary so that $\|E_2^m\|_2 \le \frac12 a\kappa_2$, Lemma~\ref{lemma:polar-stability} yields
	\begin{equation}
		\mathtt{msign}(M_2^{\rm Muon})
		=
		\mathcal{B}(\mathbf{1})+\Delta_2,
		\quad
		\|\Delta_2\|_F\le \bar{C}(\tau_h+\tau_m+\tau_x+a).
		\nonumber
	\end{equation}
	Hence with $\|E_2^x\|_F\le \bar{C} a(\tau_h+\tau_m+\tau_x+a)$, we have
	\begin{equation}
		X_2^{\rm Muon}
		=
		X_1^{\rm Muon}-\eta^{\rm Muon}\mathtt{msign}(M_2^{\rm Muon})
		=
		X^\star+\mathcal{B}(a x_0-2\eta^{\rm Muon}\mathbf{1})+E_2^x.
		\nonumber
	\end{equation}
	Using again the loss expansion around $X^\star$ and the leakage bound \eqref{equation:hessian-leakage}, we obtain
	\begin{equation}
		\mathcal{F}(X_2^{\rm Muon})-\mathcal{F}(X_0)
		=
		-2\eta^{\rm Muon}a\,\mathbf{1}^\top Kx_0
		+
		2(\eta^{\rm Muon})^2\mathbf{1}^\top K\mathbf{1}
		+
		O\!\left(a^2(\tau_h+\tau_m+\tau_x+a)\right).
		\nonumber
	\end{equation}
	Substituting $\eta^{\rm Muon}=(1+\varepsilon)a y_0$ and $\mathbf{1}^\top Kx_0=y_0\,\mathbf{1}^\top K\mathbf{1}$ yields a strictly positive leading term:
	\begin{equation}
		\mathcal{F}(X_2^{\rm Muon})-\mathcal{F}(X_0)
		=
		2\varepsilon(1+\varepsilon)a^2 y_0^2\,\mathbf{1}^\top K\mathbf{1}
		+
		O\!\left(a^2(\tau_h+\tau_m+\tau_x+a)\right).
		\nonumber
	\end{equation}
	Choosing $a_0,\tau^h_0,\tau^m_0,\tau^x_0$ sufficiently small ensures the error term is dominated by the positive leading term, hence
	\begin{equation}
		\mathcal{F}(X_2^{\rm Muon})>\mathcal{F}(X_0).
		\nonumber
	\end{equation}
	
	\paragraph{Step 3: two-step momentum GD decreases for a small constant step size.}
	Consider momentum GD initialized from the same $(X_0,M_0)$:
	\begin{align}
		M_1^{\rm GD}=\xi M_0+\nabla \mathcal{F}(X_0),
		\qquad&
		X_1^{\rm GD}=X_0-\eta^{\rm GD}M_1^{\rm GD},
		\nonumber \\
		M_2^{\rm GD}=\xi M_1^{\rm GD}+\nabla \mathcal{F}(X_1^{\rm GD}),
		\qquad&
		X_2^{\rm GD}=X_1^{\rm GD}-\eta^{\rm GD}M_2^{\rm GD}.
		\nonumber
	\end{align}
	We show that there exists $\eta_0=O(1)$ such that $\mathcal{F}(X_2^{\rm GD})<\mathcal{F}(X_0)$ for all $0<\eta^{\rm GD}\le \eta_0$ and all sufficiently small $(a,\tau_h,\tau_m,\tau_x)$.
	
	Define the quadratic model function
	\begin{equation}
		Q(Y):=\frac12\langle Y,H_\star[Y]\rangle.
		\nonumber
	\end{equation}
	By Assumption~\ref{assumption:local-quadratic-model-with-third-order-control}, for all iterates staying in the local neighborhood,
	\begin{equation}
		\mathcal{F}(X^\star+Y)=\mathcal{F}(X^\star)+Q(Y)+O(\|Y\|_F^3),
		\nonumber
	\end{equation}
	and
	\begin{equation}
		\nabla \mathcal{F}(X^\star+Y)=H_\star[Y]+O(\|Y\|_F^2).
		\nonumber
	\end{equation}
	Since $X_0-X^\star=O(a)$ and we only take two steps with $\eta^{\rm GD}=O(1)$, by choosing $a_0$ sufficiently small we ensure $\|X_t^{\rm GD}-X^\star\|_F=O(a)$ for $t=0,1,2$, and thus the Taylor remainders contribute at most $O(a^3)$ to function values and $O(a^2)$ to gradients.
	
	We now compare the true two-step map with the quadratic dynamics obtained by replacing $\nabla \mathcal{F}(X)$ by $H_\star[X-X^\star]$.
	Let $\tilde{X}_t$ denote this quadratic two-step sequence (initialized at the same $\tilde{X}_0=X_0$ and $\tilde{M}_0=M_0$):
	\begin{align}
		\tilde{M}_1=\xi \tilde{M}_0+H_\star[\tilde{X}_0-X^\star],
		\qquad&
		\tilde{X}_1=\tilde{X}_0-\eta^{\rm GD}\tilde{M}_1,
		\nonumber \\
		\tilde{M}_2=\xi \tilde{M}_1+H_\star[\tilde{X}_1-X^\star],
		\qquad&
		\tilde{X}_2=\tilde{X}_1-\eta^{\rm GD}\tilde{M}_2.
		\nonumber
	\end{align}
	Because all operators are linear in this quadratic model, $\tilde{X}_2-X^\star$ is an affine polynomial in $\eta^{\rm GD}$ whose coefficients depend continuously on $(\tilde{X}_0-X^\star,\tilde{M}_0)$.
	In particular, the map $\eta^{\rm GD}\mapsto Q(\tilde{X}_2-X^\star)$ is a polynomial in $\eta^{\rm GD}$.
	
	We claim that the directional derivative at $\eta^{\rm GD}=0$ is strictly negative uniformly over all admissible $(X_0,M_0)$ in Assumption~\ref{assumption:late-stage-initialization-inside-a-tube}.
	At $\eta^{\rm GD}=0$, we have $\tilde{X}_1=\tilde{X}_0$ and hence
	\begin{equation}
		\left.\frac{d}{d\eta^{\rm GD}}
		Q(\tilde{X}_2-X^\star)\right|_{\eta^{\rm GD}=0}
		=
		-\left\langle \tilde{X}_0-X^\star,\; H_\star\!
		\left[
		\xi(1+\xi)\tilde{M}_0
		+
		(2+\xi)H_\star[\tilde{X}_0-X^\star]
		\right]
		\right\rangle.
		\nonumber
	\end{equation}
	Using $\tilde{X}_0-X^\star=\mathcal{B}(a x_0)+E_0^x$ and $\tilde{M}_0=\mathcal{B}(m_0)+E_0^m$, expand the inner products.
	The leading river contribution equals
	\begin{equation}
		-a^2(2+\xi)\,(Kx_0)^\top Kx_0
		-
		a\,\xi(1+\xi)\,(Kx_0)^\top m_0,
		\nonumber
	\end{equation}
	which is strictly negative because $K\succ 0$, $Kx_0\succ \mathbf{0}$, and $m_0$ is $O(a)$ with nonnegative coordinates in the late-stage regime (captured by the positivity margin in \eqref{equation:kappa1}).
	All remaining terms involve at least one hill component ($E_0^x$ or $E_0^m$) and are therefore bounded by
	\begin{equation}
		O\!\left(a^2(\tau_x+\tau_m+\tau_h)\right),
		\nonumber
	\end{equation}
	using \eqref{equation:hessian-elliptic}, \ref{equation:hessian-leakage}, and Cauchy--Schwarz.
	Thus, for sufficiently small $\tau_x+\tau_m+\tau_h$, the derivative remains strictly negative uniformly over the admissible set.
	
	Since $Q(\tilde{X}_2-X^\star)$ is a polynomial in $\eta^{\rm GD}$ with coefficients bounded uniformly over the admissible set (because $\|X_0-X^\star\|_F=O(a)$, $\|M_0\|_F=O(a)$, and $H_\star$ is bounded), there exists $\eta_0>0$ independent of $a$ such that
	\begin{equation}
		Q(\tilde{X}_2-X^\star)<Q(X_0-X^\star)
		\qquad
		\text{for all }0<\eta^{\rm GD}\le \eta_0.
		\nonumber
	\end{equation}
	Finally, the difference between the true iterate $X_2^{\rm GD}$ and the quadratic iterate $\tilde{X}_2$ over two steps is $O(a^2)$ in norm, because each gradient differs from its linearization by $O(a^2)$ and the step size is $O(1)$.
	Therefore,
	\begin{equation}
		\mathcal{F}(X_2^{\rm GD})-\mathcal{F}(X_0)
		=
		\left(Q(\widetilde{X}_2-X^\star)-Q(X_0-X^\star)\right)
		+
		O\!\left(a^2(\tau_h+\tau_m+\tau_x)+a^3\right),
		\nonumber
	\end{equation}
	and by shrinking $a_0,\tau^h_0,\tau^m_0,\tau^x_0$ if necessary, the negative quadratic decrease dominates the perturbation, giving
	\begin{equation}
		\mathcal{F}(X_2^{\rm GD})<\mathcal{F}(X_0)
		\qquad
		\text{for all }0<\eta^{\rm GD}\le \eta_0.
		\nonumber
	\end{equation}
	
	Combining Steps 1--3 proves \eqref{equation:local-separation}.
\end{proof}

\subsection{Implications for Algorithm Design}
\label{subsec:implications-for-algorithm-design}

\subsubsection{Late-stage Adam behaves more like gradient methods than Muon}
\label{subsubsec:late-stage-adam}

We finally record an Adam-like comparison.
The purpose of this proposition is not to analyze the full Adam recursion, but rather to show that, once the diagonal preconditioner has stabilized, the resulting short-horizon behavior is qualitatively closer to gradient-based refinement than to Muon's sign-type dynamics.
As a transparent late-stage surrogate for Adam, we consider a momentum method with a frozen positive diagonal preconditioner:
\begin{equation}
	M_{t+1}^{\rm Adam}
	=
	\xi M_t^{\rm Adam}
	+
	\nabla \mathcal{F}(X_t^{\rm Adam}),
	\qquad
	X_{t+1}^{\rm Adam}
	=
	X_t^{\rm Adam}
	-
	\eta^{\rm Adam} P^{-1} M_{t+1}^{\rm Adam},
	\label{equation:adam-like-recursion}
\end{equation}
where
\begin{equation}
	P=\mathtt{Diag}(p_1,\cdots,p_{n^2}),
	\qquad
	0<\underline p\le p_i\le \overline p,
	\qquad
	i=1,2,\cdots,n^2.
	\nonumber
\end{equation}
This model is intended to approximate the late-stage behavior of Adam over a short time horizon after the second-moment denominator has effectively stabilized.

The next proposition shows that, under a frozen positive diagonal preconditioner and mild alignment and size conditions, the first two Adam-like steps are both descent steps, with an admissible step size that is independent of the residual scale $a$.

\begin{proposition}[Two-step descent for an Adam-like method]
	\label{proposition:two-step-descent-for-Adam-like-method}
	Assume that $\mathcal{F}$ is $L$-smooth on a neighborhood containing the iterates under consideration, i.e.,
	\begin{equation}
		\|\nabla \mathcal{F}(X+\Delta X)-\nabla \mathcal{F}(X)\|_F
		\le
		L\|\Delta X\|_F
		\qquad
		\text{for all }X,\Delta X\text{ in the neighborhood}.
		\nonumber
	\end{equation}
	Let $(X_t^{\rm Adam},M_t^{\rm Adam})$ evolve according to \eqref{equation:adam-like-recursion}, and suppose that for $t=0,1$,
	\begin{align}
		\bigl\langle
		\nabla \mathcal{F}(X_t^{\rm Adam}),
		P^{-1}M_{t+1}^{\rm Adam}
		\bigr\rangle
		&\ge
		\underline{C}_A\,
		\|\nabla \mathcal{F}(X_t^{\rm Adam})\|_F^2,
		\nonumber
		\\
		\|M_{t+1}^{\rm Adam}\|_F
		&\le
		\overline{C}_A\,
		\|\nabla \mathcal{F}(X_t^{\rm Adam})\|_F,
		\nonumber
	\end{align}
	for some constants $\underline{C}_A,\overline{C}_A>0$ independent of the residual scale $a$.
	Then every step size
	\begin{equation}
		0<\eta^{\rm Adam}\le \eta_0
		:=
		\frac{2\underline{C}_A\underline p^2}{L\overline{C}_A^2}
		\nonumber
	\end{equation}
	guarantees
	\begin{equation}
		\mathcal{F}(X_1^{\rm Adam})<\mathcal{F}(X_0^{\rm Adam}),
		\qquad
		\mathcal{F}(X_2^{\rm Adam})<\mathcal{F}(X_1^{\rm Adam}).
		\label{equation:adam-two-step-descent}
	\end{equation}
	In particular, the admissible Adam-like step size is independent of the late-stage residual scale.
\end{proposition}

\begin{proof}
	By the descent lemma for $L$-smooth functions, for each $t=0,1$,
	\begin{equation}
		\mathcal{F}(X_t^{\rm Adam}-\eta^{\rm Adam}D_t)
		\le
		\mathcal{F}(X_t^{\rm Adam})
		-
		\eta^{\rm Adam}
		\langle \nabla \mathcal{F}(X_t^{\rm Adam}),D_t\rangle
		+
		\frac{L(\eta^{\rm Adam})^2}{2}\|D_t\|_F^2,
		\nonumber
	\end{equation}
	where $D_t:=P^{-1}M_{t+1}^{\rm Adam}$.
	Since $X_{t+1}^{\rm Adam}=X_t^{\rm Adam}-\eta^{\rm Adam}D_t$, we obtain
	\begin{align}
		&\mathcal{F}(X_{t+1}^{\rm Adam})-\mathcal{F}(X_t^{\rm Adam})
        \nonumber \\
		\le&
		-\eta^{\rm Adam}
		\bigl\langle
		\nabla \mathcal{F}(X_t^{\rm Adam}),
		P^{-1}M_{t+1}^{\rm Adam}
		\bigr\rangle
		+
		\frac{L(\eta^{\rm Adam})^2}{2}
		\|P^{-1}M_{t+1}^{\rm Adam}\|_F^2
		\nonumber\\
		\le&
		-\eta^{\rm Adam} \underline{C}_A
		\|\nabla \mathcal{F}(X_t^{\rm Adam})\|_F^2
		+
		\frac{L(\eta^{\rm Adam})^2}{2\underline p^2}
		\|M_{t+1}^{\rm Adam}\|_F^2
		\nonumber\\
		\le&
		\left(
		-\eta^{\rm Adam} \underline{C}_A
		+
		\frac{L(\eta^{\rm Adam})^2 \overline{C}_A^2}{2\underline p^2}
		\right)
		\|\nabla \mathcal{F}(X_t^{\rm Adam})\|_F^2.
		\nonumber
	\end{align}
	Therefore, whenever $0<\eta^{\rm Adam}\le\eta_0$,
	the coefficient in parentheses is negative, and hence
	\begin{equation}
		\mathcal{F}(X_{t+1}^{\rm Adam})<\mathcal{F}(X_t^{\rm Adam})
		\qquad\text{for }t=0,1.
		\nonumber
	\end{equation}
	This proves \eqref{equation:adam-two-step-descent}.
\end{proof}

\subsubsection{Guidance for algorithm design}
\label{subsubsec:guidance-algorithm-design}

The preceding analysis suggests a simple principle for optimizer design: the update geometry that is beneficial in the early stage need not remain appropriate in the late stage.

At a river point
\begin{equation}
	X_t=X^\star+\mathcal{B}(x_t),
	\nonumber
\end{equation}
the natural positive-chamber Muon scale induced by the local quadratic model is
\begin{equation}
	y_t :=
	\frac{\mathbf{1}^\top K x_t}{\mathbf{1}^\top K\mathbf{1}}.
	\label{equation:residual-dependent-scale}
\end{equation}
If $x_t=a_t x_0$, then the safe displacement scale for Muon contracts proportionally to the current river residual.
Consequently, a schedule depending only on time, such as cosine decay,
\begin{equation}
	\eta_t^{\rm Muon}
	=
	\frac{\eta_{\max}}{2}
	\left(
	1+\cos\frac{\pi t}{T}
	\right),
	\qquad
	0\le t\le T,
	\label{equation:cosine-schedule}
\end{equation}
can become mismatched with the actual late-stage geometry in two distinct ways.

First, if $\eta_t^{\rm Muon}$ is too large relative to $y_t$, then Muon may step across the minimizer, after which the momentum can continue pushing in essentially the same polar direction and produce overshoot.
Second, if $\eta_t^{\rm Muon}$ is too small relative to $y_t$, then Muon still moves by additive sign-type increments, but now does so unnecessarily slowly along the river.

This mismatch is already visible in the one-dimensional quadratic model: $\frac12 x^2$.
A sign-based or Muon-like step takes the form $x \gets x-\eta_t^{\rm Muon}\,\mathtt{sign}(x)$, which increases the objective whenever
\begin{equation}
	0<|x|<\eta_t^{\rm Muon} / 2.
	\nonumber
\end{equation}
By contrast, gradient descent with step size $1$ sends $x$ to $0$ in one step, and the Adam-like model above also admits an $O(1)$ admissible step size under a stabilized positive preconditioner.

Switching away from Muon is therefore not merely another learning-rate schedule; it changes the geometry of the update itself.
Muon intentionally suppresses gradient magnitude information through $\mathtt{msign}$, which is advantageous when the goal is rapid spectral alignment and coarse river tracking.
However, in the final river-refinement regime, the residual magnitude contains precisely the information needed to calibrate the step and avoid overshoot.
Gradient methods preserve this information through the raw gradient and momentum, while Adam-like methods preserve it through a positive diagonal preconditioner.
This suggests that a Muon-to-GD or Muon-to-Adam transition can serve as a residual-adaptive alternative to relying solely on a manually prescribed decay schedule.

\section{Additional Details for Section~\ref{sec:empirical-evidence-for-Muon-as-an-early-exploration-optimizer}}
\label{sec:appendix-numerical-experiment}
\subsection{Additional Case Studies}
\label{subsec:additional_case_studies}

\subsubsection{2D Anisotropic Spectral Slice Case}
\label{subsubsec:2d_anisotropic_spectral_slice_case}

We consider a simple two-dimensional anisotropic spectral landscape as a motivating example for the phenomena discussed in the main text. Specifically, we choose two orthogonal rank-one spectral directions $U = uu^\top$ and $V = vv^\top$ with $\langle U,V\rangle_F = 0$.
This allows us to restrict the matrix variable to the two-dimensional spectral slice $X(a,b) = aU + bV$,
where $a,b\in\mathbb{R}$ are scalar coordinates. On this slice, we define an anisotropic quadratic objective
\begin{equation}
	f(a,b)
	=
	\frac{1}{2}\lambda_w (a-a^\star)^2
	+
	\frac{1}{2}\lambda_s (b-b^\star)^2,
	\qquad 
	\lambda_s \gg \lambda_w .
	\nonumber
\end{equation}
Equivalently, lifting the gradient back to the ambient matrix space gives
\begin{equation}
	\nabla_X f
	=
	\lambda_w (a-a^\star) U
	+
	\lambda_s (b-b^\star) V .
	\nonumber
\end{equation}

\begin{figure*}[ht]
	\centering
	\includegraphics[width=\textwidth]{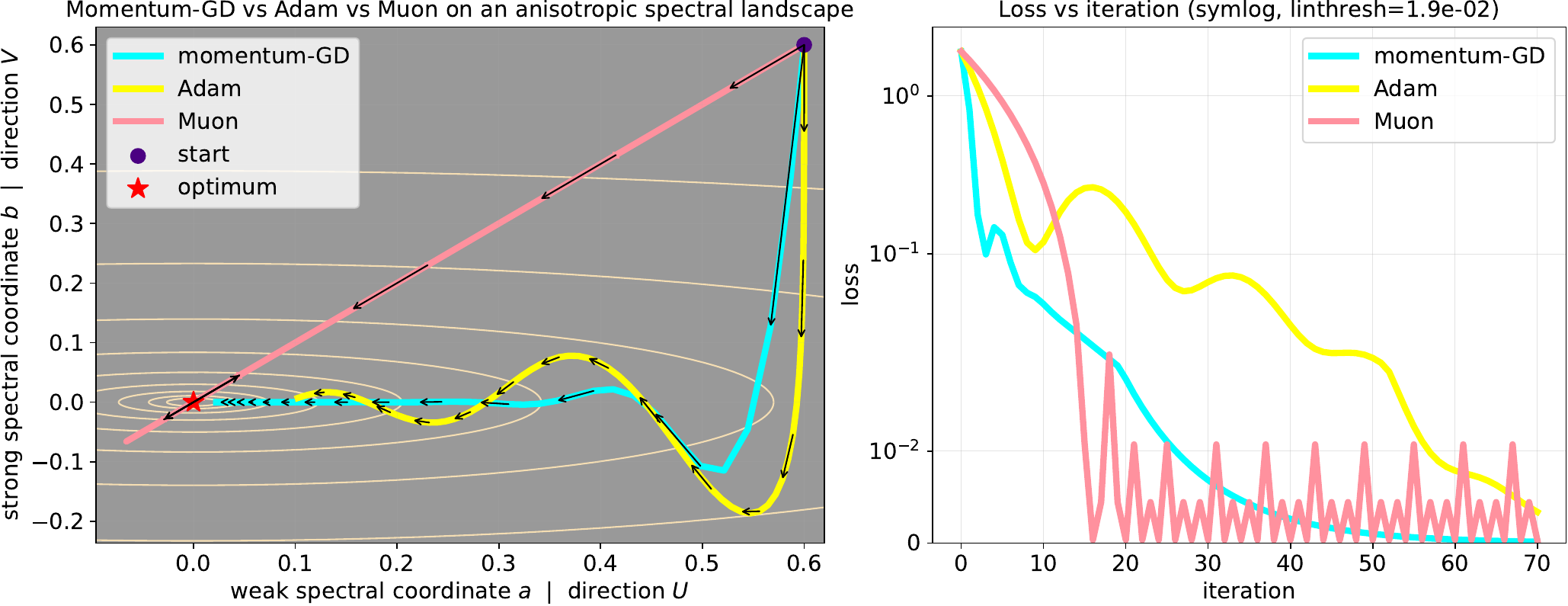}
	\caption{
		\textbf{2D anisotropic spectral slice case: pure optimizers.}
		Left: optimization trajectories of momentum-GD, Adam, and Muon under a fixed learning rate $0.037$. The black arrows indicate the displacement every other iteration. 
		Right: loss versus iteration for the three optimizers.
	}
	\label{figure:two_dim_spectral_slice_case_pure}
\end{figure*}

This toy problem can be viewed as a ``no-bulk'' version of the rank-one mixed-spectrum matrix sensing model studied in the paper. In our experiments, we take
\begin{equation}
	u = [1,1]^\top,
	\qquad 
	v = [1,-1]^\top,
	\qquad
	\lambda_w = 0.6,
	\qquad
	\lambda_s = 10.0.
	\nonumber
\end{equation}
so that the two rank-one directions are orthogonal in Frobenius inner product. We initialize the optimization at $(a_0,b_0) = (0.6,0.6)$, and set the ground-truth point to $(a^\star,b^\star) = (0,0)$.
Since the dynamics are fully described by the two coordinates $(a,b)$, we can directly visualize the optimization trajectories on the two-dimensional contour plot of the objective.

Figure~\ref{figure:two_dim_spectral_slice_case_pure} reports the trajectories and loss curves for momentum-GD, Adam, and Muon under a fixed learning rate. We observe that momentum-GD and Adam move relatively slowly along the weak-curvature direction and approach the ground truth only gradually within $70$ iterations. In contrast, Muon reaches the neighborhood of the ground truth much earlier, confirming its fast early-stage progress in anisotropic spectral landscapes (corroborate Theorem~\ref{theorem:early-stage-muon-acceleration-along-the-spectral-river}). However, the loss curve reveals a complementary drawback: with a fixed learning rate, Muon repeatedly overshoots around the ground-truth point and exhibits persistent oscillations, leading to substantially worse late-stage convergence (corroborate Theorem~\ref{theorem:late-stage-muon-overshoot-along-the-spectral-river}). By comparison, momentum-GD and Adam converge more accurately once they enter the local basin. This example also illustrates that, in the late stage, Adam behaves more similarly to GD-type methods than to Muon (corroborate Proposition~\ref{proposition:two-step-descent-for-Adam-like-method}). 

\begin{figure*}[ht]
	\centering
	\includegraphics[width=\textwidth]{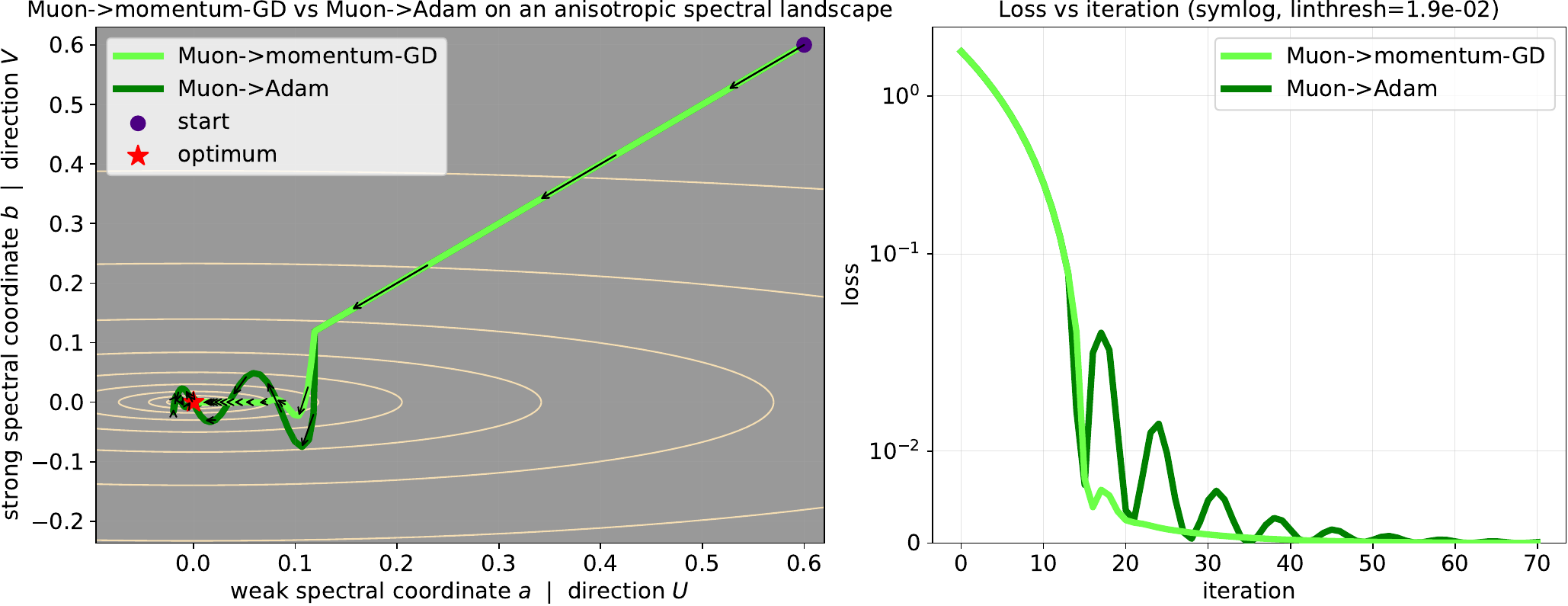}
	\caption{
		\textbf{2D anisotropic spectral slice case: hybrid optimizers.}
		Left: optimization trajectories of the two hybrid schemes under a fixed learning rate $0.037$. The black arrows indicate the displacement every other iteration.
		Right: loss versus iteration for the pure and hybrid optimizers.
	}
	\label{figure:two_dim_spectral_slice_case_hybrid}
\end{figure*}

Overall, the pure-optimizer experiment is consistent with our intuition: Muon achieves rapid early-stage progress but poor final accuracy, while momentum-GD and Adam converge more slowly but more precisely in late-stage.
Motivated by the intuition, we further evaluate two hybrid optimization schemes: ``Muon $\to$ momentum-GD'' and ``Muon $\to$ Adam''. In both cases, we switch from Muon to the second optimizer at iteration $13$, at which point Muon has already reached a small neighborhood of the ground truth. The results are shown in Figure~\ref{figure:two_dim_spectral_slice_case_hybrid}. Both hybrid methods inherit the fast initial movement of Muon and the accurate late-stage convergence of momentum-GD or Adam. Consequently, they outperform the corresponding pure optimizers in terms of both speed and final accuracy. This toy example motivates (also corroborates) the analysis in Section~\ref{sec:river-valley-as-tool-in-generalized-settings}, where the river-valley perspective is used as a tool for understanding generalized optimization settings.

\subsubsection{Mixed-Spectrum Matrix Sensing Case}
\label{subsubsec:mixed_spectrum_matrix_sensing_case}

\begin{figure}[ht]
	\centering
	\includegraphics[width=1.0\linewidth]{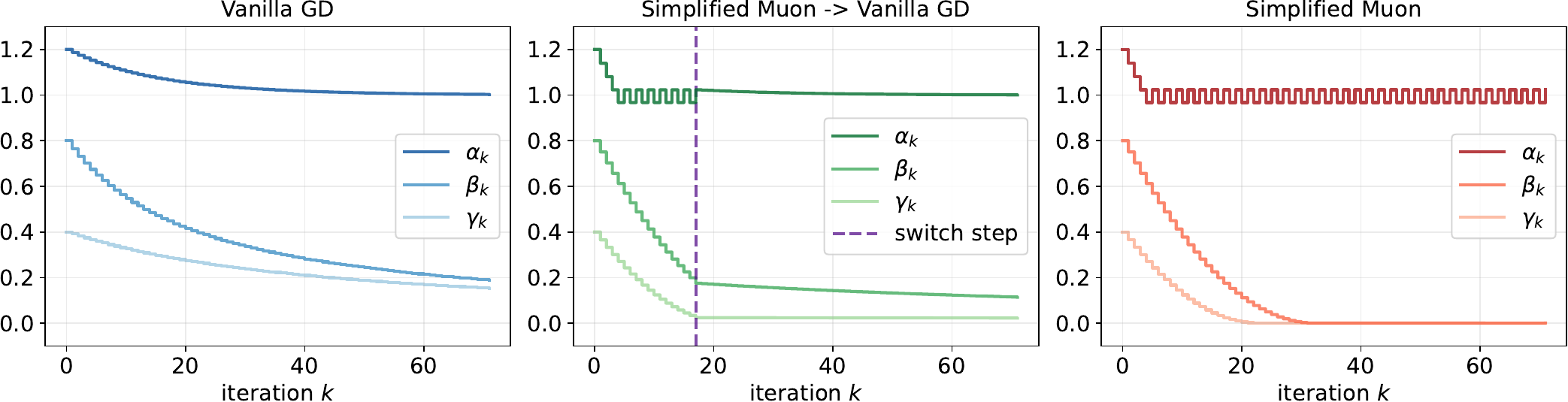}
	\caption{
		\textbf{Mixed-spectrum MS with a diagonal interaction matrix $\hat{K}$.}
		Evolution of the reduced coefficients $\alpha_k,\beta_k,\gamma_k$ versus iteration.
		Left: pure GD.
		Middle: hybrid optimization, which switches from simplified Muon to vanilla GD.
		Right: pure Muon.
	}
	\label{figure:mixed_spectrum_ms_diag_coefficients}
\end{figure}

\paragraph{Diagonal case.}
We first consider the diagonal reduced dynamics of the mixed-spectrum matrix sensing model. The coefficient iterations for this example are illustrated in Figure~\ref{figure:mixed_spectrum_ms_diag_coefficients}. In this experiment, we set $\hat{K} = \mathtt{Diag}(0.5,0.5,0.5)$, and run the continuous-time dynamics up to time $T=2.0$. The hybrid optimizer switches from Muon to GD at time $t=0.45$, since at this point the Muon trajectory has already reached a small neighborhood of the ground-truth solution (PSD case has the same solution)
\begin{equation}
	(\alpha^\star,\beta^\star,\gamma^\star) = (1,0,0).
	\nonumber
\end{equation}
The discrete trajectories $\{\alpha_k,\beta_k,\gamma_k\}$ are obtained by discretizing the reduced continuous dynamics in Proposition~\ref{proposition:reduced_continuous_dynamics_of_gd_and_muon_diag}. We initialize the iterates at $\alpha_0 = 1.2$, $\beta_0 = 0.8$, $\gamma_0 = 0.4$ and use stepsize $0.028$.

The rightmost panel shows that Muon rapidly approaches the ground truth, but its late-stage convergence is inaccurate. In particular, the signal coordinate $\alpha$ oscillates around the target value $\alpha^\star=1$ (which also corresponds to the river). Note that the nuisance coordinates $\beta$ and $\gamma$ do not visibly oscillate below zero in the plot because, for visualization purposes, we apply a 
$\max\{\cdot,0\}$ truncation (also used in PSD case) when evaluating the square-root expressions appearing in the discrete simulation and closed-form trajectories. This truncation prevents negative plotted values, but it should not be interpreted as evidence that Muon attains high-accuracy convergence in the $\beta$ and $\gamma$ directions. 

By contrast, vanilla GD converges more slowly. In the diagonal case, the GD reduced dynamics are fully decoupled and admit the closed-form solutions in \eqref{equation:gd_diag_closed_form}, which exhibit inverse-logistic or inverse-polynomial type convergence. In particular, $\beta$ and $\gamma$ decay only at a $1/t$-type rate.
For simplified Muon, the corresponding diagonal reduced dynamics take the form in \eqref{equation:muon_diag_closed_form}.
Hence, in the idealized continuous-time diagonal model, Muon moves each coordinate toward its target at an approximately constant speed in the square-root scale. This explains its rapid early progress. However, under discretization with a fixed stepsize, the same aggressive update can lead to overshooting and oscillatory late-stage behavior near the optimum.
The middle panel of Figure~\ref{figure:mixed_spectrum_ms_diag_coefficients} shows the performance of the hybrid scheme. The hybrid method combines the fast early-stage movement of simplified Muon with the accurate late-stage convergence of vanilla GD. Consequently, it reaches the neighborhood of the ground truth quickly and then refines the solution more stably than pure Muon. In this sense, the hybrid method improves over both pure optimizers.

We emphasize that this diagonal example is fully decoupled: the river coordinate $\alpha$ evolves independently of the hill coordinates $\beta$ and $\gamma$. As a result, the river progress and hill progress are not synchronized. It is therefore possible for the signal coordinate to reach the ground-truth neighborhood while the nuisance coordinates remain comparatively far from zero. For this reason, the diagonal case does not fall under the coupled river-hill setting covered by Theorem~\ref{theorem:formal_mixed_spectrum_ms} (they even have different river definitions). Nevertheless, it remains a useful sanity check for the main intuition of the paper: Muon is best viewed as an early-stage exploration optimizer that rapidly moves along some spectral directions, while GD-type methods are better suited for high-accuracy late-stage refinement.

\begin{figure*}[ht]
	\centering
	\includegraphics[width=\textwidth]{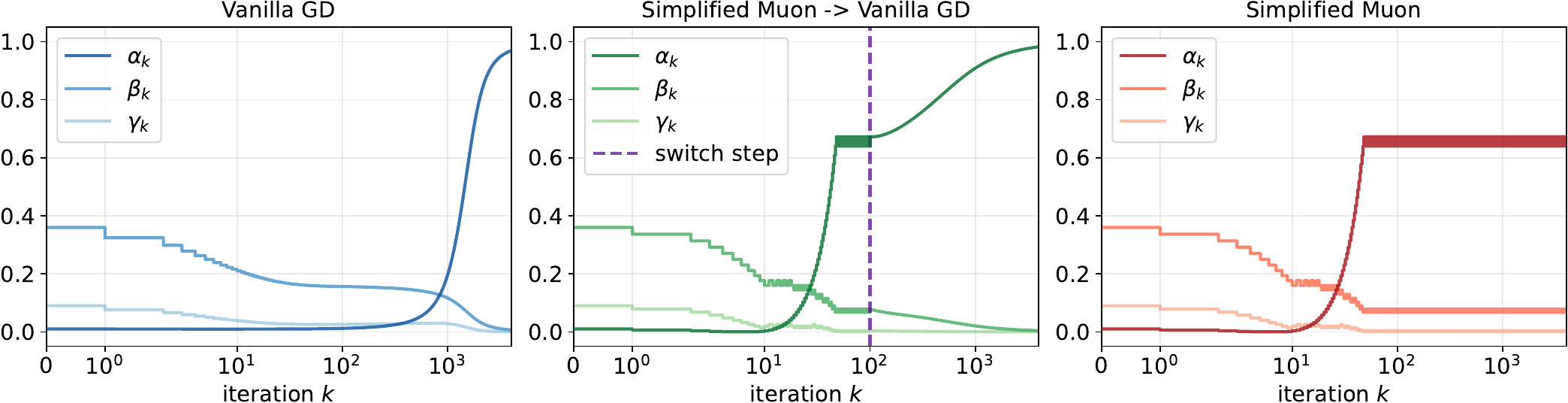}
	\caption{
		\textbf{Mixed-spectrum MS with a PSD interaction matrix $\hat{K}$.}
		Evolution of the reduced coefficients $\alpha_k,\beta_k,\gamma_k$ versus iteration (horizontal axis plotted in the symmetric-log scale).
		Left: pure GD.
		Middle: hybrid optimization, which switches from simplified Muon to GD.
		Right: pure Muon.
	}
	\label{figure:mixed_spectrum_ms_psd_coefficients}
\end{figure*}

\paragraph{PSD case.}
We next consider a more general positive semidefinite interaction matrix $\hat{K}$, beyond the diagonal setting discussed above. Unlike the diagonal case, the coordinates are now coupled through $\hat{K}$, and this example falls within the setting of Theorem~\ref{theorem:formal_mixed_spectrum_ms}. In the experiment, we choose a rank-1 positive semidefinite matrix $\hat{K}$ (the outer product of $[0.1,0.5,0.8]$ with itself).
We simulate the GD dynamics using the iteration in \eqref{equation:gd_river_valley_iteration}, and the Muon dynamics using the iteration in \eqref{equation:muon_river_valley_iteration}. Both methods use the same stepsize $0.02$.
The initial condition is $(\alpha_0,\beta_0,\gamma_0)=(0.01,0.36,0.09)$, and all methods are run for $4000$ iterations. For the hybrid method, we switch from Muon to GD at iteration $100$, since the Muon trajectory has already entered an oscillatory regime by that time.

The results are shown in Figure~\ref{figure:mixed_spectrum_ms_psd_coefficients}. We observe that pure Muon makes extremely rapid initial progress: within fewer than $100$ iterations, it reaches a neighborhood of a local basin close to the ground-truth solution. However, after this fast transient phase, its trajectory begins to oscillate rather than refine. In particular, the coefficient $\alpha_k$ starts oscillating around approximately $0.6$, and the nuisance coefficient $\beta_k$ does not oscillate near zero. This indicates that, under the chosen fixed stepsize, Muon can become trapped in an oscillatory neighborhood of a local minimum: the stepsize is not large enough to escape the basin, yet not small enough to ensure accurate convergence within it. Thus, the late-stage behavior of Muon is highly sensitive to learning-rate tuning.
By comparison, pure GD progresses much more slowly in the early stage. It requires more than $1000$ iterations to reach a region comparable to where Muon arrives within the first $100$ iterations. Nevertheless, once GD approaches the relevant basin, it continues to refine the coefficients and eventually converges more accurately toward the ground-truth solution. This behavior is consistent with the diagonal case: GD is slow but stable, whereas Muon is fast but can suffer from persistent oscillations under a fixed stepsize.

The hybrid method again combines the advantages of both optimizers. During the initial Muon phase, it rapidly approaches the relevant basin. After switching to GD at iteration $100$, it avoids the oscillation of pure Muon and instead performs stable late-stage refinement. 
This PSD example thus reinforces the main message of the paper: Muon is particularly effective as an early-stage exploration method, while GD-type dynamics are better suited for accurate late-stage convergence.

\subsubsection{Two-Layer Neural Network Case}
\label{subsubsec:two_layer_neural_network_case}

We further evaluate the behavior of different optimizers in a standard supervised learning setting: training a two-layer neural network on the MNIST handwritten digit classification task~\cite{deng2012mnist}. This experiment serves as an empirical extension of the mixed-spectrum matrix sensing examples to a more conventional machine learning problem.
There is also a close connection between Burer-Monteiro (BM)-factorized matrix sensing and the training of a two-layer neural network~\cite{li2018algorithmic}. In BM-factorized matrix sensing, the sensing matrices $\{A_i\}_{i=1}^m$ can be interpreted as data samples, while the PSD ground-truth matrix
\begin{equation}
	M^\star = X^\star {X^\star}^{\top},
	\qquad
	X^\star \in \mathbb{R}^{n\times r_\star},
	\nonumber
\end{equation}
plays a role analogous to the optimal weight representation of the model. From this perspective, the present case study can be viewed as extending the mixed-spectrum matrix sensing construction back to a general BM-factorized MS or shallow neural network setting. This extension differs from the theoretical model in two important ways. First, the sensing matrices are no longer required to follow the special mixed-spectrum construction in \eqref{equation:mixed_spectrum_sensing_matrix}; instead, they are replaced by general data-dependent matrices induced by the MNIST samples. Second, the objective is no longer the population loss in \eqref{equation:main_mixed_spectrum_matrix_objfunc}; rather, we optimize the empirical mini-batch training loss.

\begin{figure*}[ht]
	\centering
	\includegraphics[width=0.82\textwidth]{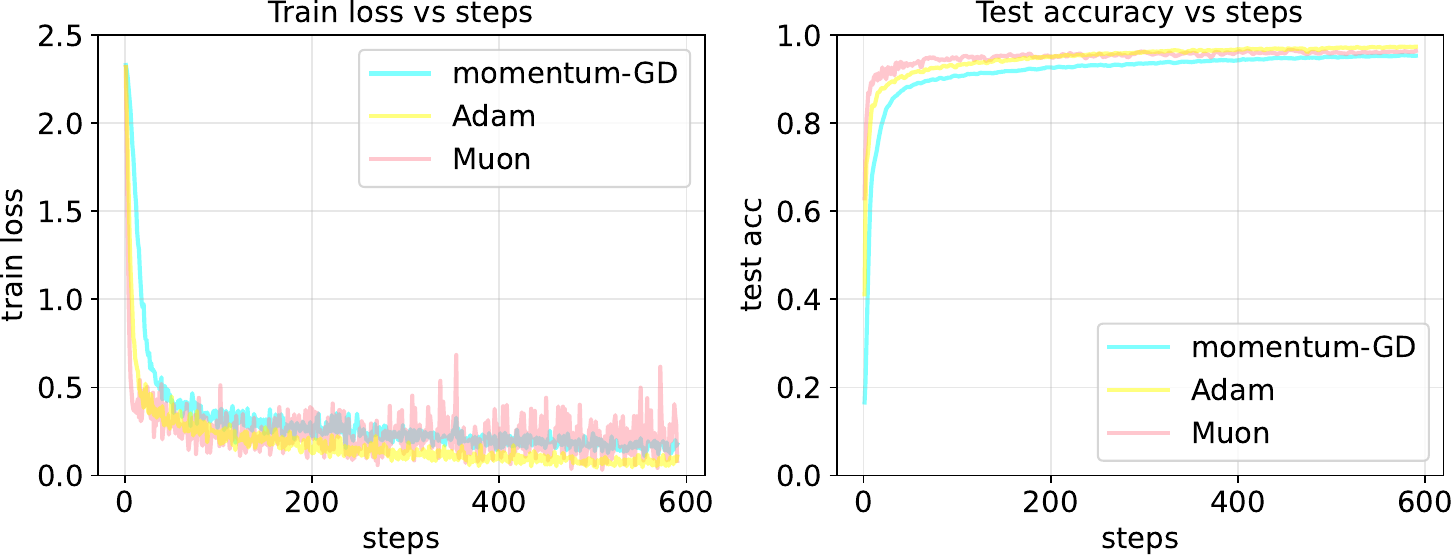}
	\caption{
		\textbf{Training a two-layer neural network on MNIST with pure optimizers.}
		Left: training loss of momentum-GD, Adam, and Muon.
		Right: test accuracy of the three optimizers.
	}
	\label{figure:two_layer_neural_network_case_pure}
\end{figure*}

Concretely, we train a two-layer neural network with input dimension $28\times 28 = 784$, hidden width $256$, and output dimension $10$, corresponding to the ten MNIST digit classes. We use the ReLU activation function~\cite{nair2010rectified} and train with mini-batches of size $512$. Unlike the matrix sensing examples, which use full-batch optimization, the present experiment adopts stochastic mini-batch training in order to better reflect common practice in modern machine learning.
Figure~\ref{figure:two_layer_neural_network_case_pure} reports the results of training for $5$ epochs, corresponding to $590$ iterations in total. For momentum-GD, we use learning rate $0.01$ and momentum coefficient $0.9$. For Muon, we use learning rate $0.1$ and momentum coefficient $0.9$; this larger learning rate is consistent with the intuition from our theoretical analysis that Muon can tolerate more aggressive early-stage steps (Theorem~\ref{theorem:early-stage-muon-acceleration-along-the-spectral-river}). For Adam, we use learning rate $0.001$, $\beta_1=0.9$, $\beta_2=0.999$, and $\epsilon=10^{-8}$. We found that using a substantially larger learning rate for Adam degrades its performance.

The left panel of Figure~\ref{figure:two_layer_neural_network_case_pure} shows that Muon decreases the training loss fastest during the early stage, followed by Adam and then momentum-GD. However, after roughly $400$ iterations, the Muon loss continues to oscillate, whereas the losses of Adam and momentum-GD keep decreasing more steadily. The right panel shows the corresponding test accuracy. Muon rapidly reaches above $90\%$ test accuracy, but is later surpassed by Adam and approached by momentum-GD. This behavior is consistent with the phenomenon observed in the matrix sensing case studies: Muon is highly effective for fast early progress, but its fixed-stepsize dynamics may lead to inferior late-stage refinement.

\begin{figure*}[ht]
	\centering
	\includegraphics[width=0.82\textwidth]{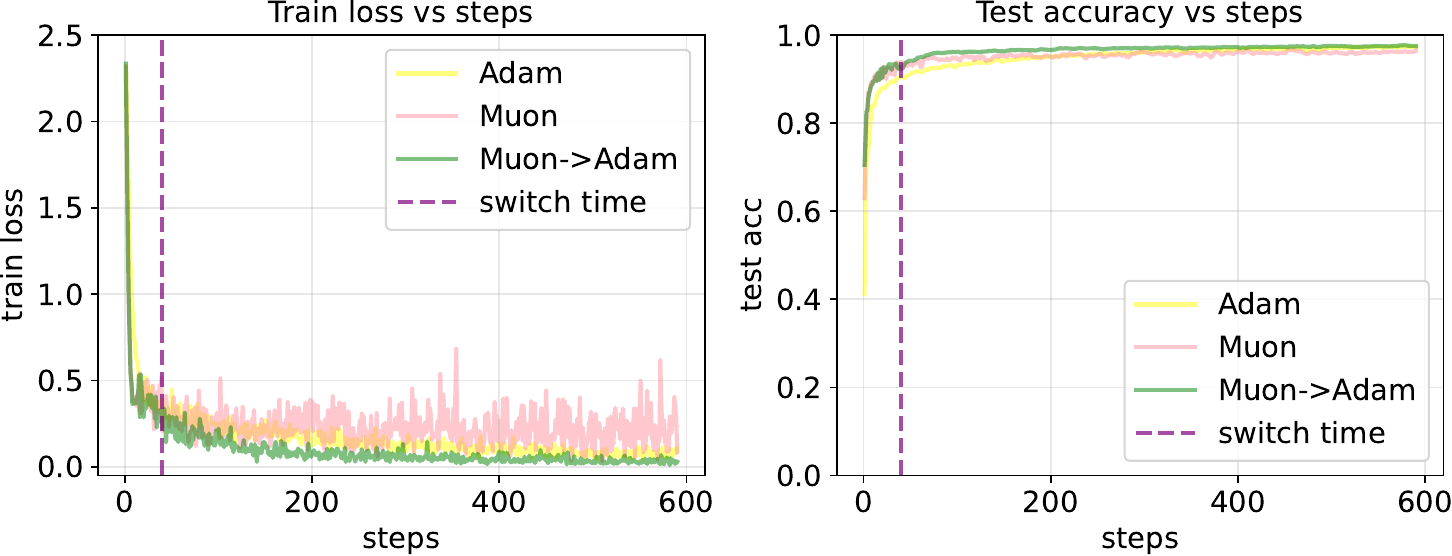}
	\caption{
		\textbf{Training a two-layer neural network on MNIST with a hybrid optimizer.}
		The optimizer first follows Muon and then switches to Adam.
		Left: training loss.
		Right: test accuracy.
	}
	\label{figure:two_layer_neural_network_case_hybrid}
\end{figure*}

Motivated by this observation, we further test a hybrid optimization strategy. In Figure~\ref{figure:two_layer_neural_network_case_hybrid}, the optimizer starts with Muon and switches to Adam at iteration $40$. During the Muon phase, we maintain the moment estimates required by Adam, although they are not used to update the parameters before the switch. After switching, Adam inherits these accumulated first- and second-order moment estimates.
The left panel of Figure~\ref{figure:two_layer_neural_network_case_hybrid} shows that, after the switch, the hybrid method avoids the persistent oscillations observed in pure Muon. Moreover, its training loss decreases below that of pure Adam. The right panel shows that the hybrid method also achieves the highest final test accuracy among the compared methods. We additionally tested a ``Muon $\to$ momentum-GD'' hybrid strategy, which exhibits qualitatively similar behavior to the ``Muon $\to$ Adam'' variant. 

Overall, this case study provides empirical evidence, beyond the stylized theoretical MS model, for the main message of the paper: Muon is particularly effective as an early-stage exploration optimizer, while GD-type methods like Adam, are better suited for accurate late-stage convergence.

\subsection{Empirical Observations on LLM Pre-training}
\label{subsec:empirical_observations_on_llm_pretraining}

\subsubsection{Implementation Details for LLM Pre-training}
\label{subsubsec:implementation_details_for_llm_pre-training}

We provide additional implementation details for the ``Muon$\to$AdamW'' pre-training experiments.
For computational efficiency, these diagnostic experiments are conducted on a fixed subset of OpenWebText2~\cite{gao2020pile}, rather than on the full corpus.
The subset contains approximately 250M tokens and is preprocessed into binary training and validation files.
All optimizer variants use the same train--validation split.
Mini-batches are sampled using a fixed data seed and without replacement.
Together, these design choices allow us to compare optimizer and scheduler variants under a controlled data protocol.

We conduct two sets of experiments with different training horizons.
In the 15K-step experiments, we compare a fixed-LR Muon baseline with Muon$\to$AdamW runs that switch at 10K, 11K, 12K, 13K, or 14K steps.
In the 20K-step experiments, we compare Muon baselines with several LR schedules, including fixed, linear, cosine, and \texttt{cos\_inf} schedules, against Muon$\to$AdamW runs with late-stage switches.
For two-stage Muon$\to$AdamW runs, we train with Muon up to a prescribed switching step, resume from the corresponding checkpoint, and continue with AdamW for the remaining steps.
At the switch, AdamW's second moment is initialized using the elementwise maximum of the calibrated squared gradients and the scaled squared Muon momentum.
Its first moment is initialized by blending the Muon update direction with the calibrated gradient direction.
This warm-start procedure is formalized in Section~\ref{subsubsec:how-to-switch}.
When a LR scheduler is used after the switch, it is applied over the remaining training horizon.
All runs are trained on $4\times$ NVIDIA RTX PRO 6000 GPUs with 96GB of memory.
Each process uses a batch size of 16.
We evaluate every 100 steps using 32 validation batches and log training and validation metrics at the same interval.
Unless otherwise stated, we use seed 0, data seed 1337, weight decay 0.1, and a gradient-clipping threshold of 1.0.

\subsubsection{Effects of Switching Time and Learning-Rate Scheduling}
\label{subsubsec:effects_of_switching_time_and_lr_scheduling}
The main LLM pre-training results in Section~\ref{sec:empirical-evidence-for-Muon-as-an-early-exploration-optimizer} focus primarily on the validation loss of the two-stage Muon$\to$AdamW procedure.
Here, we provide additional results that report both training and validation losses, offering a more complete view of the optimization dynamics.

\begin{figure*}[ht]
	\centering
	\includegraphics[width=0.9\textwidth]{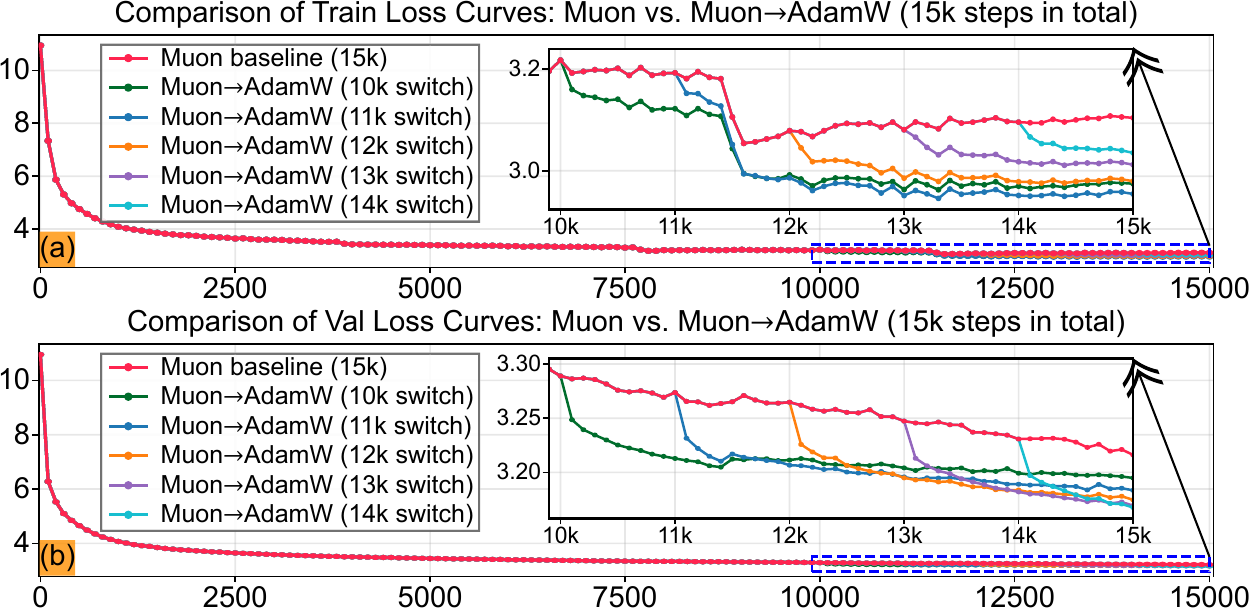}
	\caption{
		\textbf{Muon$\to$AdamW switching times over a 15K-step pre-training horizon.}
		We compare a fixed-LR Muon baseline with two-stage Muon$\to$AdamW runs that switch from Muon to AdamW at 10K, 11K, 12K, 13K, or 14K steps.
		Panel~(a) shows the training loss, and Panel~(b) shows the validation loss.
		The insets zoom in on the late-stage region where the optimizer switch occurs.
	}
	\label{figure:llm_pretrain_comparison_train_val_15k}
\end{figure*}

Figure~\ref{figure:llm_pretrain_comparison_train_val_15k} studies the effect of switching time in a 15K-step training horizon without learning-rate decay.
In this setting, the Muon baseline is trained with a fixed LR of $5\times 10^{-3}$, while the AdamW phase in the two-stage runs uses a fixed LR of $5\times 10^{-5}$.
We switch from Muon to AdamW at 10K, 11K, 12K, 13K, or 14K steps.
The results reveal two main observations.
First, Panel~(a) shows late-stage switching from Muon to AdamW consistently outperforms the fixed-step Muon baseline across all tested switching times.
This indicates that the advantage of the Muon$\to$AdamW transition is robust to the precise choice of switching point.
Second, within the tested range, Panel~(b) shows delaying the switch generally leads to lower validation loss.
This suggests that it is beneficial to exploit Muon for the earlier exploration phase and use AdamW mainly for late-stage refinement.

Figure~\ref{figure:llm_pretrain_comparison_train_val_20k} further examines this phenomenon under LR decay over a 20K-step horizon.
We compare Muon-only baselines trained with several LR schedules, including linear, cosine, and \texttt{cos\_inf}, against two-stage Muon$\to$AdamW runs that switch at different training steps.
For the hybrid runs, the AdamW phase uses a peak LR of $9\times 10^{-5}$ and is evaluated with the same set of schedules.
For each switching time, we report the best-performing Muon$\to$AdamW configuration, which uses the \texttt{cos\_inf} schedule during the AdamW phase.
Compared with the scheduled Muon-only baselines, the differences in training loss are relatively modest, excluding the fixed-LR Muon baseline, which performs substantially worse.
However, the validation-loss curves show clearer benefits from the Muon$\to$AdamW transition.
In particular, switching at 17K or 18K steps yields markedly lower final validation losses than switching at 15K or 16K steps.
The 19K-switch setting also decreases rapidly after the transition and appears promising.
Nevertheless, because only a small number of steps remain after the switch, its validation loss does not fall below the best-performing earlier-switch configurations by the end of training.
This suggests that, under a fixed training budget, switching too late may leave insufficient time for AdamW to fully realize its benefits.

\begin{figure*}[ht]
	\centering
	\includegraphics[width=0.9\textwidth]{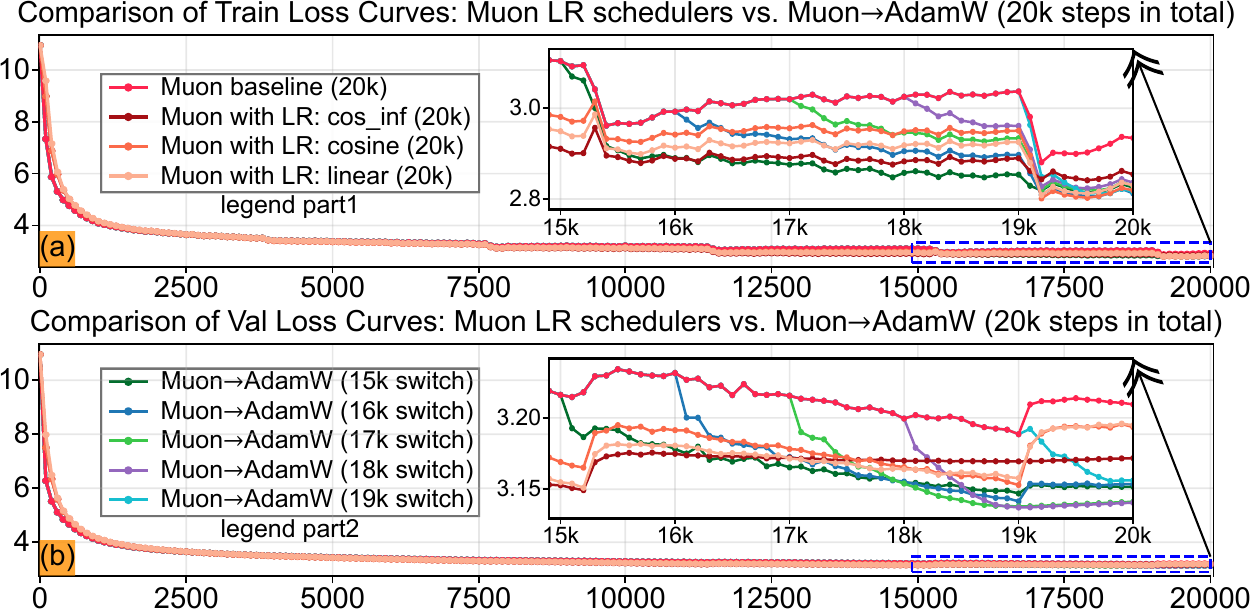}
	\caption{
		\textbf{Muon LR schedules versus Muon$\to$AdamW switching over a 20K-step pre-training horizon.}
		We compare Muon baselines with fixed, linear, cosine, and \texttt{cos\_inf} LR schedules against two-stage Muon$\to$AdamW runs that switch from Muon to AdamW at 15K, 16K, 17K, 18K, or 19K steps.
		Panel~(a) shows the training loss, and Panel~(b) shows the validation loss.
		The insets zoom in on the final 5K steps, where the effects of LR scheduling and optimizer switching are most visible.
	}
	\label{figure:llm_pretrain_comparison_train_val_20k}
\end{figure*}

\subsubsection{How to Switch: Warm-Starting AdamW from a Muon Checkpoint}
\label{subsubsec:how-to-switch}

At a selected switching iteration $k$, we restore the model parameters $\theta_k$ and the Muon optimizer state from the corresponding checkpoint, then instantiate a new AdamW optimizer at the same model parameters. Because Muon does not maintain AdamW's element-wise second moment, its state
cannot be converted exactly into an AdamW state. We therefore construct a warm start that combines two complementary sources of information: gradients evaluated at the restored model and the momentum direction accumulated by Muon. We use subscripts $i$, $k$, and $\ell$ to index a parameter tensor, the
switching iteration, and a calibration mini-batch, respectively. Let $p_{i,k}$ denote the Muon momentum buffer for parameter tensor $i$ stored at iteration $k$.

We first estimate local gradient statistics without updating the restored model. Using $B$ calibration mini-batches, let $g_{i,k,\ell}$ be the gradient of parameter tensor $i$ on calibration mini-batch $\ell$. We compute
\begin{equation*}
    \bar g_{i,k}^{\rm cal}
    :=
    \frac{1}{B}\sum_{\ell=1}^{B}g_{i,k,\ell},
    \qquad
    q_{i,k}^{\rm cal}
    :=
    \frac{1}{B}\sum_{\ell=1}^{B}
    g_{i,k,\ell}\odot g_{i,k,\ell},
\end{equation*}
where $\odot$ denotes element-wise multiplication. Thus, $q_{i,k}^{\rm cal}$ is the mean of the element-wise squared gradients, not the element-wise square of $\bar g_{i,k}^{\rm cal}$.

We next initialize AdamW's second-moment state. The calibration statistic captures the gradient scale at the switching point, whereas the squared Muon buffer retains information accumulated before the switch. We combine them as
\begin{equation*}
    \tilde v_{i,k}
    :=
    \max\!\left\{
        c_g q_{i,k}^{\rm cal},
        c_\mu (s p_{i,k})^2
    \right\},
    \qquad
    \tau_{i,k}
    :=
    \delta\max\!\left\{
        10^{-30},
        \operatorname{mean}(\tilde v_{i,k})
    \right\},
\end{equation*}
where both the square and the maximum are element-wise, except that $\operatorname{mean}(\tilde v_{i,k})$ is a scalar. Here, $s$ rescales the checkpoint momentum, $c_g$ and $c_\mu$ weight the two sources, and $\delta$ sets a relative numerical floor. The initial AdamW second moment is
\begin{equation*}
    v_{i,k}^{\rm init}
    :=
    \max\!\left\{\tilde v_{i,k},\tau_{i,k}\right\},
\end{equation*}
with the scalar floor broadcast over the tensor.

To initialize the first moment, we blend the Muon direction with the local gradient direction in AdamW's preconditioned coordinates. Let $Q_{i,k}$ be the five-step Newton--Schulz orthogonalization of $s p_{i,k}$. A tensor with more than two dimensions is first reshaped to $d_0\times\prod_{j>0}d_j$ and reshaped back after orthogonalization, whereas a one-dimensional tensor passes through without orthogonalization. For a matrix with dimensions $r_i\times c_i$, we additionally multiply the
orthogonalized output by $\max\{\sqrt{r_i/c_i},\sqrt{c_i/r_i}\}$, following the scaling used by the Muon implementation. 
Define
\begin{equation*}
    u_{i,k}^{\rm grad}
    :=
    \frac{\bar g_{i,k}^{\rm cal}}{\sqrt{v_{i,k}^{\rm init}}},
    \qquad
    \tilde u_{i,k}^{\rm Muon}
    :=
    Q_{i,k}
    \frac{\operatorname{RMS}(u_{i,k}^{\rm grad})}
         {\operatorname{RMS}(Q_{i,k})},
\end{equation*}
where $\operatorname{RMS}(x):=\sqrt{\operatorname{mean}(x^2)}$. The first expression is the calibrated gradient direction under the initialized AdamW preconditioner; the second rescales the Muon direction to the same RMS. 
We then form
\begin{equation*}
    u_{i,k}^{\rm init}
    :=
    \alpha\tilde u_{i,k}^{\rm Muon}
    +(1-\alpha)u_{i,k}^{\rm grad},
    \qquad
    m_{i,k}^{\rm init}
    :=
    u_{i,k}^{\rm init}\odot\sqrt{v_{i,k}^{\rm init}}.
\end{equation*}
Consequently, $m_{i,k}^{\rm init}/\sqrt{v_{i,k}^{\rm init}}=u_{i,k}^{\rm init}$ before AdamW applies its subsequent moment update and bias correction. If calibration gradients are unavailable for a parameter, the implementation falls back to the orthogonalized Muon direction without the gradient-based blend.

Finally, we initialize AdamW's per-parameter step counter to the checkpoint iteration $k$ and continue with the configured $(\beta_1,\beta_2)$ and LR schedule. This construction carries both the pre-switch Muon direction and local gradient information into the handoff, but it should not be interpreted
as reconstructing the moments that AdamW would have accumulated had it been used from the beginning of training. In the 20K-step experiments, we use $B=1024$, $s=1$, $c_g=0.3$, $c_\mu=0.01$, $\delta=10^{-5}$, and $\alpha=0.8$. We use AdamW with $(\beta_1,\beta_2)=(0.9,0.95)$, a peak LR of $9\times10^{-5}$, and zero weight decay. After computing the calibration statistics at $\theta_k$, we reset the training-data reader to the switching point before resuming optimization.

\subsubsection{When to Switch: A Possible Local Criterion}
\label{subsubsec:when-to-switch}

The experiments above indicate that an exact switching iteration is not essential within the late-stage window considered here. Across a broad range of late switching points, the Muon$\to$AdamW runs improve upon the corresponding Muon-only baselines.
These observations suggest that the Muon$\to$AdamW strategy may be useful beyond language-model pre-training, particularly in large-scale vision, video, and embodied-AI settings where training commonly involves long horizons and staged optimization pipelines~\citep{yuan2025dvp,yuan2026video,yuan2026autodrive}.
The main practical restriction is that a switch made too close to the end of the training horizon leaves too few iterations for AdamW to complete the refinement phase. Thus, for the present experiments, selecting a reasonable late-stage switching window is more important than identifying a single precise iteration.

Designing an adaptive switching rule remains an interesting direction for future work. We give a brief local calculation that suggests one possible criterion, while leaving stochastic, nonconvex, and fully general variants to future study. Consider a twice-differentiable objective $\mathcal F$ that is smooth and locally strongly convex around the current matrix parameter $X$.
Let
\begin{equation*}
    G:=\nabla\mathcal F(X),
    \qquad
    \mathcal H:=\nabla^2\mathcal F(X),
\end{equation*}
and assume that, for every matrix perturbation $Y$ in this neighborhood,
\begin{equation*}
    \lambda_{\min}\|Y\|_F^2
    \le
    \langle Y,\mathcal H[Y]\rangle
    \le
    \lambda_{\max}\|Y\|_F^2,
    \qquad
    0<\lambda_{\min}\le\lambda_{\max}.
\end{equation*}
For a candidate descent direction $-D$, the local quadratic model gives
\begin{equation*}
    \mathcal F(X-\eta D)-\mathcal F(X)
    \approx
    -\eta\langle G,D\rangle
    +\frac{\eta^2}{2}\langle D,\mathcal H[D]\rangle.
\end{equation*}
Provided $\langle G,D\rangle>0$, the minimizing coefficient and the resulting
predicted change are
\begin{equation*}
    \eta_D^*
    =
    \frac{\langle G,D\rangle}
         {\langle D,\mathcal H[D]\rangle},
    \qquad
    \Delta_{\rm quad}(D)
    =
    -\frac{\langle G,D\rangle^2}
          {2\langle D,\mathcal H[D]\rangle}.
\end{equation*}

For GD, take $D_{\rm GD}:=G$. For an idealized Muon step, take $D_{\rm Muon}:=\mathtt{msign}(G)$, and define $r_{\rm pol}:=\|D_{\rm Muon}\|_F^2$. Then $\langle G,D_{\rm Muon}\rangle=\|G\|_*$, where $\|\cdot\|_*$ denotes the nuclear norm. Under the compact-polar convention, $r_{\rm pol}=\operatorname{rank}(G)$; the same identity holds for the full-rank rectangular gradients typically encountered by Muon. 
The corresponding predicted changes are
\begin{equation*}
    \Delta_{\rm GD}
    =
    -\frac{\|G\|_F^4}
          {2\langle G,\mathcal H[G]\rangle},
    \qquad
    \Delta_{\rm Muon}
    =
    -\frac{\|G\|_*^2}
          {2\langle D_{\rm Muon},\mathcal H[D_{\rm Muon}]\rangle}.
\end{equation*}
The curvature bounds imply
\begin{equation*}
    \Delta_{\rm GD}
    \le
    -\frac{\|G\|_F^2}{2\lambda_{\max}},
    \qquad
    \Delta_{\rm Muon}
    \ge
    -\frac{\|G\|_*^2}{2\lambda_{\min}r_{\rm pol}}.
\end{equation*}
Therefore, a sufficient condition for the locally optimized GD step to provide at least as much predicted decrease as the locally optimized Muon step is
\begin{equation*}
    \|G\|_*^2
    \le
    \frac{r_{\rm pol}}{\kappa(\mathcal H)}\|G\|_F^2,
    \qquad
    \kappa(\mathcal H):=
    \frac{\lambda_{\max}}{\lambda_{\min}}.
\end{equation*}
Equivalently, defining the spectral effective rank $r_{\rm eff}(G):=\|G\|_*^2/\|G\|_F^2$, the condition is $\kappa(\mathcal H)r_{\rm eff}(G)\le r_{\rm pol}$. This calculation suggests that a future switching rule could monitor two quantities: local conditioning and the dispersion of the gradient's singular values. However, it is only a local deterministic comparison between idealized GD and Muon directions; it does not yet account for AdamW's adaptive preconditioner, momentum, stochastic gradients, nonconvex curvature, higher-order terms, or variation across parameter blocks. Constructing reliable estimators for these quantities and turning this sufficient condition into a practical Muon$\to$AdamW switching criterion are left for future work.



\end{document}